%% file: main.tex
\documentclass{article}

\usepackage{microtype}
\usepackage{graphicx}
\usepackage{subfigure}
\usepackage{booktabs} % for professional tables
\usepackage[utf8]{inputenc} 

\usepackage{hyperref}
\usepackage{url}

\usepackage[accepted]{icml2023}

\usepackage{amsmath}
\usepackage{amssymb}
\usepackage{mathtools}
\usepackage{amsthm}

\usepackage[capitalize,noabbrev]{cleveref}

\usepackage{relsize}
\usepackage{algorithm}
\usepackage{algorithmic}
\usepackage{xspace}
\usepackage{color,soul}
\usepackage{comment}
\usepackage{amsmath,amssymb,amsfonts}
\usepackage{graphicx}
\usepackage{wrapfig}
\usepackage{amsthm}
\usepackage{stfloats}
\usepackage{multirow}
\let\origNabla\nabla
\renewcommand*\nabla{\mathlarger\origNabla}

\theoremstyle{plain}
\newtheorem{theorem}{Theorem}[section]

\newtheorem{lemma}[theorem]{Lemma}

\theoremstyle{definition}
\newtheorem{definition}[theorem]{Definition}

\theoremstyle{remark}
\newtheorem{remark}[theorem]{Remark}

\usepackage[textsize=tiny]{todonotes}

\icmltitlerunning{Mitigating Propagation Failures in PINNs using R3 Sampling}

\begin{document}

\twocolumn[
\icmltitle{Mitigating Propagation Failures in Physics-informed Neural Networks  \\ using Retain-Resample-Release (R3) Sampling}

% It is OKAY to include author information, even for blind
% submissions: the style file will automatically remove it for you
% unless you've provided the [accepted] option to the icml2023
% package.

% List of affiliations: The first argument should be a (short)
% identifier you will use later to specify author affiliations
% Academic affiliations should list Department, University, City, Region, Country
% Industry affiliations should list Company, City, Region, Country

% You can specify symbols, otherwise they are numbered in order.
% Ideally, you should not use this facility. Affiliations will be numbered
% in order of appearance and this is the preferred way.
% \icmlsetsymbol{equal}{*}

\begin{icmlauthorlist}
\icmlauthor{Arka Daw}{xxx}
\icmlauthor{Jie Bu}{xxx}
\icmlauthor{Sifan Wang}{yyy}
\icmlauthor{Paris Perdikaris}{yyy}
\icmlauthor{Anuj Karpatne}{xxx}
\end{icmlauthorlist}

\icmlaffiliation{xxx}{Department of Computer Science, Virginia Tech, Blacksburg, Virginia, USA}
\icmlaffiliation{yyy}{University of Pennsylvania, Philadelphia, Pennsylvania, USA}

\icmlcorrespondingauthor{Arka Daw}{darka@vt.edu}

% You may provide any keywords that you
% find helpful for describing your paper; these are used to populate
% the "keywords" metadata in the PDF but will not be shown in the document
\icmlkeywords{Physics-informed Neural Networks, Partial Differential Equations (PDEs), AI for Science}
\vskip 0.3in
]

% this must go after the closing bracket ] following \twocolumn[ ...

% This command actually creates the footnote in the first column
% listing the affiliations and the copyright notice.
% The command takes one argument, which is text to display at the start of the footnote.
% The \icmlEqualContribution command is standard text for equal contribution.
% Remove it (just {}) if you do not need this facility.

% \printAffiliationsAndNotice{}  % leave blank if no need to mention equal contribution
% \printAffiliationsAndNotice{\icmlEqualContribution} % otherwise use the standard text.

\begin{abstract}
Despite the success of physics-informed neural networks (PINNs) in approximating partial differential equations (PDEs), PINNs can sometimes fail to converge to the correct solution in problems involving complicated PDEs. This is reflected in several recent studies on characterizing the ``failure modes'' of PINNs,  although a thorough understanding of the connection between PINN failure modes and sampling strategies is missing. In this paper, we provide a novel perspective of failure modes of PINNs by hypothesizing that training PINNs relies on successful ``propagation'' of solution from initial and/or boundary condition points to interior points. We show that PINNs with poor sampling strategies can get stuck at trivial solutions if there are \textit{propagation failures}, characterized by highly imbalanced PDE residual fields. To mitigate propagation failures, we propose a novel \textit{Retain-Resample-Release sampling} (R3) algorithm that can incrementally accumulate collocation points in regions of high PDE residuals with little to no computational overhead. We provide an extension of R3 sampling to respect the principle of causality while solving time-dependent PDEs. We theoretically analyze the behavior of R3 sampling and empirically demonstrate its efficacy and efficiency in comparison with baselines on a variety of PDE problems. 
\end{abstract}

\input{intro}

\input{background}
\input{propagation_hypo}
\input{method}
\input{experiments}

\input{results}

\input{discussion}

\bibliography{main}
\bibliographystyle{icml2023}

\newpage
\appendix
\onecolumn

\input{appendix}

\end{document}

%% file: intro.tex
% \vspace{-3ex}
\section{Introduction}
% \vspace{-2ex}
Our understanding of physical systems in a number of domains largely relies on our ability to solve partial differential equations (PDEs), and hence, enhancing PDE solution accuracy and computational speed can yield substantial benefits.
% While the conventional approach for solving PDEs is to use numerical methods such as finite elements methods (FEM)\citep{zienkiewicz2005finite},
% analytically is often not feasible in many real-world setting, thus, we primarily rely on computationally expensive numerical methods such as finite element methods (FEM)\citep{zienkiewicz2005finite}, finite difference methods (FDM) \citep{leveque2007finite} and finite volume methods (FVM) \citep{leveque2002finite} to obtain approximate solutions for the PDE. 
Physics-informed neural networks (PINNs) \citep{raissi2019physics} represent a seminal line of work in deep learning for solving PDEs.
%\citep{wang2020understanding, wang2022and, mcclenny2020self, bu2021quadratic, jagtap2020extended, jagtap2020adaptive, wang2021learning}.
% ), machine learning shows promise in serving as the key to solving PDEs efficiently. 
% \textcolor{red}{[CITATIONS] PDEs are fundamental in scientific and engineering fields, leading to technological advancements and breakthroughs. Enhancing PDE solutions through improved accuracy and speed can benefit multiple application domains, including reducing the need for expensive experiments, accurate simulations, and optimizing complex systems.}
The basic idea of PINNs is to train a neural network to minimize errors w.r.t. the PDE solution provided at initial/boundary points of a spatio-temporal domain, as well as the PDE residuals observed over a sample of interior points, referred to as collocation points. 
Recent success with PINNs has shown significant promise across a broad range of scientific applications, from fluid dynamics \citep{rao2020physics, zhu2021machine}, medical imaging \citep{sahli2020physics, van2022physics}, material science \citep{shao2020pinn, lin2022development}, geophysics \citep{yang2021revisit, voytan2020wave} and climate modeling \citep{lutjens2021pce}.
Despite the success of PINNs, it is known that PINNs  sometimes fail to converge to the correct solution in problems involving complicated PDEs, as reflected in several recent studies on characterizing the ``failure modes'' of PINNs \citep{wang2020understanding, wang2022and, krishnapriyan2021characterizing}. Many of these failure modes are related to the susceptibility of PINNs in getting stuck at trivial solutions acting as poor local minima, due to the unique optimization challenges of PINNs. In particular, training PINNs is different from conventional deep learning problems as we only have access to the correct solution on the initial and/or boundary points, while for all interior points, we can only compute PDE residuals. Also, minimizing PDE residuals does not guarantee convergence to a correct solution since there are many trivial solutions of commonly observed PDEs that  show 0 residuals. While previous studies have mainly focused on modifying network architectures or balancing loss functions during PINN training, the effect of sampling collocation points on avoiding failure modes of PINNs has been largely overlooked. Although some previous approaches have explored the effect of sampling strategies on PINN training \citep{wang2022is,lu2021deepxde}, they either suffer from large computation costs or fail to converge to  correct solutions, empirically demonstrated in our results.

In this work, we present a novel perspective of failure modes of PINNs by postulating the propagation hypothesis: ``in order for PINNs to avoid converging to trivial solutions at interior points, the correct solution must be \textit{propagated} from the initial/boundary points to the interior points.'' When this propagation is hindered, PINNs can get stuck at trivial solutions that are difficult to escape, referred to as the \textit{propagation failure} mode.
% ,  termed the ``\emph{propagation hypothesis},'' 
% Thus for PINNs to converge to the optimal solution, we rely on the propagation of correct solution available on the initial and/or boundary points to the interior points which can be quite far away.
This hypothesis is motivated from a similar behavior observed in numerical methods where the solution of the PDE at initial/boundary points are iteratively propagated to interior points using finite differencing schemes \citep{leveque2007finite}. 
% , where the solution at a known region is propagated to nearby locations using a discrete approximation.

We show that propagation failures in PINNs are characterized by highly imbalanced PDE residual fields, making it difficult to adequately represent the high residual regions in the set of collocation points at every iteration. This motivates us to develop sampling strategies that dynamically focus on collocation points from high residual regions during PINN training. This is related to the idea of {local-adaptive mesh refinement} used in FEM \citep{zienkiewicz2005finite} to refine the computational mesh in regions with high errors.

We propose a novel \emph{Retain-Resample-Release sampling} (R3) strategy that can accumulate collocation points in high PDE residual regions, thereby dynamically emphasizing on these skewed regions as we progress in training iterations. We theoretically show that R3 \textit{retains} points from high residual regions if they persist over iterations (Retain Property) and \textit{releases} points if they have been resolved by PINN training (Release Property), while maintaining non-zero representation of points \textit{resampled} from a uniform distribution over the entire domain (Resample Property). We also provide a causal extension of our proposed R3 sampling algorithm (Causal R3) that can explicitly encode the strong inductive bias of causality in propagating the solution from initial points to interior points over training iterations, when solving time-dependent PDEs. We empirically evaluate the performance of R3 in multiple benchmark PDE problems. We show the R3  and Causal R3 are able to mitigate propagation failure modes and converge to the correct solution with significantly smaller sample sizes as compared to baseline methods, while incurring negligible computational overhead. 
% We also demonstrate the ability of R3 to solve a particularly hard PDE problem---solving 2D Eikonal equations for complex arbitrary surface geometries.%, where baseline methods fail to converge.

The novel contributions of our work are as follows: (1) We provide a novel perspective for characterizing failure modes in PINNs by postulating the ``Propagation Hypothesis’’ and empirically demonstrate how regions with high-PDE residuals lead to propagation failures in PINNs. (2) We propose a novel R3 sampling algorithm to adaptively sample collocation points in PINNs that shows superior prediction performance empirically with little to no computational overhead compared to existing methods. (3) We theoretically prove the three key properties of R3 sampling: Retain, Resample, and Release properties.

%% file: background.tex
\vspace{-1ex}
\section{Background and Related Work}
\label{sec:related}
% \vspace{-1ex}
\textbf{Physics-Informed Neural Networks (PINNs).} The basic formulation of PINN \citep{raissi2017physics1} is to use a neural network $f_{\theta}(x, t)$ to infer the forward solution $u$ of a non-linear PDE:

\vspace{-5ex}

\begin{align}
     &u_t + \mathcal{N}_x[u]=0, \: x \in \mathcal{X}, t \in [0, T]; \nonumber \\
     &u(x, 0) = h(x), \: x \in \mathcal{X}; \\
     &u(x, t) = g(x, t), \: t \in [0, T], x \in \partial \mathcal{X} \nonumber
\end{align}
\vspace{-4ex}

where $x$ and $t$ are the space and time coordinates, respectively, $\mathcal{X}$ is the spatial domain, $\partial \mathcal{X}$ is the boundary of spatial domain, $T$ is the time horizon, and $\mathcal{N}_x$ is the non-linear differential operator. The PDE is enforced on the entire spatio-temporal domain ($\Omega = \mathcal{X} \times [0, T]$) on a set of collocation points $\{\mathbf{x_r}^i = (x_r^i, t_r^i)\}_{i=1}^{N_r}$ by computing the PDE residual ($\mathcal{R}(x, t)$) and the corresponding PDE Loss ($\mathcal{L}_r$) as follows:  

\vspace{-5ex}

\begin{align}
    &\mathcal{R}_{\theta}(x, t) = \frac{\partial}{\partial t}f_{\theta}(x, t) - \mathcal{N}_x[f_{\theta}(x, t)] \\ \vspace{-3ex}
    &\mathcal{L}_r(\theta) = \mathbb{E}_{\mathbf{x_r} \sim \mathcal{U}(\Omega)}[\mathcal{R}_{\theta}(\mathbf{x_r})^2] \approx  \frac{1}{N_r}\sum_{i=1}^{N_r}[\mathcal{R}_{\theta}(x^i_r, t^i_r)]^2
\end{align}
\vspace{-4ex}

where $\mathcal{L}_r$ is the expectation of the squared PDE Residuals over collocation points sampled from a uniform distribution $\mathcal{U}$. 
% In practice, we approximate the expectation on a finite set of collocation points using a Monte-Carlo Estimate.
PINNs approximate the solution of the PDE by optimizing the following overall loss function $\mathcal{L}= \lambda_{r} \mathcal{L}_r(\theta) + \lambda_{bc} \mathcal{L}_{bc}(\theta) + \lambda_{ic} \mathcal{L}_{ic}(\theta)$,
% \begin{align}
%     \mathcal{L}(\theta) = \lambda_{r} \mathcal{L}_r(\theta) + \lambda_{bc} \mathcal{L}_{bc}(\theta) + \lambda_{ic} \mathcal{L}_{ic}(\theta)
% \end{align}
where $\mathcal{L}_{ic}$ and $\mathcal{L}_{bc}$ are the mean squared loss on the initial and boundary data respectively, and $\lambda_{r}, \lambda_{ic}, \lambda_{bc}$ are hyperparameters that control the interplay between the different loss terms. Although PINNs can be applied to inverse problems, i.e., to estimate PDE parameters from observations, we only focus on forward problems in this paper.

% What is the objective of PINNs?
% What are the loss functions? How is the expected PDE residual over $\Omega$ computed? 
\textbf{Prior Work on Characterizing Failure Modes of PINNs.} Despite the popularity of PINNs in approximating PDEs, several works have emphasized the presence of failure modes while training PINNs. One early work \citep{wang2020understanding} demonstrated that imbalance in the gradients of multiple loss terms could lead to poor convergence of PINNs, motivating the development of Adaptive PINNs. Another recent development \citep{wang2022and} made use of the Neural Tangent Kernel (NTK) theory to indicate that the different convergence rates of the loss terms can lead to training instabilities. Large values of PDE coefficients have also been connected to possible failure modes in PINNs \citep{krishnapriyan2021characterizing}. In another line of work, the tendency of PINNs to get stuck at trivial solutions due to poor initializations has been demonstrated theoretically in \cite{wong2022learning} and empirically in \cite{rohrhofer2022understanding}. In all these works, the effect of sampling collocation points on PINN failure modes has  largely been overlooked. Although some recent works have explored strategies to grow the representation of  collocation points with high residuals, either by 
modifying the sampling procedure \cite{wu2022comprehensive,lu2021deepxde,nabian2021efficient} or choosing higher-order $L^p$ norms of PDE loss \cite{wang2022is}.
% , they either require a prohibitively dense set of collocation points and hence are computationally expensive, or suffer from poor convergence to the correct solution as empirically demonstrated in Section \ref{sec:results}. 
In another recent line of work on Causal PINNs \citep{wang2022respecting}, it was shown that traditional approaches for training PINNs can violate the principle of causality for time-dependent PDEs. Hence, they proposed an explicit way of incorporating the causal structure in the training procedure. Further, a recent study introduces the concept of ``fixed points'' $u^*$ \citep{rohrhofer2023role} which were defined as the roots of the non-linear PDE function, i.e., $\mathcal{R}[u^*] = 0$ (trivial solution $u=0$ is a special case of fixed point). PINNs are attracted towards these ``fixed points'' during the initial training iterations, and can get trapped in local minimas, leading to premature convergence.

%% file: propagation_hypo.tex
\section{Propagation Hypothesis} %[need rephrasing]
\label{sec:prop_hypo}
% High PDE residual in small regions, high skewness and kurtosis.

\begin{figure}[t]
    % \vspace{-3ex}
    \centering
    \includegraphics[width=0.49\textwidth]{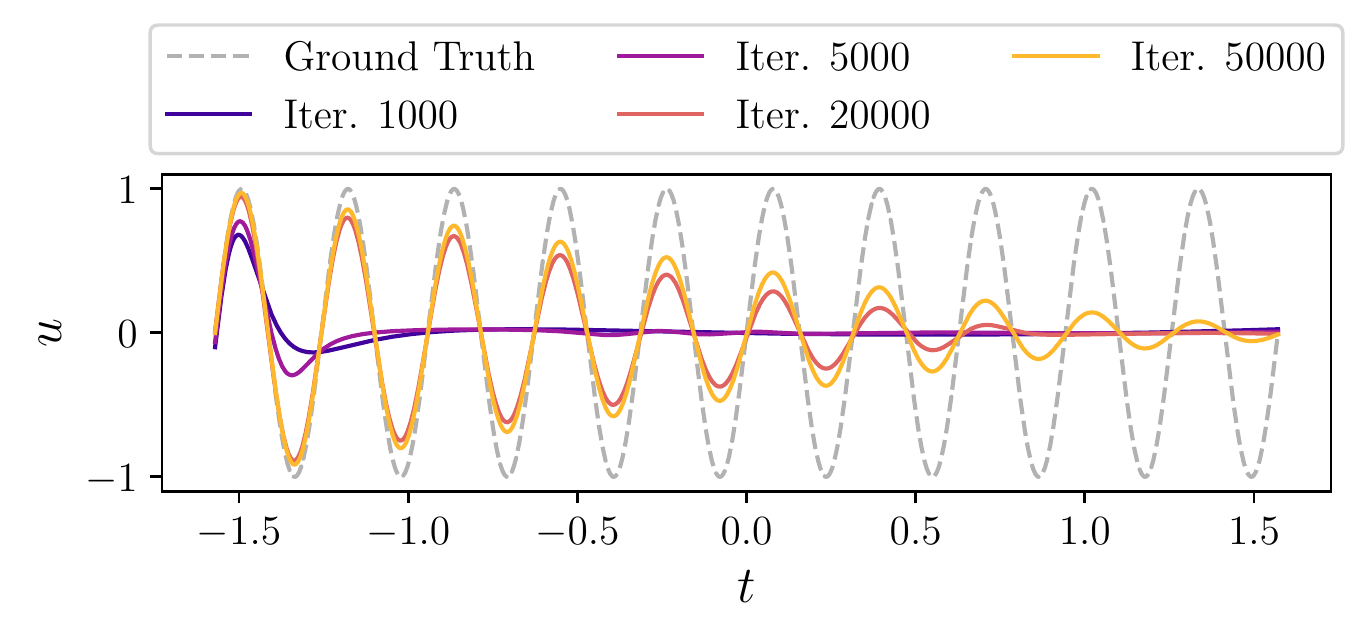}
    \vspace{-5ex}
    \caption{PINN solutions for a simple ODE: $u_{xx} + k^2 u = 0$ ($k =20$)  with the analytical solution, $u=A\sin(kx) + B\cos(kx)$. The boundary condition was set to $u(-\pi/2)=0$, and the PINN is trained with 1000 equispaced collocation points. We can see smooth \textit{propagation} of the correct solution from the boundary point at $x=0$ to interior points ($x>0$) as we increase training iterations. }
    \label{fig:propagation_simple}
    \vspace{-3ex}
\end{figure}

\vspace{-1ex}
\textbf{What is Unique About Training PINNs?} Training PINNs presents fundamentally different optimization challenges than those encountered in conventional deep learning problems. In a conventional supervised learning problem,
% Physics-informed Neural Networks (PINNs) fall into an unique category which is distinct from conventional machine learning models to a great extent. In supervised machine learning setups, 
the correct solution for every training sample is known and the training samples are considered representative of the test samples such that the trained model can easily be extrapolated on closely situated test samples. 
% we have access to labeled training samples where the correct solution for every sample is known. Further, the training samples are generally considered representative of the test samples such that the trained ML function can easily be extrapolated on closely situated test samples. 
% On the other hand,  
% input-output pairs which is used to evaluate the performance of the model. During the training procedure, an optimal model would fit a smooth function that minimizes the empirical loss on the training data. Therefore, by observing the empirical loss over the training epochs we are often able to speculate if the model is on the right track towards convergence. 
However, in the context of PINNs, we only have access to the ``correct'' solution of the PDE on the initial and/or boundary points, while not having any labels for the interior points in the spatio-temporal domain $\Omega$. Note that the interior points in $\Omega$ can be quite far away from the initial/boundary points, making extrapolation difficult. Further, training PINNs involves minimizing the PDE residuals over a set of collocation points sampled from $\Omega$. However, minimizing PDE residuals alone is not sufficient to ensure convergence to the correct solution, since there may exist many trivial solutions of a PDE showing very small residuals. For example, $u(x,t) = 0$ is a trivial solution for any homogeneous PDE, which a neural network is likely to get stuck at in the absence of correct solution at initial/boundary points. Another unique aspect of training PINNs is that minimizing PDE residuals requires computing the gradients of the output w.r.t. $(x,t)$ (e.g., $u_x$ and $u_t$).
% , which describe the rules of local variations and curvatures. 
Hence, the solution at a collocation point is affected by the solutions at nearby points leading to local propagation of solutions.
% However, since differential equations only describe the rules of local variations and curvatures, the correct prediction on one point relies on its neighboring points to be accurate.

% This local propagation of information in PINNs occurs as the optimization of the PDE residual depend on the computation of the gradients of the output (e.g., $u_x$, $u_t$, etc.). Thus, the gradient update for a single collocation point would affect the nearby regions leading to the propagation of information.

\textbf{Propagation Hypothesis.} In light of the unique properties of PINNs, we postulate that in order for PINNs to converge to the ``correct'' solution, the correct solution must propagate from the initial and/or boundary points to the interior points as we progress in training iterations. We draw inspiration for this hypothesis
% ,  termed the ``\emph{propagation hypothesis},'' 
from a similar behavior observed in numerical methods for solving PDEs, where the solution of the PDE at initial/boundary points are iteratively propagated to interior points using finite differencing schemes \citep{leveque2007finite}. Figure \ref{fig:propagation_simple} demonstrates the propagation hypothesis of PINNs for a simple ordinary differential equation (ODE).

% This propagation behavior can be even seen in case of a simple homogeneous ODE: $u_{xx} + k^2u = 0$ with the analytical solution $u=A\sin(kx) + B\cos(kx)$ (shown in Figure \ref{fig:propagation_simple}).
% that the flow of information from the initial (and boundary) points occur during the training of PINNs which is critical for the convergence to the optimal PDE solution. 
% This is analogous to classical numerical methods, 
% where the solution of the PDE at the initial point $t=0$ is used to approximate the solution at $t = \Delta t$, and then recursively applying the same procedure to mimic the evolution of the spatio-temporal system. This propagation behavior can be even seen in case of a simple homogeneous ODE: $u_{xx} + k^2u = 0$ with the analytical solution $u=A\sin(kx) + B\cos(kx)$ (shown in Figure \ref{fig:propagation_simple}).

% similar to numerical methods, This propagation behavior can be even seen in case of a simple sin example simple ODE

% long range propagation

\textbf{Propagation \textit{Failure}: Why It Happens and How to Diagnose?}
% While the propagation from the initial (and boundary) occurs during training, we expect the PDE residuals to be well behaved and gradually get smaller and smaller once the propagation is complete.
As a corollary of the propagation hypothesis, PINNs can fail to converge to the correct solution if the solution at initial/boundary points is \textit{unable to propagate} to interior points during the training process. We call this phenomenon the ``\textit{propagation failure}'' mode of PINNs. This is likely to happen if some collocation points start converging to trivial solutions before the correct solution from initial/boundary points is able to reach them. 
% A potential reason for propagation failure is that 
% In a traditional PINN setup that relies on optimizing the PDE residuals on the uniformly distributed collocations points in the spatio-temporal domain, the ``correct'' information flow can occur from the initial (and boundary) conditions to the nearby regions, making them highly likely to learn the optimal PDE solution. 
% While the collocation points that are further away from the initial (and boundary) points could easily learn a trivial solution before the ``correct'' information flow form the initial points reach them.
Such collocation points would also propagate their trivial solutions to nearby interior points, leading to a cascading effect in the learning of trivial solutions over large regions of $\Omega$ and further hindering the propagation of the correct solution from initial/boundary points. 
% Since the PDE residuals of trivial solutions are also close to 0, these regions act as barriers in the propagation of the correct solution from initial/boundary points, leading to poor convergence.

To \textit{diagnose} propagation failures, note that the PDE residuals are expected to be low over both types of regions: regions that have converged to the correct solution and regions that have converged to trivial solutions. However, the boundaries of these two types of regions would show a sharp discontinuity in solutions, leading to very high PDE residuals in very narrow regions. 
% when the propagation of correct solution is hindered, $\Omega$ gets partitioned into two regions: (1) region where the ``correct'' solution has propagated from the initial/boundary points, (2) region that has converged to trivial solutions. Since the PDE residuals are small for both these regions, thus creating a prominent boundary with relatively high PDE residuals. We call this phenomena as the ``propagation failure'' mode of the PINNs, where small high density PDE residuals acts as a barrier to the flow of ``correct'' information leading to a trivial solution collapse. 
A similar phenomenon is observed in numerical methods where sharp high-error regions disrupt the evolution of the PDE solution at surrounding regions, leading to cascading of errors.
We use the imbalance of high PDE residual regions as a diagnosis tool for characterizing propagation failure modes in PINNs.

% These collocation points that fit the trivial solution would then further propagate the trivial solution to other collocation points in their vicinity. This would ultimately partition the input domain into two broad sub-domains: (1) region where the ``correct'' PDE solution was propagated from the initial (and boundary) condition, (2) region where the collocations points fell into the local minima by overfitting a trivial solution. The key characteristic of these two regions is that they both have very low PDE residuals, thus creating a prominent boundary with relatively high PDE residuals. We call this phenomena as the ``propagation failure'' mode of the PINNs, where small high density PDE residuals acts as a barrier to the flow of ``correct'' information leading to a trivial solution collapse. This phenomenon is similar to numerical methods where a sharp high error region would disrupt the evolution of the spatio-temporal system and thus leading to erroneous approximations of the subsequent regions.

% This puts PINNs in an extremely vulnerable position where the optimization could easily fall into a local minima leading to trivial or spurious solutions for scenarios where the ``correct'' solution was not properly propagated from the initial/boundary of the PDE domain.
%  supervisions from data-driven loss provide direct guidance to approaching the correct solution. However, since differential equations only describe the rules of local variations and curvatures, the correct prediction on one point relies on its neighboring points to be accurate.
 
To demonstrate propagation failure, let us consider an example PDE for the convection equation: 
$\frac{\partial u}{\partial t} + \beta \frac{\partial u}{\partial x} = 0,  u(x, 0) = h(x)$, 
% \begin{align}
%     \frac{\partial u}{\partial t} + \beta \frac{\partial u}{\partial x} = 0, \: x \in \Omega, t \in [0, T]; \: \: \: u(x, 0) = h(x), \: x \in \Omega
% \end{align}
where $\beta$ is the convection coefficient and $h(x)$ is the initial condition (see Appendix \ref{sec:details_pde} for details about this PDE). 
% We consider the case of constant $\beta$ with periodic boundary conditions. 
% Under these conditions, 
In a previous work \citep{krishnapriyan2021characterizing}, it has been shown that PINNs fail to converge for this PDE for $\beta > 10$. We experiment with two cases, $\beta = 10$ and $\beta = 50$, in Figure \ref{fig:propagation_hypothesis}. We can see that the PDE loss steadily decreases with training iterations for both these cases,  but the relative error w.r.t. the ground-truth solution only decreases for $\beta=10$, while for $\beta=50$, it remains flat. This suggests that for $\beta=50$, PINN is likely getting stuck at a trivial solution that shows low PDE residuals but high errors. To diagnose this failure mode, we plot two additional metrics in Figure \ref{fig:propagation_hypothesis} to measure the imbalance in high PDE residual regions: Fisher-Pearson's coefficient of Skewness \citep{kokoska2000crc} and Fisher's Kurtosis \citep{kokoska2000crc} (see Appendix \ref{sec:metric_details} for computation details). High Skewness indicates lack of symmetry in the distribution of PDE residuals while high Kurtosis indicates the presence of a heavy-tail. 
% A schematic of the skewness and the kurtosis of different distributions are shown using the Figure [PLACEHOLDER]. 
For $\beta=10$, we can see that both Skewness and Kurtosis are relatively small across all iterations, indicating absence of imbalance in the residual field. However, for $\beta=50$, both these metrics shoot up significantly as the training progresses, which indicates the formation of very high residuals in very narrow regions---a characteristic feature of the propagation failure mode. Figure \ref{fig:convection_heatmap} confirms that this indeed is the case by visualizing the PINN solution and PDE residual maps. Specifically, we observe that for $\beta = 50$, the PDE residual is very high at the bottom left corner, which overlaps with the region where we observe the hindrance in propagation of the PDE solution. We see similar trends of propagation failure for other values of $\beta > 10$ (see Appendix \ref{sec:prop_failure_viz}). We also illustrate in Appendix \ref{sec:prop_failure_viz_ks_equation} that regions with high error rates contribute to the propagation failures in PINNs for another complex PDE, specifically the Kuramoto-Sivashinksy Equation (chaotic regime).
% , where the When both the skewness and kurtosis increase, both the PDE residual of the region increases as well as the region becomes smaller and smaller. This observation is in agreement with out hypothesised ``propagation failure'' mode. We further demonstrate similar failure modes for larger values of $\beta$ in the appendix [PLACEHOLDER]. 
% This is in accord with our ``propagation bottleneck'' effect. 

% we visualize the error map of the predicted PINN solution of PINN along with the 2D-heatmap of the PDE residuals for both of these cases as shown in Figure \ref{fig:convection_heatmap}. We observe that indeed for the $\beta=30$ case, the PINN fails to propagate the ``correct'' solution from the initial (and boundary) points, and the PDE residuals have small high error regions although the mean of the squared PDE residual is in the order of $5e-5$ (which is similar to $\beta=10$ case). Although the PINN relies on the PDE residuals to propagate the information from the initial points to the entire domain, in practice it just optimizes the PDE residuals on the entire grid with equal importance to every point. Hence, when regions of high residual density forms during the training, the PINN at times ignores them if they are relatively small compared to the entire spatio-temporal domain and overly focus on minimizing the PDE residuals on every other region, thereby minimizing the mean of the squared PDE residual. However, this small region can disrupt the propagation and lead to trivial (or spurious) solutions. 

\begin{figure}[ht]
    \vspace{-1ex}
    \centering
    \includegraphics[width=0.5\textwidth]{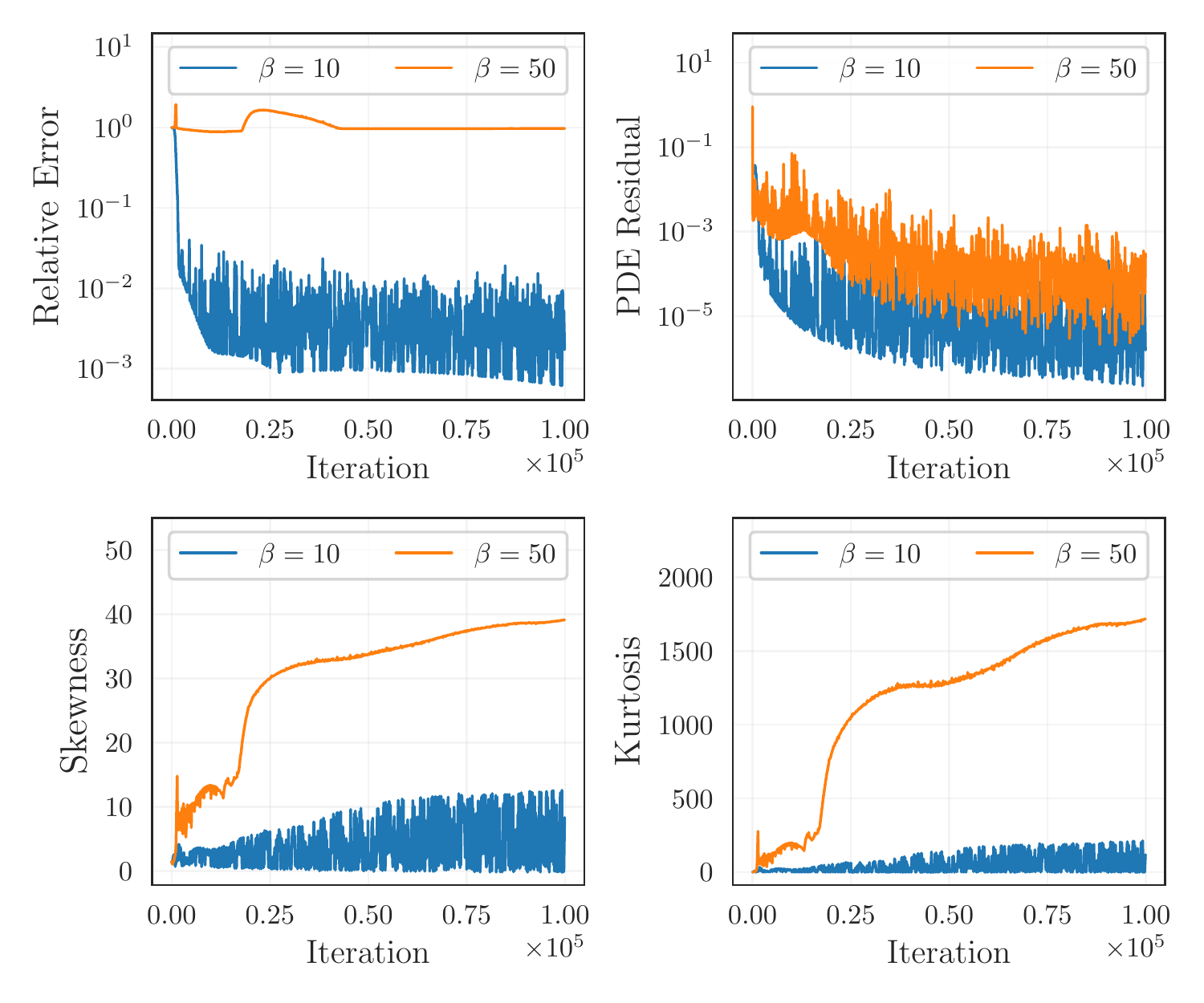}
    \vspace{-2ex}
    \caption{Demonstration of propagation failure using skewness and kurtosis while solving the convection equation with $\beta=50$.}
    \label{fig:propagation_hypothesis}
    % \vspace{-1ex}
\end{figure}

\begin{figure*}[ht]

    \centering
    \includegraphics[width=1.01\textwidth]{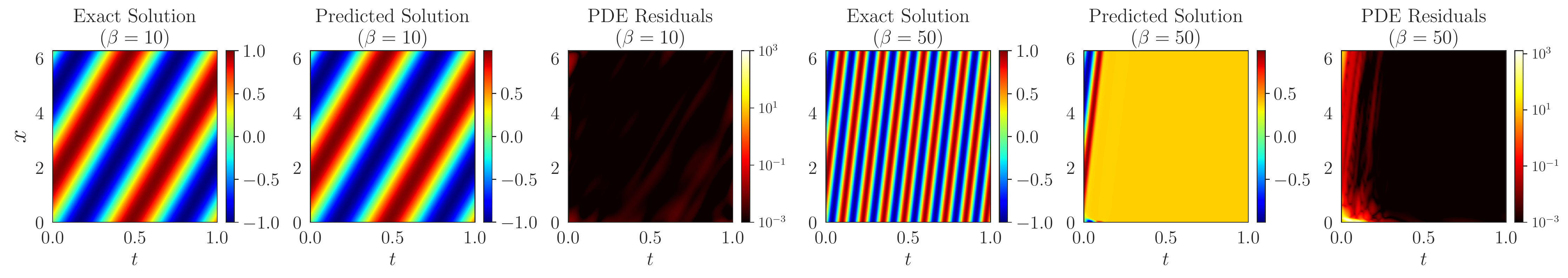}
    \vspace{-3ex}
    \caption{Demonstration of the Exact PDE solution, the predicted PDE solution, and the PDE Residual Field for the convection equation with $\beta = 10$ (first three figures) and $\beta = 50$ (last three figures) respectively.}
    \label{fig:convection_heatmap}
    \vspace{-1ex}
\end{figure*}

%% file: method.tex
\begin{remark}
\emph{``Propagation failures''} in PINNs are characterized by skewed regions with very high residuals acting as ``barriers'' in the flow of information. Algorithms that can selectively focus on high residual regions during training can potentially break such barriers and thus offer some respite.
\end{remark}
% Thus, we remark that 

% \textbf{Potential Remedies for Mitigating Propagation Failures.} 
\textbf{Prior Works Focussing on High Residual Regions.} Recently, a number of algorithms have been proposed to increase the importance of collocation points from high residual regions in PINN training, which can be broadly grouped into two categories.  The first category includes methods that alter the sampling strategy such that a point $\mathbf{x_r}$ is picked up as a collocation point with probability proportional to its residual, i.e.,  $p(\mathbf{x_r}) \propto |\mathcal{R}(\mathbf{x_r})|^k$, where $k \geq 1$. This includes residual-based adaptive refinement (RAR) methods \citep{lu2021deepxde} and its variants \cite{wu2022comprehensive,nabian2021efficient}, where a dense set of collocation points $\mathcal{P}_{dense}$ is maintained to approximate the continuous residual field $\mathcal{R}(\mathbf{x})$, and points with high residuals are regularly added from $\mathcal{P}_{dense}$ to the set of collocation points every $K$ iterations according to a sampling function. 

A second line of work was recently proposed in \cite{wang2022is}, where higher-order $L^p$ norms of the PDE loss (e.g., $L^\infty$) were advocated to be used in the training of PINNs in contrast to the standard practice of using $L^2$ norms, to ensure stability of learned solutions in control problems involving high-dimensional PDEs. Note that by using higher-order $L^p$ norms, we are effectively increasing the importance of collocation points from high residual regions in the PDE loss, thereby having a similar effect as increasing their representation in the sample of collocation points (see Proposition \ref{theorem:connection_lp_sampling} for more details).

\textbf{Limitations of Prior Work.}
There are two main challenges faced by the prior work described above that limit their effectiveness in mitigating propagation failures in PINNs. \\
% Even though the potential remedies discussed above can provide additional emphasis on the high residual regions during training, they have their own limitations. 
(1) \textbf{High computational complexity:} 
Sampling methods such as RAR and its variants require using a dense set of collocation points $\mathcal{P}_{dense}$ (typically with $100\text{k}\sim1\text{M}$ points spread uniformly across the entire domain) to locate high residual regions, such that points from high residual regions can be added to the training set every $K$ iterations. This increases the computational cost in two ways. First, computing the PDE residuals on the entire dense set is very expensive. Second, the size of the training set keeps growing every $K$ iterations, further increasing the training costs at later iterations. See Appendix  \ref{sec:complexity} for a detailed analysis of the computational complexity of RAR based methods. \\
% Sampling methods such as RAR and its variants require maintaining a dense set of collocation points $\mathcal{P}_{dense}$ (typically with $100\text{K} \sim 1\text{M}$ points) and computing PDE residuals at every point in $\mathcal{P}_{dense}$ every $K$ iterations. This introduces a severe computational overhead for adding new points, which also necessitates the resampling period $K$ to be high (See Appendix  \ref{sec:complexity} for a detailed analysis of the computational complexity of RAR based methods). 
% Note that it is essential for the $\mathcal{S}(\mathbf{x_r})$ set to be dense (typically $100k \sim 1M$) to ensure that the PDE residuals across the entire domain $\Omega$ are well-represented (an important criterion to sample collocation points from high density regions that can be very narrow as shown in Section \ref{sec:prop_hypo}). 
% This additionally limits these algorithms to re-evaluate their adaptive strategies every $N$ iterations, resulting in an unreasonable assumption that the PDE residuals should remain almost constant during this time for the adaptive sampling to be correct. Otherwise, the distribution of the collocation points would not reflect the changes in the PDE residuals that occurred during these $N$ iterations.
(2) \textbf{Difficulty selecting ideal values of $k$ or $p$:} While increasing the value of $k$ in sampling-based methods and $p$ in $L^p$ norm-based methods influences a greater skew in the sample set towards points with higher residuals, the optimal values of $k$ or $p$ are generally unknown for an arbitrary PDE. Additionally, choosing $L^\infty$ loss (or equivalently, only sampling collocation points with highest residuals) may not ideal for the training dynamics of PINNs, as it can lead to oscillatory behavior between different peaks of the PDE residual landscape, while forgetting to retain the solution over other regions. We empirically demonstrate in Section \ref{sec:results}, both sampling-based methods such as RAR and its variants as well as $L^\infty$ PDE loss-based methods suffer from poor performance in converging to the correct solution for complex PDE problems due to the aforementioned limitations.

\section{Proposed R3 Sampling Algorithm}
\label{sec:evo}
% \vspace{-2ex}

We propose a novel ``Retain-Resample-Release'' (R3) Sampling algorithm to overcome the limitations in prior works while effectively prioritizing the high residual regions during training to mitigate propagation failures. This algorithm has three key properties which are defined as follows: \\
(1) \textbf{\emph{Retain Property}}: To break propagation barriers, R3 focus on \textit{retaining} points from high residual regions such that the set of collocation points starts accumulating in these regions until the PINN training process eventually resolves them. This is similar to starting from $L^2$-norm and dynamically increasing the order of $L^p$ norm if high residual regions persist over iterations, in contrast to using a fixed value of $p$ as is done in prior works. \\
(2) \textbf{\emph{Resample Property}}: At every iteration, R3 ensures that the set of collocation points contains non-zero support of points resampled from a uniform distribution over the entire domain $\Omega$. This is done so that the collocation points do not collapse solely to high residual regions, which is one of the limitations of using $L^\infty$ loss. \\
% To maintain such an uniform density, some collocation points can be resampled from an uniform distribution.
(3) \textbf{\emph{Release Property}}: Upon sufficient minimization of a high residual region through PINN training, collocation points that were once accumulated from the region are \textit{released}, such that R3 can focus on minimizing other high residual regions in later iterations. Note that if the retained points are not released, the set of collocation points would keep growing, thus increasing the computational costs and biasing the sampling to these regions even at later iterations (which is one of the limitations of RAR-based methods).

Along with satisfying the above properties, the R3 algorithm also incurs little to no computational overhead in sampling collocation points from high residual regions, in contrast to prior works. Specifically, we are able to add points from high residual regions without maintaining a dense set of collocation points, $\mathcal{P}_{dense}$, and by only observing the  residuals over a small set of $N_r$ points at every iteration.

Algorithm \ref{alg:evo} shows the pseudo-code of our proposed R3 sampling strategy.  
% First, we define a ``residual function'' $\mathcal{F}(\mathbf{x_r})$ for every collocation point $\mathbf{x_r}$ such that points with higher residuals are retained in the next iteration. 
At iteration 0, we start with an initial population $\mathcal{P}_0$ of $N_r$ points sampled from a uniform distribution. At iteration $i$, in order to update the population to the next iteration, we first construct the ``retained population'' $\mathcal{P}^r_i$ comprising of points from $\mathcal{P}_i$ falling in high residual regions. Specifically, we define the ``residual function'' $\mathcal{F}(\mathbf{x_r})$ for every collocation point $\mathbf{x_r}$ as the absolute value of the PDE residual of $\mathbf{x_r}$, i.e., $\mathcal{F}(\mathbf{x_r}) = |\mathcal{R}(\mathbf{x_r})|$. We compute $\tau_{i}$ as the expected value of residual function over all points in $\mathcal{P}_i$. Points in $\mathcal{P}_i$ with residual function values greater than $\tau_{i}$ are then considered to be part of the retained population $\mathcal{P}^r_i$, i.e.,  $\mathcal{P}^r_i \leftarrow \{\mathbf{x}_r^j: \mathcal{F}(\mathbf{x}_r^j)> \tau_{i} \}$. The remainder of collocation points in $\mathcal{P}_i$ with $\mathcal{F}(\mathbf{x}_r^j) \leq \tau_{i}$ are dropped and replaced with points re-sampled from a uniform distribution, thus constructing the ``re-sampled population'' $\mathcal{P}^s_i \leftarrow \{ \mathbf{x_r}^j: \mathbf{x_r}^j \sim \mathcal{U}(\Omega) \}$. The retained and re-sampled population are then merged to generate the population for the next iteration, $\mathcal{P}_{i+1}$. Note that the size of the population at every iteration is constant (equal to $N_r$). We schematically show the dynamics of collocation points in R3 Algorithm over training iterations in Figure \ref{fig:evosample_schematic}. 

\begin{algorithm}[ht]
\caption{Proposed R3 Sampling Algorithm For PINN}\label{alg:evo}
\small
\begin{algorithmic}[1]
\STATE Sample the initial population $\mathcal{P}_0$ of $N_r$ collocations point $\mathcal{P}_0 \leftarrow \{\mathbf{x_r}\}_{i=1}^{N_r}$ from a uniform distribution $\mathbf{x_r}^i \sim \mathcal{U}(\Omega)$, where $\Omega$ is the input domain ($\Omega = [0, T] \times \mathcal{X}$).
\FOR{i = 0 to max\_iterations - 1}
  \STATE Compute the residual function of  collocation points $\mathbf{x_r} \in \mathcal{P}_{i}$ as $\mathcal{F}(\mathbf{x_r}) = |\mathcal{R}(\mathbf{x_r})|$.
  \STATE Compute the threshold $\tau_{i} = \frac{1}{N_r}\sum_{j=1}^{N_r} \mathcal{F}(\mathbf{x_r}^j)$
  \STATE Select the retained population $\mathcal{P}^r_{i}$ such that $\mathcal{P}^r_{i} \leftarrow \{\mathbf{x_r}^j : \mathcal{F}(\mathbf{x_r}^j) > \tau_{it}\}$
  \STATE Generate the re-sampled population $\mathcal{P}^s_{i} \leftarrow \{\mathbf{x_r}^j: \mathbf{x_r}^j \sim \mathcal{U}(\Omega)\}$, s.t. $|\mathcal{P}^s_i| + |\mathcal{P}^r_i| = N_r $
  \STATE Merge the two populations $\mathcal{P}_{i+1} \leftarrow \mathcal{P}^r_{i} \cup \mathcal{P}^s_{i}$
\ENDFOR
\end{algorithmic}
\end{algorithm}

\begin{figure*}
    % \vspace{-1ex}
    \centering
    \includegraphics[width=0.8\textwidth]{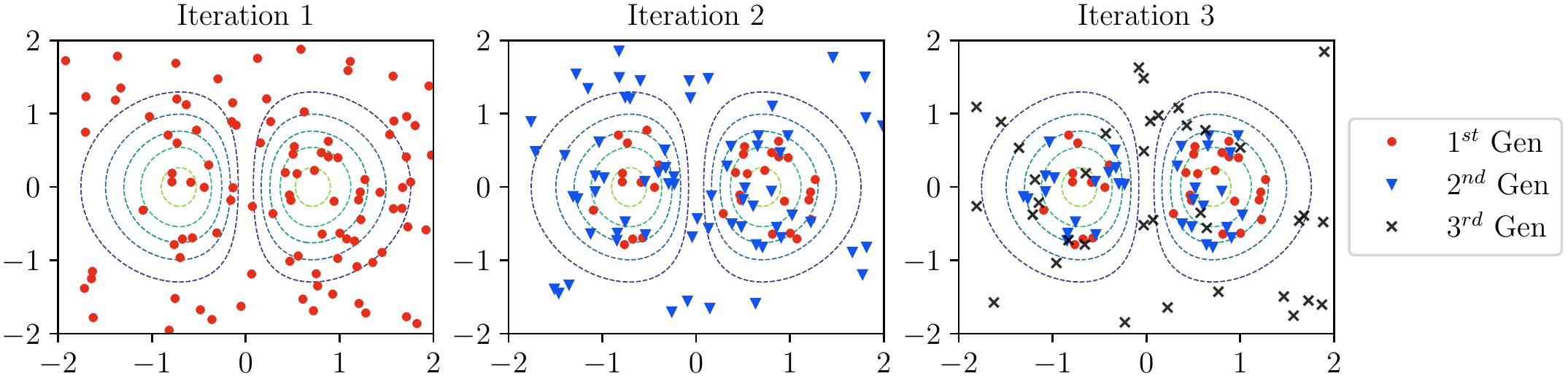}
    \caption{Schematic to describe our proposed R3 sampling algorithm, where collocation points are incrementally accumulated in regions with high PDE residuals (shown as contour lines).}
    \vspace{-3ex}
    \label{fig:evosample_schematic}
\end{figure*}

% \vspace{-2ex}

% Intuitively, Evo biases the sampling at every iteration to include more points from  high PDE residual regions, $\Omega_{high}$, and thus increase the representation of 
% $\mathbb{E}_{\Omega_{high}}[\mathcal{L}_r(\theta)]$ in the overall PDE residual loss, in contrast to fixed sampling and dynamic random sampling.
 
% See Appendix for detailed theoretical analysis and intuition behind the dynamics of Evolutionary Sampling algorithm.

% \textcolor{red}{Introduce (1)Validation of the 3 properties bring in theorems, (2) MIN-MAX OPTIMIZATION concept}
% \textbf{Analysis of R3 Sampling:}
\subsection{Theoretical Validations of R3 Sampling properties} 
A unique feature of our proposed R3 sampling algorithm is that while prior works have only addressed a subset of the three sampling properties: Retain, Resample, and Release, we are able to satisfy all of them while being computationally efficient (i.e., without maintaining a dense set of collocation points). See Table \ref{tab:comp_baselines} in the Appendix for a comparison of our R3 sampling algorithm with prior works. In the following, we theoretically verify the ability of R3 to satisfy the three sampling properties.

(1) \textbf{Retain Property:} Theorem \ref{theorem:expectation_retained_pop} shows that for a fixed $\mathcal{F}(\mathbf{x})$ (e.g., when $\theta$ is kept constant), the expectation of the retained population in R3 sampling becomes maximum (equal to $L^\infty$) when the number of iterations approaches $\infty$.

\begin{theorem}[Accumulation  Dynamics Theorem]
% \vspace{-1ex}
\label{theorem:expectation_retained_pop}
    Let $\mathcal{F}_\theta(\mathbf{x}):\mathbb{R}^n \rightarrow \mathbb{R}^+$ be a fixed real-valued $k$-Lipschitz continuous objective function optimized using the R3 Sampling algorithm. Then, the expectation of the retained population  $\mathbb{E}_{\mathbf{x} \in \mathcal{P}^r}[\mathcal{F}(\mathbf{x})] \geq \max_\mathbf{x} \mathcal{F}(\mathbf{x}) - k\epsilon$ as iteration $i \rightarrow \infty$, for any arbitrarily small $\epsilon > 0$.
    \vspace{-1ex}
\end{theorem}

The proof of Theorem \ref{theorem:expectation_retained_pop} can be found in  Appendix \ref{sec:grad_accum_proof}. 
This demonstrates the \textit{Retain property} of R3 sampling as points from high residual regions would keep accumulating in the retained population and make its expectation maximal if the residual function is kept fixed. Note that since the PINN optimizer is also minimizing the residuals at every iteration, we would not expect the residual function to be fixed unless a high residual region persists over a long number of iterations. In fact, points from a high residual region would keep on accumulating until they are resolved by the PINN optimizer and thus eventually released from $\mathcal{P}^r$. 
% \textcolor{red}{Gradual accumulation property alleviates the need to select a specific $k$ or $p$ as in prior-work owing to its adaptive behavior from L-2 to L-inf as empirically shown in Figure \ref{fig:optim}.}

(2) \textbf{Resample Property:} Theorem \ref{theorem:uniform_background} states that the size of the resampled population $\mathcal{P}^s$ is always greater zero. As a result, there is always some background density of collocation points sampled from a uniform distribution, preventing the R3 Sampling algorithm from collapsing to high residual regions.

\begin{theorem}[Non-Empty Theorem]
% \vspace{-1ex}
\label{theorem:uniform_background}
    For any population $\mathcal{P}$ generated at an arbitrary iteration of R3 sampling, the re-sampled population is always non-empty, i.e., $|\mathcal{P}^s|>0$.
    \vspace{-1ex}
\end{theorem}

The proof of Theorem \ref{theorem:uniform_background} can be found in  Appendix \ref{sec:grad_accum_proof} stated as a part of Lemma \ref{lemma:pop_properties}. 

(3) \textbf{Release Property:} Let us define that for an arbitrary iteration $i$, a collocation point $\mathbf{x_r} \in \mathcal{P}_i$ is ``\emph{sufficiently minimized}'' if its residual function is less than $\tau_i$, i.e., $\mathcal{R}_\theta(\mathbf{x_r}) \leq \mathbb{E}_{\mathbf{x_r} \in \mathcal{P}_i}[\mathcal{R}_\theta(\mathbf{x_r})]=\tau_i$ (where $\tau_i$ is the threshold at the $i$-th iteration). Then, by definition, $\mathbf{x_r}$ will belong to the ``non-retained population'', i.e., it will be ``released'' from the population $\mathcal{P}_i$ and replaced by a new point re-sampled from the uniform distribution. Thus, R3 sampling satisfies the ``\emph{Release Property}'' by design.

\textbf{Connections to Numerical Methods:}
Note that R3 sampling shares a similar motivation as \textit{local-adaptive mesh refinement} methods developed for Finite Element Methods (FEM) \citep{zienkiewicz2005finite}, where the goal is to preferentially refine the computational mesh used in numerical methods based on localization of the errors. It is also related to the idea of \textit{boosting} in ensemble learning where training samples with  larger errors are assigned higher weights of being picked in the next epoch, to increasingly focus on high error regions \citep{schapire2003boosting}. 

\textbf{Connections to Evolutionary Algorithms:} Our proposed R3 sampling strategy also shares similarity with evolutionary algorithms that are used for modeling biological evolution \citep{eiben2003introduction}. The residual function $\mathcal{F}$ is equivalent to the fitness function defined for typical evolutionary algorithms. At each iteration, we retain the points with the largest fitness (i.e., points corresponding to the highest PDE residuals), which is similar to the ``survival of the fittest'' concept commonly found in evolutionary algorithms.

\subsection{Causal Extension of R3 Sampling (Causal R3)}
\label{sec:causalevo}

In problems with time-dependent PDEs, a strong prior dictating the propagation of solution  is the \textit{principle of causality}, where the solution of the PDE needs to be well-approximated at time $t$ before moving to time $t + \Delta t$. 
% Although Evo can effectively sample collocation points from high residual density regions, in case of a time-dependent PDE the additional causal structure needs to be taken in consideration. 
To incorporate this prior guidance, we present a Causal Extension of R3 (\textit{Causal R3}) that includes two modifications: (1) we develop a causal formulation of the PDE loss $\mathcal{L}_r$ that pays attention to the temporal evolution of PDE solutions over iterations, and (2) we develop a causally biased sampling scheme that respects the causal structure while sampling collocation points. We describe both these modifications in the following.

% To this end, we propose an extension of our Evolutionary Sampling framework to respect the causal inductive bias for time-dependent PDEs.

% \begin{figure}

%     \centering
%     \includegraphics[width=0.4\textwidth]{fig/results/CausalGate.pdf}
%     \vspace{-1ex}
%     \caption{Causal Gate.}
%     \label{fig:causal_gate}
% \end{figure}

\textbf{Causal Formulation of PDE Loss.} The key idea here is to utilize a simple time-dependent gate function $g(t)$ that can explicitly enforce causality by revealing only a portion of the entire time-domain to PINN training. Specifically, we introduce a continuous gate function $g(t)= (1 - \tanh(\alpha(\tilde{t} - \gamma)))/2$, where $\gamma$ is the scalar shift parameter that controls the fraction of time that is revealed to the model, $\alpha =5$ is a constant scalar parameter that determines the steepness of the gate, and $\tilde{t}$ is the normalized time, i.e., $\tilde{t} = t/T$. 
Example causal gates for different settings of $\gamma$ are provided in Appendix  \ref{sec:causal_evo_appendix}.
% Making use of the ``grid-free'' formulation of the causal gate, 
We use $g(t)$ to obtain a causally-weighted PDE residual loss as
$\mathcal{L}^{g}_r(\theta) = \frac{1}{N_r}\sum_{i=1}^{N_r}[\mathcal{R}(x^i_r,t^i_r)]^2 * g(t^i_r)$.
% \vspace{-2ex}
% \begin{align}
%     &\mathcal{L}^{g}_r(\theta) = \frac{1}{N_r}\sum_{i=1}^{N_r}[\mathcal{R}(x^i_r,t^i_r)]^2 * g(t^i_r) 
% \end{align}
% \vspace{-1ex}
We initially start with a small value of the shift parameter ($\gamma = -0.5$), which essentially only reveals a very small portion of the time domain, and then gradually increase $\gamma$ during training to reveal more portions of the time domain. For $\gamma\geq1.5$, the entire time domain is revealed. %, which is equivalent to the non-causal formulation.

\textbf{Causally Biased Sampling.} We bias the sampling strategy in R3 such that it not only favors the selection of collocation points from high residual regions but also accounts for the causal gate values at every iteration. In particular, we modify the residual function as $\mathcal{F}(\mathbf{x_r}) = |\mathcal{R}(\mathbf{x_r})| * g(t_r)$. Figure \ref{fig:causal_evo_schematic} represents a schematic describing the causally biased R3 sampling described in Section \ref{sec:causalevo} and the causal gate $g$ that is updated every iteration.

\begin{figure*}[ht]
\centering
\subfigure[Causal Gate.]{\label{fig:causal_gate} \includegraphics[scale=0.46]{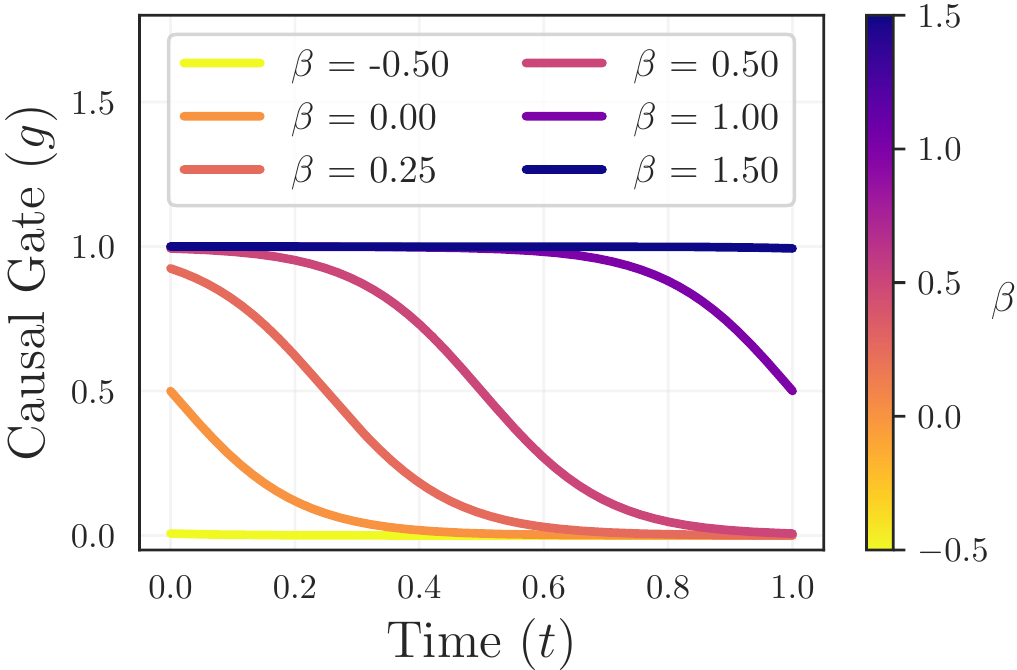}}
\subfigure[Causally R3 Sampling]{\label{fig:causal_evo_schematic} \includegraphics[scale=0.58]{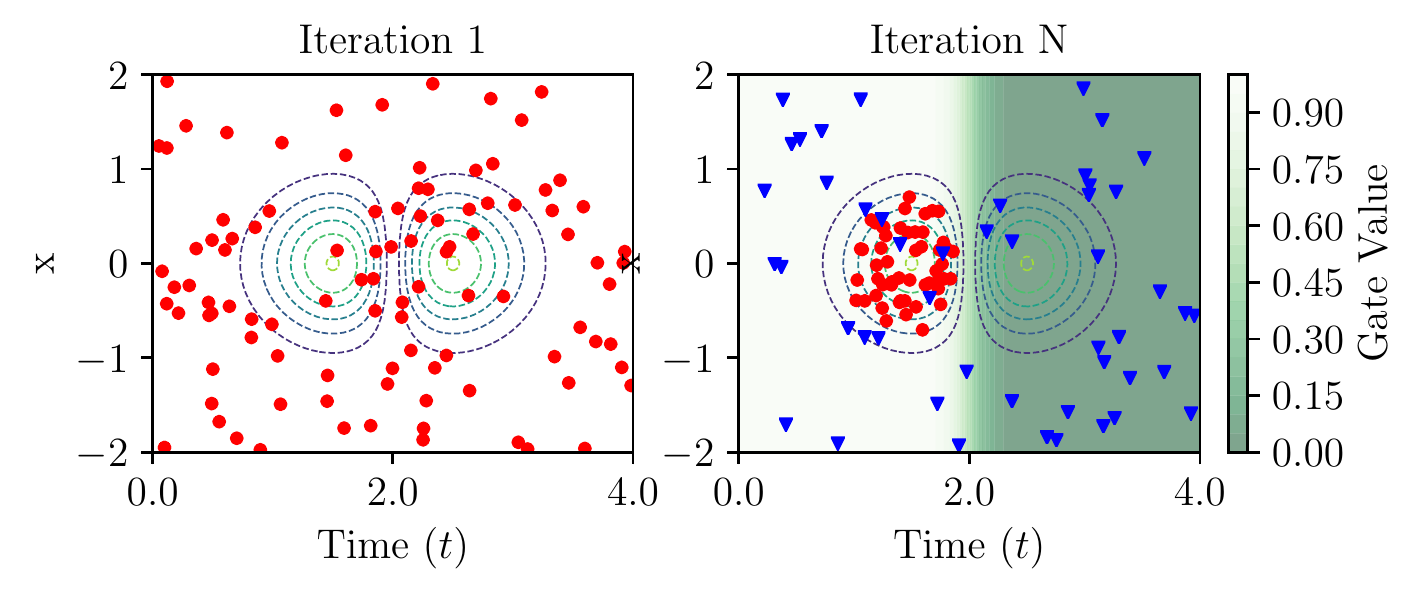}} 
\vspace{-3ex}
\caption{Causal R3 uses a time-dependent causal gate for computing PDE loss and for sampling. }
\label{fig:causal_evo}
\end{figure*}
% A schematic illustration of causally biased sampling is provided in Appendix  \ref{sec:causal_evo_appendix}. 

\textbf{How to Update $\boldsymbol{\gamma}$?}  Ideally, at some iteration $i$, we would like increase $\gamma_i$ at the next iteration \textit{only if} the PDE residuals at iteration $i$ are low. Otherwise, $\gamma_i$ should remain in its place until the PDE residuals under the current gate are minimized. To achieve this behaviour, we propose the following update scheme for $\gamma$:
    $\gamma_{i+1} = \gamma_{i} + \eta_{g} e ^ {-\epsilon \mathcal{L}^{g}_r(\theta)}$,
% \begin{align}
%     \gamma_{i+1} = \gamma_{i} + \eta_{g} e ^ {-\epsilon \mathcal{L}^{g}_r(\theta)}
% \end{align}
where $\eta_{g}$ is the learning rate and $\epsilon$ denotes tolerance that controls how low the PDE loss needs to be before the gate shifts to the right. 
% Note that the update for the beta is always positive, i.e., we fix the direction of causality or information flow. 
Since the update in $\gamma$ is inversely proportional to the causally-weighted PDE loss $\mathcal{L}^g_r$, the gate will shift slowly if the PDE residuals are large. 
% Thus, the PDE residual of the revealed region should be decreased first before the gate gets a significant update, which is in accord with ideal behavior discussed above. 
Also note that increasing $\gamma$ also increases the value of $g(t)$ for all collocation points, thus increasing the causally-weighted PDE loss and slowing down gate movement. Upon convergence, $\gamma$ attains a large value such that the entire time domain is revealed.
% This cyclic behavior of the causal gate adaptively updates the $\gamma$ during training.

%% GRADIENT CLIPPING in APPENDIX

\begin{table*}[bp]
\setlength\tabcolsep{3pt} % let LaTeX figure out amount of intercolumn whitespace
\fontsize{8.1pt}{9}\selectfont
\caption{\fontsize{10pt}{11}\selectfont Relative $\mathcal{L}_2$ errors (in $\%$) of comparative methods over benchmark PDEs with $N_r = 1000$.\\}
\label{tab:results}
\begin{footnotesize}
\begin{center}
% \vspace{-3ex}
\begin{tabular}{l|cc|cc|c}\hline
                & \multicolumn{2}{c|}{Convection ($\beta = 30$)} & \multicolumn{2}{c|}{Convection ($\beta = 50$)} & Allen Cahn \\ 
Epochs.    & 100k          & 300k          & 150k          & 300k          & 200k       \\ \hline
PINN (fixed)     & $107.5\pm10.9 \%$        & $107.5\pm10.7 \%$        & $108.5\pm6.38 \%$        & $108.7\pm6.59 \%$        & $69.4\pm4.02 \%$    \\ 
PINN (dynamic) & $2.81\pm1.45 \%$        & $1.35\pm0.59 \%$        & $24.2\pm23.2 \%$        & $56.9\pm9.08 \%$        & $0.77\pm0.06 \%$     \\ 
Curr Reg \citep{krishnapriyan2021characterizing} & $63.2\pm9.89 \%$        & $2.65\pm1.44 \%$        & $48.9\pm7.44 \%$        & $31.5\pm16.6 \%$        & -          \\ 
CPINN (fixed) \cite{wang2022respecting}   & $138.8\pm11.0 \%$        & $138.8\pm11.0 \%$        & $106.5\pm10.5 \%$        & $106.5\pm10.5 \%$        & $48.7\pm19.6 \%$    \\ 
CPINN (dynamic) \cite{wang2022respecting}   & $52.2\pm43.6 \%$        & $23.8\pm45.1 \%$        & $79.0\pm5.11 \%$        & $73.2\pm8.36 \%$        & $1.5\pm0.75 \%$    \\ 
RAR-G \cite{lu2021deepxde}      &       $10.5 \pm 5.67 \%$  &  $2.66 \pm 1.41 \%$      &  $65.7 \pm 17.0 \%$       & $43.1 \pm 28.9 \%$       & $25.1 \pm 23.2 \%$    \\ 
RAD \cite{nabian2021efficient}      &   $3.35 \pm 2.02 \%$      &  $1.85 \pm 1.90 \%$      &   $66.0 \pm 1.55 \%$      &  $64.1 \pm 1.98 \%$      &   $0.78 \pm 0.05 \%$    \\ 
RAR-D  \cite{wu2022comprehensive}     &  $67.1 \pm 4.28 \%$       &  $32.0 \pm 25.8 \%$      &    $82.9 \pm 5.99 \%$     & $75.3 \pm 9.58 \%$       &   $51.6 \pm 0.41 \%$    \\
$L^\infty$       &  $66.6\pm 2.35 \%$       &   $41.2 \pm 27.9 \%$      &    $76.6 \pm 1.04 \%$     & $75.8 \pm 1.01 \%$      &   $1.65 \pm 1.36 \%$    \\ 
R3 (ours)      & $\mathbf{1.51\pm0.26} \%$        & $0.78\pm0.18 \%$        & $6.03\pm6.99 \%$        & $\mathbf{1.98\pm0.72} \%$        & $0.83\pm0.15 \%$      \\
Causal R3 (ours)       & $2.12\pm0.67 \%$        & $\mathbf{0.75\pm0.12} \%$       & $\mathbf{5.99\pm5.25} \%$        & $2.28\pm0.76 \%$       & $\mathbf{0.71\pm0.007} \%$    \\ \hline
\end{tabular}
\end{center}
\end{footnotesize}
\end{table*}

%% file: experiments.tex
% \section{Baselines}
% \begin{itemize}
%     \item PINN (fixed)
%     \item PINN (random sampling) -> (optional either 1 or 2)
%     \item Curriculum Regularization (for PDE 1 only)
%     \item Causal PINN
%     \item PINN + Importance Sampling
%     \item Evo Sampling
%     \item Causal Evo Sampling
% \end{itemize}

% \section{PDEs to compare}

% \begin{itemize}
%     \item Convection Equation ($\beta$ = 30, 50, 70) -> (earlier papers compare 10-30-40)
%     \item Allen Cahn Equation
%     % \item 1D- Reaction Diffusion (optional)
%     \item KS Equation 3-cases [regular, 2xchaotic]
%     \item Navier Stokes (from Paris)
% \end{itemize}

% \section{Experiments}
% \begin{itemize}
%     \item Comparison of the results vs baselines ($N_f$ = 10000) and train for 200k (limit the training time) [maybe just train for 300k]
%     \item Show Faster convergence of Evo Sampling (plot Training: [L, $L_bc$, $L_ic$, $L_f$], Eval: [Rel L2, or Mean Abs Error]
%     \item Collocation Sample Efficiency (Repeat Experiment 1 with Nf = 1000)
%     \item Sampling vs Weighting (Compare with M5: Importance sampling)
%     \item Evidence for convergence of Evolutionary Sampling to the distribution $p \propto L_f$
%     \item Show some failure cases where there are “small” regions that have large errors [may not be true if PDE has forcing terms: Then trivial solution x=0 is not possible perhaps.]
%     \item Computational Overhead: Time Required Experiment
%     \item Coarse grid to fine scale
% \end{itemize}

%% file: results.tex
% \vspace{1ex}
\section{Results}
\label{sec:results}
% \vspace{-2ex}

\begin{figure}[t]
    \label{fig:comparison_pinn_evo}
    \centering
    \includegraphics[width=0.49\textwidth]{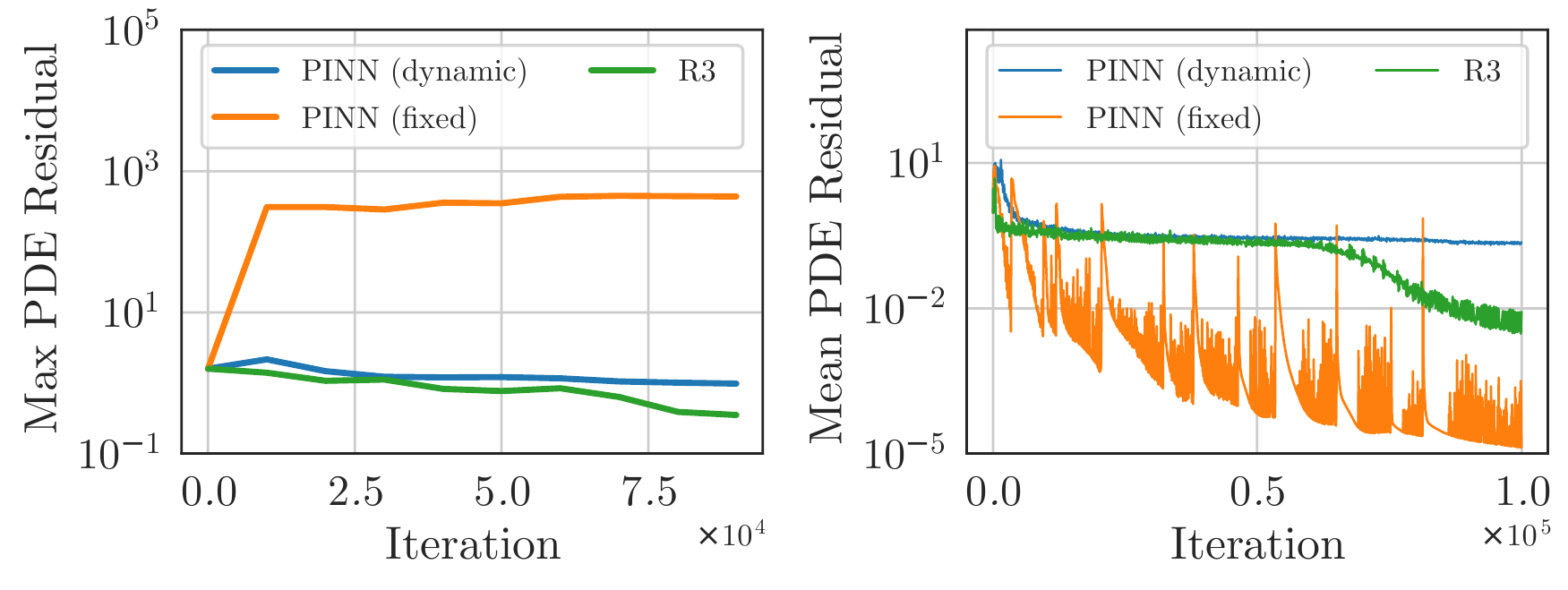}
    \vspace{-4ex}
    \caption{Comparison of Max PDE Residuals, and Mean PDE Residuals over training iterations for PINN-fixed, PINN-dynamic, and R3 Sampling for convection equation with $\beta=50$.}
    \vspace{-5ex}
\end{figure}

\textbf{Experiment Setup.} We perform experiments over three benchmark PDEs that have been used in existing literature to study failure modes of PINNs. In particular, we consider two time-dependent PDEs: convection equation (with $\beta = 30$ and $\beta = 50$) and Allen Cahn equation, and one time-independent PDE: the Eikonal equation for solving signed distance fields for varying input geometries. For PINNs and CausalPINNs \cite{wang2022respecting} we have defined two separate baselines, (1) fixed: which uses a fixed set of uniformly sampled collocation points, (2) dynamic: a simple modification where the collocation points are dynamically sampled from a uniform distribution every iteration (see Appendix  \ref{sec:analyse_dynamic} for more details). We introduce another baseline $L^\infty$ where we sample top $N_r$ collocation points at every iteration from a dense set $\mathcal{P}_{dense}$ to approximate $L^\infty$ norm. The other baselines have been listed in Table \ref{tab:results}. For every benchmark PDE, we use the same neural network architecture and hyper-parameter settings across all baselines and our proposed methods, wherever possible. Details about the PDEs, experiment setups, and hyper-parameter settings are provided in Appendix \ref{sec:hyperparam_setting}. All of our codes and datasets are available here \footnote{\url{https://github.com/arkadaw9/r3_sampling_icml2023}}.

\begin{figure}[t]
\centering
\subfigure[Auckley Function]{\label{fig:auckley_dynamiclp_norm} \includegraphics[scale=0.38]{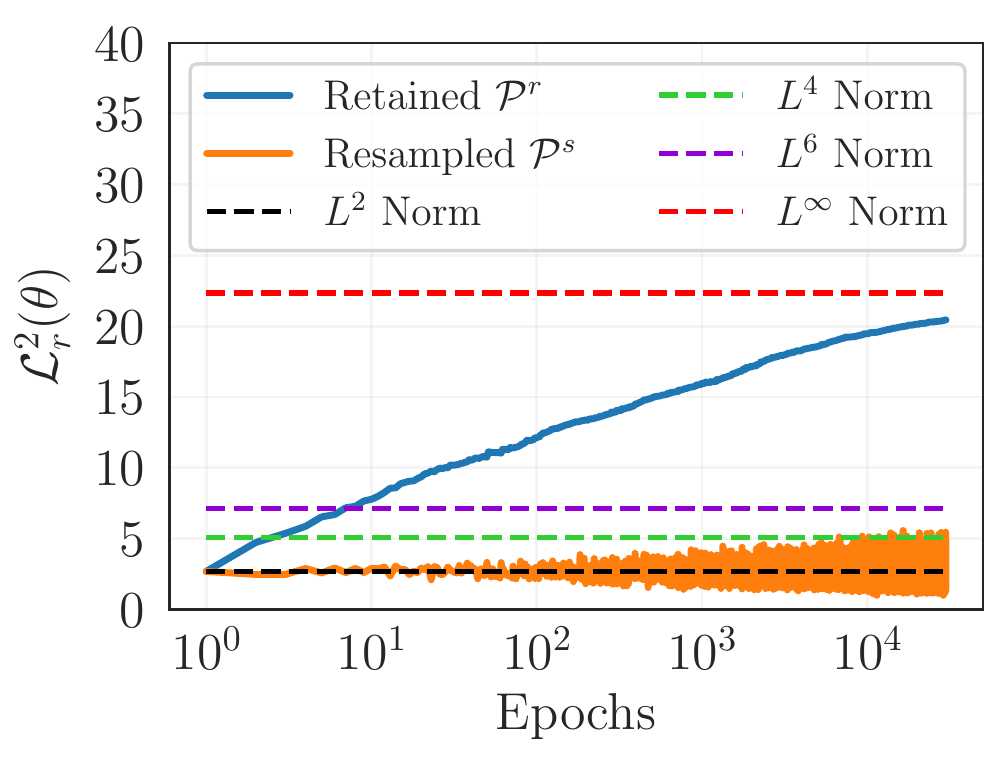}}
\subfigure[Michaelewicz Function]{\label{fig:michael_dynamiclp_norm} \includegraphics[scale=0.38]{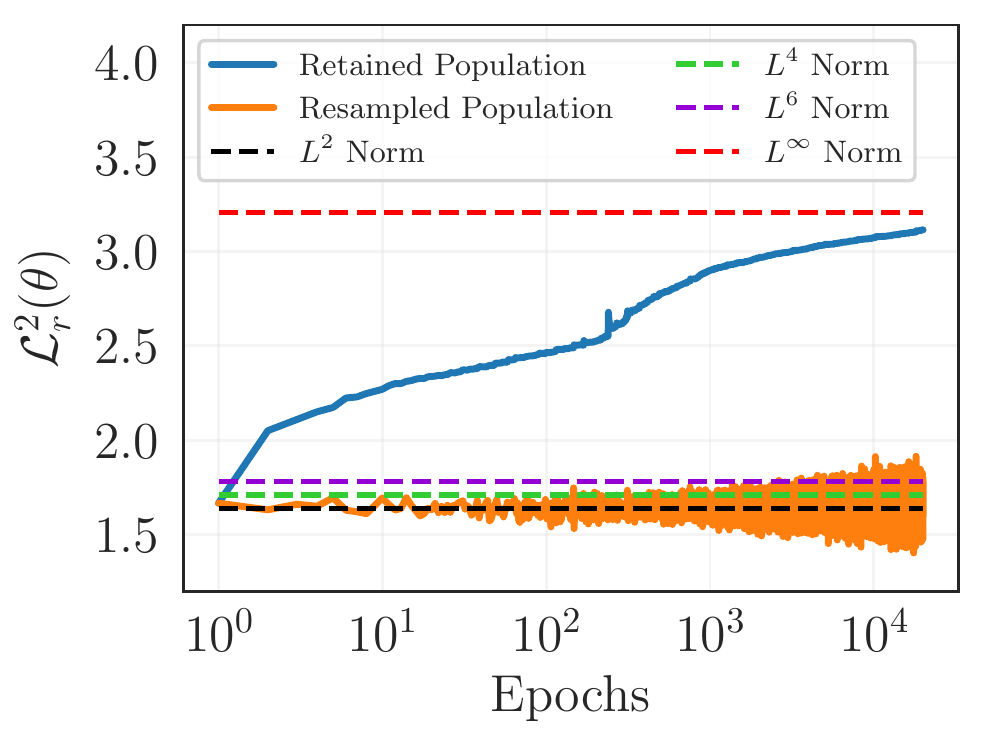}} 
\vspace{-2ex}
\caption{Demonstration of accumulation behavior of R3 sampling on two sample objective functions using the PDE loss of $\mathcal{P}^r$ going from $L^2$ to $L^\infty$ with iterations for two sample objective function.}
\label{fig:optim}
\vspace{-1ex}
\end{figure}

\textbf{Comparing Prediction Performance.} 
% To further illustrate the efficacy of dynamic sampling approaches, we use the Allen Cahn Equation along with the Convection Equation presented earlier. In particular, we consider the 1-D Allen Cahn Equation: $u_t - 0.0001 u_{xx} + 5 u^3 - 5u=0, \: \: t \in [0, 1], x \in [-1, 1]$, with the initial condition $u(x, 0)=x^2\cos(\pi x)$ and periodic boundary conditions. 
Table \ref{tab:results} shows the relative $\mathcal{L}_2$ errors (over 5 random seeds) of PDE solutions obtained by comparative methods w.r.t. ground-truth solutions for different time-dependent PDEs  when $N_r$ is set to 1K. We particularly chose a small value of $N_r$ to study the effect of small sample size on PINN performance (note that the original formulations of baseline methods used very high $N_r$). We can see that while PINN-fixed fails to converge for convection ($\beta=30$) and Allen Cahn equations (admitting very high errors), PINN-dynamic shows significantly lower errors. 
% This demonstrates that even a very simple baseline of dynamic sampling can be effective in mitigating propagation failures. 
However, for complex PDEs such as convection ($\beta=50$), PINN-dynamic is still not able to converge to low errors. We also see that cPINN-fixed shows high errors across all PDEs when $N_r=1000$. This is likely because the small size of collocation samples are insufficient for cPINNs to converge to the correct solution. As we show later, cPINN indeed is able to converge to the correct solution when $N_r$ is large. Performing dynamic sampling with cPINN shows some reduction in errors, but it is still not sufficient for convection ($\beta=50$) case. All other baseline methods including Curr Reg, RAR-based methods, and $L^\infty$ fail to converge on most PDEs and show worse performance than even the simple baseline of PINN-dynamic. On the other hand, our proposed approaches (R3 and Causal R3) consistently show the lowest errors across all PDEs. 
% we can see that shows PINNs and Causal-PINNs (cPINNs) with dynamic sampling consistently outperform their counterparts with fixed grid collocation points. As one of the dynamic sampling methods, we observe that evolutionary sampling produces significantly more accurate and the causal modification further improves it in most of the results. 
% Our results support the observation made in \cite{wang20222} that directly optimizing $L^p$ loss for large values of $p$ can suffer from poor convergence due to the stiffness in the loss landscape.
% \textbf{Effect of Sampling on Mitigating Propagation Failures.} 
Figure \ref{fig:comparison_pinn_evo} shows that R3 and PINN-dynamic are indeed able to mitigate \textit{propagation failures} for convection ($\beta=50$), by maintaining low values of max PDE residuals across all iterations, in contrast to PINN-fixed. Note that R3 is further able to reduce the mean PDE residuals better than PINN-dynamic.  
% We also see that Evo shows a larger decline in the mean PDE loss after 50K iterations compared to PINN-dynamic while keeping skewness and Kurtosis low, indicating faster propagation of the correct solution from initial/boundary points.
% PDE residual regions of PINN-fixedl  leads to better results. By plotting the skewness and kurtosis w.r.t. training epochs, It is particular interesting that by merely dynamic sampling the collocation points, the PINNs is able to alleviate the high skewness in the PDE residual field. Though the mean PDE residual keeps decreasing during training, the maximum PDE residual is very high, thus the PINN with fixed grid collocations gradually forms a highly skewed PDE residual field that never gets resolved in training. We further observe that the mean PDE residual stagnates for the PINN with dynamic sampling while Evo is able to further optimize the PDE residuals without falling into trivial solutions.
% \textbf{Accumulation of Collocation Points in Evolutionary Sampling:} 
Additional visualizations of the evolution of samples in R3, Caual R3, and RAR-based methods across iterations are provided in Appendix  \ref{sec:additional_results}. Sensitivity of RAR-based methods to hyper-parameters is provided in Appendix \ref{sec:sensitivity_rar}. %, where we show that the collocation points are getting accumulated in high PDE residual regions. 

\textbf{Accumulation Behavior of R3 Sampling:} To demonstrate the ability of R3 Sampling to accumulate high residual points in the retained population $\mathcal{P}^r$ (or equivalently, focus on higher-order $L^p$ norms of PDE loss), we consider optimizing two fixed objective functions: the Auckley function and Michaelewicz Function in Figure \ref{fig:optim} (see Appendix \ref{sec:auckley} and \ref{sec:michael} for details of these function). We can see that at iteration 1, the expected loss over  $\mathcal{P}^r$ is equal to the $L^2$ norm of PDE loss over the entire domains. As training progresses, the expected loss over $\mathcal{P}^r$ quickly reaches higher-order $L^p$ norms, and approaches $L^\infty$ at very large iterations. This confirms the gradual accumulation behavior of R3 sampling as theoretically stated in Theorem \ref{theorem:expectation_retained_pop}. Addition visualizations of the dynamics of R3 sampling for a number of test optimization functions are provided in Appendix \ref{sec:test_optim_func}.

\textbf{Sampling Efficiency:} Figure \ref{fig:sample} shows the scalability of R3 sampling to smaller sample sizes of collocation points, $N_r$. Though all the baselines demonstrate similar performances when $N_r$ is large ($>10K$), only R3 and Causal R3 manage to maintain low errors even for very small values of $N_r = 100$, showing two orders of magnitude improvement in sampling efficiency. Note that the sample size $N_r$ is directly related to the compute and memory requirements of training PINNs. We also show that R3 sampling
and Causal R3 show faster convergence speed than  baseline methods for both convection and Allen Cahn equations (see Appendix \ref{sec:convergence_speed} for details). Further, in Figure \ref{fig:ks_equation_vary_Nr_main} we compare sample efficiency of CPINN, R3 sampling, and Causal R3 on the Kuramoto-Sivashinsky (KS) equation (regular case) (for details refer to Appendix \ref{sec:details_pde}). We can see that both R3 and Causal R3 show improvements over CPINN when the number of collocation points is small ($N_r = 128$). As the number of collocation points is increased, Causal R3 shows better performance than R3, as it incorporates an additional prior of causality along with satisfying the three properties of R3. Additional results on other cases of KS Equations including chaotic behavior are provided in Appendix \ref{sec:ks_results}.

% Since Evo and Causal Evo can work in ultra-low sample size regimes, they are particularly useful for solving high-dimensional PDEs where sampling efficiency can greatly impact running time and memory consumption, e.g., in 3D fluid dynamics problems.
% Convergence speed.

% Figure \ref{fig:ks_equation_vary_Nr} compares the performance of CPINN, R3 sampling, and Causal R3 on the KS equation (regular case) as a function of the number of collocation points used in PINN training. We can see that both R3 and Causal R3 show improvements over CPINN when the number of collocation points is small ($N_r = 128$). As the number of collocation points is increased, Causal R3 shows better performance than R3, as it incorporates an additional prior of causality along with satisfying the three properties of R3. Overall, Causal R3 mostly performs better than both CPINN and R3 across different training set sizes. Note that these curves have been obtained using a single run of every method due to the computational cost of training for the KS equations, and having multiple runs for every method will help to quantify the variance in these results.

\begin{figure}[ht]
\centering
\subfigure[Convection Equation \quad \quad ($\beta =50$, Iter: 300k)]{\label{fig:convection_sample} \includegraphics[scale=0.40]{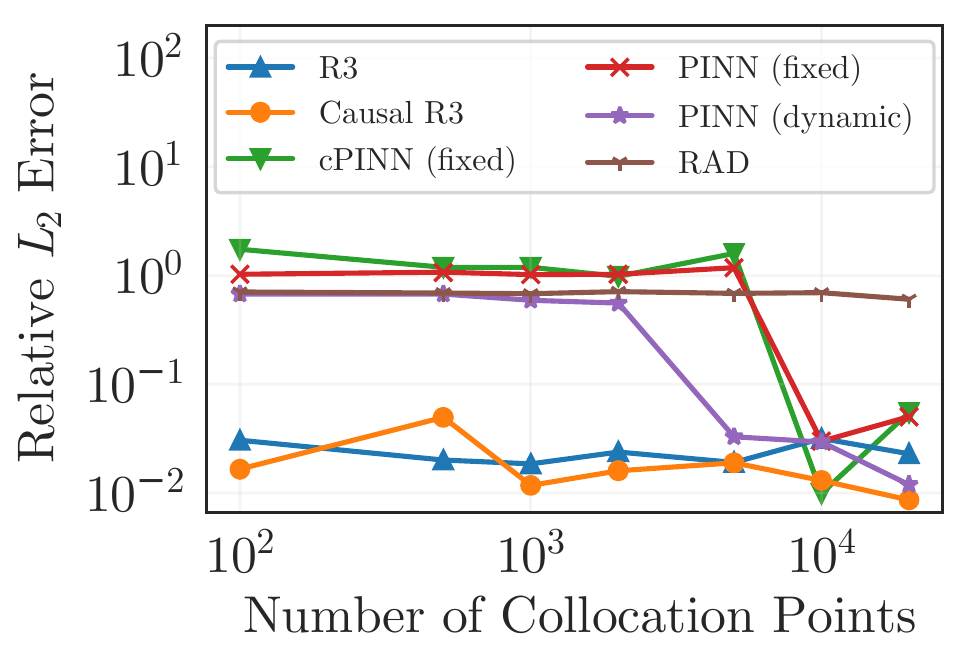}}
\subfigure[Allen Cahn Equation \quad \quad \quad  (Iter: 200k)]{\label{fig:AC_sample} \includegraphics[scale=0.38]{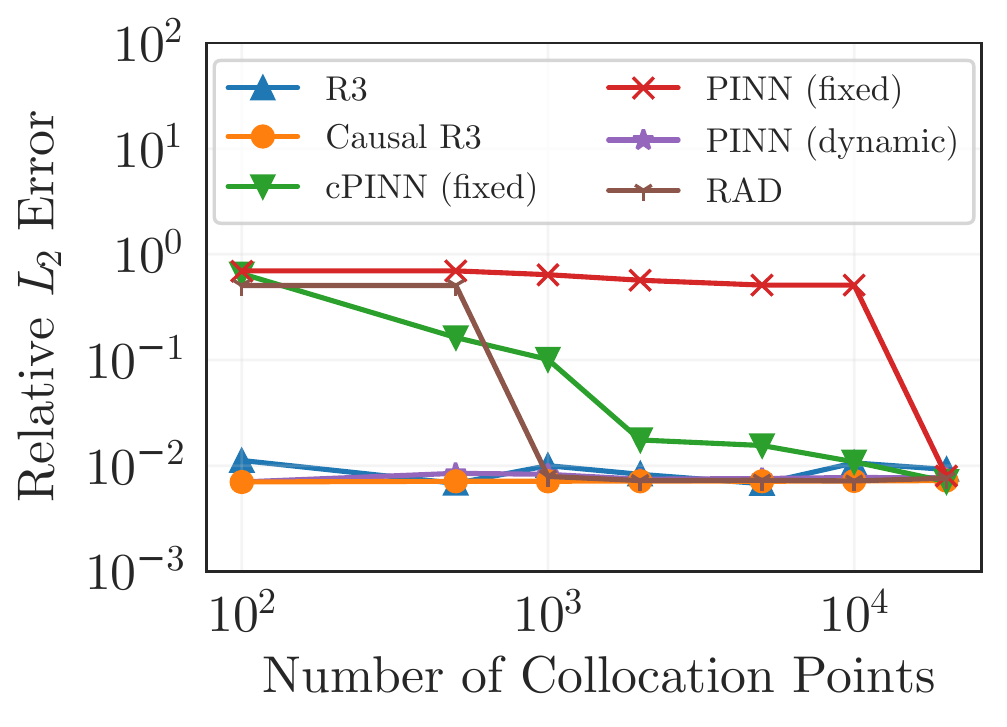}} 
\vspace{-3ex}
\caption{Demonstrating Sample Efficiency of R3 sampling and Causal R3 at low values of $N_r$.}
\label{fig:sample}
\end{figure}

\begin{figure}[ht]
    \centering
    \includegraphics[width=0.25\textwidth]{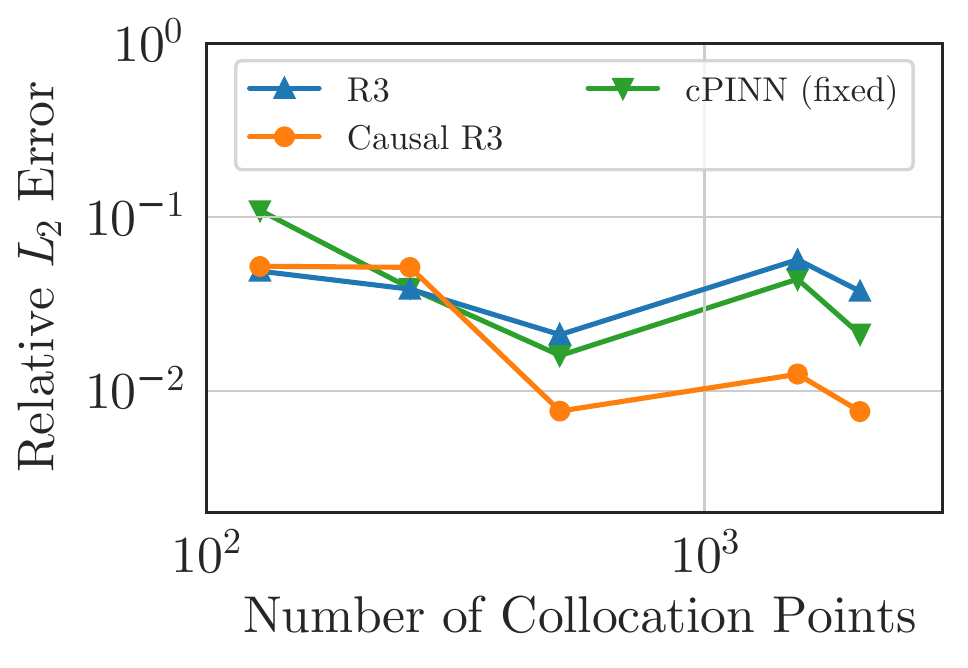}
    \vspace{-1ex}
    \caption{Comparison of R3, Causal R3 and CausalPINNs on KS Equation (regular) with varying number of collocation points.}
    \label{fig:ks_equation_vary_Nr_main}
\end{figure}

% \begin{wrapfigure}{h}{.33\textwidth}
%     \begin{minipage}{\linewidth}
%     \centering\captionsetup[subfigure]{justification=centering}
%     \includegraphics[width=\linewidth]{fig/results/Convection_SampleEfficiency.pdf}
%     \subcaption{\centering Convection ($\beta =50$, Iter: 300k)}
%     \label{fig:convection_sample}\par\vfill
%     \includegraphics[width=\linewidth]{fig/results/AllenCahn_SampleEfficiency.pdf}
%     \subcaption{\centering Allen Cahn (Iter: 200k)}
%     \label{fig:AC_sample}
% \end{minipage}
% \caption{Sample Efficiency of Evo and Causal Evo at low $N_r$.}\label{fig:sample}
% \vspace{-3ex}
% \end{wrapfigure}

\begin{figure}[ht]
    \centering
    \includegraphics[width=0.50\textwidth]{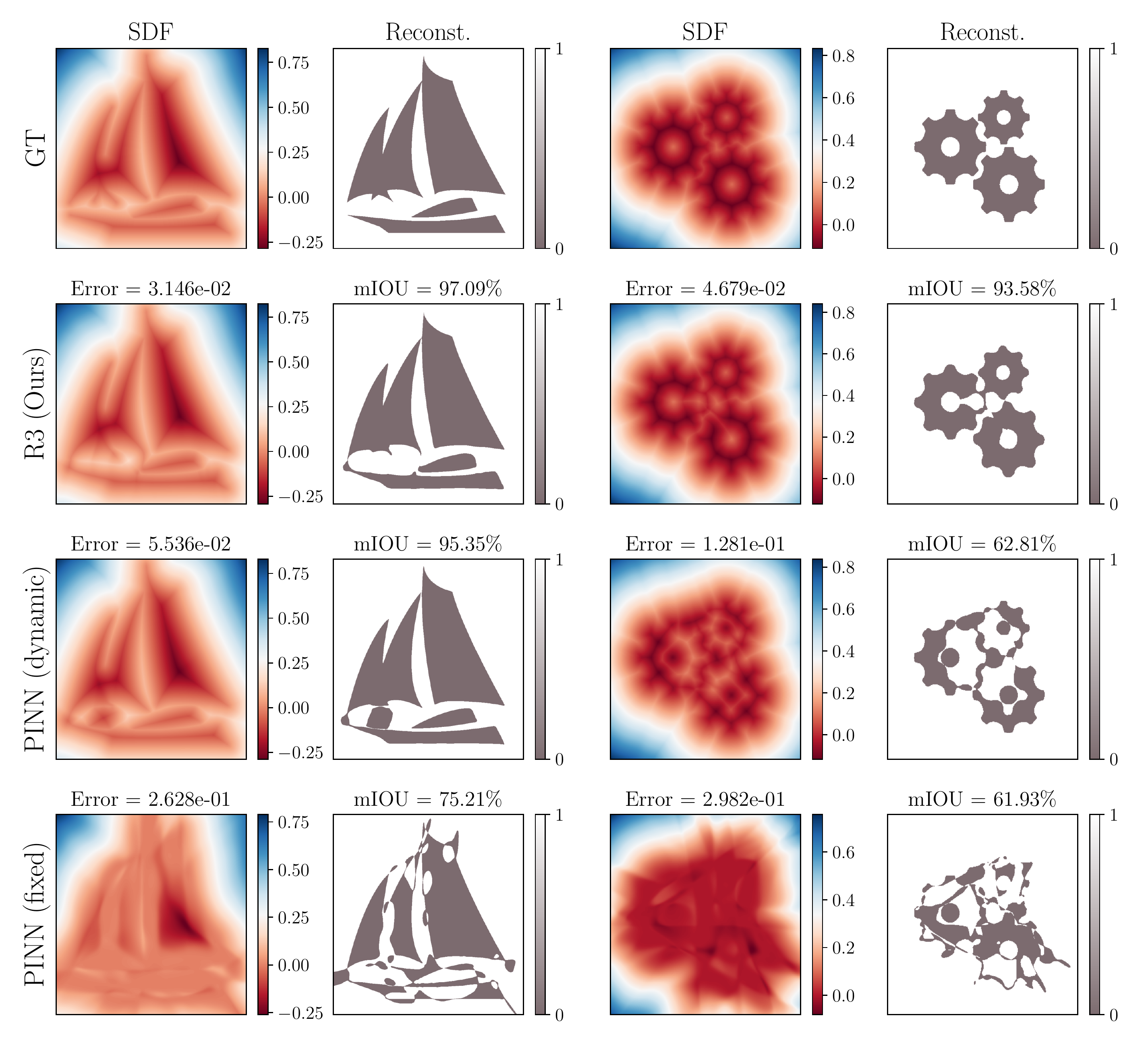}
    \vspace{-2ex}
    \caption{Solving Eikonal equation for signed distance field (SDF). The color of the heatmap represents the values of the SDF. The gray region shows the negative values of SDF that represents the interior points in the reconstructed geometry from predicted SDF.}
    \label{fig:eikonal}
    \vspace{-4ex}
\end{figure}

\textbf{Solving Eikonal Equations.} Given the equation of a surface geometry in a 2D-space, $u(x_s, y_s) = 0$, the Eikonal equation is a time-independent PDE used to solve for the \textit{signed distance field} (SDF), $u(x,y)$, which has negative values inside the surface and positive values outside the surface. See Appendix \ref{sec:details_pde} for details of the Eikonal equation. 
% We used 2D Eikonal equation as an example of time-independent PDE to evaluate different methods: $\nabla_{x, y}u = 1$ with boundary conditions are $u(-1, y) > 0, u(1, y) > 0, u(x, -1) > 0, u(x, 1) > 0$ for $\forall x, y \in \Omega$, and on surface $u(x_s, y_s) = 0$ for $\forall x_s, y_s \in \mathcal{S}$ (where $\mathcal{S}$ is the surface). The solution of the Eikonal equation is often referred to as \textit{signed distance fields} (SDF) which have negative values inside an object and positive values at the outside. 
The primary difficulty in solving Eikonal equation comes from determining the sign of the field (interior or exterior) in regions with rich details. We compare the performance of different baseline methods with respect to the ground-truth (GT) solution obtained from numerical methods for two complex surface geometries in Figure \ref{fig:eikonal}. 
% We can see that PINN predicted and ground-truth (GT) signed distance fields generated from numerical methods for three complex geometries. The performance is measured by relative L2 errors. 
% A few errors in the field may result in a flip in the sign of the predicted SDF in some regions. 
We also plot the reconstructed geometry of the predicted solutions to demonstrate the real-world application of solving this PDE, e.g., in downstream real-time graphics rendering. The quality of reconstructed geometries are quantitatively evaluated using the mean Intersection-Over-Union (mIOU) metric.
The results show that PINN-fixed shows poor performance across both geometries, while PINN-dynamic is able to capture most of the outline of the solutions with a few details missing. On the other hand, R3 sampling is able to capture even the fine details of the SDF for both geometries and thus show better reconstruction quality. We can see that mIOU of R3 sampling is significantly higher than baselines. See Appendix Section \ref{sec:evo_viz} for more discussion and visualizations.

%% file: discussion.tex
\vspace{-2ex}
\section{Conclusions And Future Work Directions}
% \vspace{-2ex}
We present a novel perspective for identifying failure modes in PINNs named ``\emph{propagation failures}.'' and  develop a novel \emph{R3 sampling} algorithm to mitigate propagation failures. 
% For time-depenedent PDEs, we provide an extension of our work that can explicitly model the causal structure of the PDEs. 
R3 empirically demonstrates better performance on a variety of benchmark PDEs.
Future work can focus on building the theoretical foundations of the propagation hypothesis and studying the interplay between minimizing PDE loss and sampling from high residual regions. Building on the success of established sampling techniques such as importance sampling in neural implicit representation learning problems, future research can focus on investigating the potential applicability of our proposed R3 sampling strategy in these problems.

\section{Acknowledgement}
We gratefully acknowledge the support of the National Science Foundation (NSF) under Grant \#2213550 and \#2107332, which provided funding for this research. We would also like to acknowledge the support of Virginia Tech Advanced Research Computing (ARC) for providing us with necessary resources for our study.

%% file: appendix.tex
\newpage

\input{Appendix/sampling_lp_norm}
\input{Appendix/evo_dynamics}

\input{Appendix/complexity_analysis}
\input{Appendix/details_pde}
\input{Appendix/hyperparams}
\input{Appendix/additional_viz}

\input{Appendix/test_optim_func}

%% file: Appendix/sampling_lp_norm.tex
\section{Connections Between $L^p$ Norm and Sampling}
\label{appendix:connection_lp_sampling}

In this section, we provide connections between adaptively sampling collocation points from a distribution $q(\mathbf{x_r}) \propto |\mathcal{R}_\theta(\mathbf{x_r})|^k$ \cite{wu2022comprehensive} and using $L^p$ norm of the PDE loss \cite{wang2022is}.

% we provide the proof of  Theorem \ref{theorem:connection_lp_sampling} presented in the main paper.

\begin{theorem}
\label{theorem:connection_lp_sampling}
    For $p \geq 2$, let $\mathcal{L}^p_r(\mathcal{U})$ denote the expected $L^p$ PDE Loss computed on collocation points sampled from a uniform distribution, $\mathcal{U}(\Omega)$. Similarly, for $k \geq 0$, let $\mathcal{L}^2_r(\mathcal{Q}^k)$ denote the expected $L^2$ PDE Loss computed on collocation points sampled from an alternate distribution $\mathcal{Q}^k(\Omega):\mathbf{x_r} \sim q(\mathbf{x_r})$, where $q(\mathbf{x_r}) \propto |\mathcal{R}_\theta(\mathbf{x_r})|^k$. Then, $\mathcal{L}^2_r(\mathcal{Q}^k) = \frac{1}{Z^{1/2}V^{-1/2}} \Big( \mathcal{L}^{k+2}_r(\mathcal{U}) \Big)^{(k+2)/2}$, where $Z$ is a normalization constant, defined as $Z = \int_{\mathbf{x}_r \in \Omega} |\mathcal{R}_\theta(\mathbf{x_r})|^k d\mathbf{x_r}$ and $V$ is the volume of the domain $\Omega$.
\end{theorem}

\begin{proof}
    
The expectation of $L^p$ PDE Loss for collocation points sampled from a uniform distribution $\mathcal{U}(\Omega):\mathbf{x_r} \sim p(\mathbf{x_r})$ can be defined as follows \footnote{We assume that the batch size/number of collocation points used to compute the PDE Loss tends to infinity, i.e., $N_r \rightarrow \infty$. This allows us to analyze the behavior of the continuous PDE loss function $\mathcal{L}_r(\theta)$ as $N_r \rightarrow \infty$}:
\begin{align}
\label{eq:lp_norm}
    \mathcal{L}^p_r(\mathcal{U}) &= \Big(\mathbb{E}_{\mathbf{x_r} \sim \mathcal{U}(\Omega)}|\mathcal{R}_\theta(\mathbf{x_r})|^p \Big)^{1/p} \nonumber\\
    &= \Big( \int p(\mathbf{x_r}) |\mathcal{R}_\theta(\mathbf{x_r})|^p d\mathbf{x_r} \Big)^{1/p}
\end{align}

Note that for a uniform distribution, $p(\mathbf{x_r})=\frac{1}{V}$, where $V$ is the volume of the domain $\Omega$, i.e., $V = \prod_{i=1}^{n} \Big( \text{supp}(x_i) - \text{inf}(x_i) \Big)$ with $\text{supp}(.)$ and $\text{inf}(.)$ being the supremum and infimum operators, and $x_i$ is the $i$-th dimension of $\mathbf{x_r}$ (e.g., the space dimension $x$ or the time dimension $t$). 

Now, let us consider the case where we are interested in sampling from an alternate distribution $\mathcal{Q}^k(\Omega):\mathbf{x_r} \sim q(\mathbf{x_r})$, where $q(\mathbf{x_r}) \propto |\mathcal{R}_\theta(\mathbf{x_r})|^k$ while using the $L^2$ PDE Loss (the most standard loss formulation used in PINNs). The sampling function of $\mathcal{Q}^k(\Omega)$ can be defined as follows:
\begin{align}
\label{eq:sampling_prob}
    q(\mathbf{x_r}) = \frac{|\mathcal{R}_\theta(\mathbf{x_r})|^k}{Z} 
\end{align}
where $Z$ is the normalizing constant, i.e., $Z = \int_{\mathbf{x}_r \in \Omega} |\mathcal{R}_\theta(\mathbf{x_r})|^k d\mathbf{x_r}$.

Hence, the $L^2$ PDE Loss for collocation points sampled from $\mathcal{Q}^k(\Omega)$ can be defined as:
\begin{align}
    \mathcal{L}^2_r(\mathcal{Q}^k) &= \Big( \mathbb{E}_{\mathbf{x_r} \sim \mathcal{Q}^k(\Omega)} |\mathcal{R}_\theta(\mathbf{x_r})|^2 \Big)^{1/2} \nonumber\\
    &= \Big( \int q(\mathbf{x_r}) |\mathcal{R}_\theta(\mathbf{x_r})|^2 d\mathbf{x_r} \Big)^{1/2}\nonumber\\
    &= \Big( \int \frac{|\mathcal{R}_\theta(\mathbf{x_r})|^k}{Z} |\mathcal{R}_\theta(\mathbf{x_r})|^2 d\mathbf{x_r} \Big)^{1/2} \qquad \qquad \qquad \qquad \qquad (\text{From Equation} \:\: \ref{eq:sampling_prob}) \nonumber\\
    &= \frac{1}{Z^{1/2}}\Big( \int |\mathcal{R}_\theta(\mathbf{x_r})|^{k+2} d\mathbf{x_r} \Big)^{1/2} \nonumber\\
    &= \frac{1}{Z^{1/2}}\Big( \int \frac{p(\mathbf{x_r})}{p(\mathbf{x_r})} |\mathcal{R}_\theta(\mathbf{x_r})|^{k+2} d\mathbf{x_r} \Big)^{1/2}\nonumber\\
    &= \frac{1}{Z^{1/2}V^{-1/2}}\Big( \int p(\mathbf{x_r}) |\mathcal{R}_\theta(\mathbf{x_r})|^{k+2} d\mathbf{x_r} \Big)^{1/2}, \qquad \because \: \mathbf{x_r} \sim \mathcal{U}(\Omega) \implies p(\mathbf{x_r})=\frac{1}{V}\nonumber\\
    &= \frac{1}{Z^{1/2}V^{-1/2}} \Big( \mathbb{E}_{\mathbf{x_r} \sim \mathcal{U}(\Omega)} |\mathcal{R}_\theta(\mathbf{x_r})|^{k+2} \Big)^{1/2} \nonumber\\
    &= \frac{1}{Z^{1/2}V^{-1/2}} \Big( \mathcal{L}^{k+2}_r(\mathcal{U}) \Big)^{(k+2)/2} \qquad  (\text{From Equation} \:\: \ref{eq:lp_norm}, \text{with} \:\: p=k+2)
\end{align}

\end{proof}

Theorem \ref{theorem:connection_lp_sampling} suggests that sampling collocation points from a distribution $q(\mathbf{x_r}) \propto |\mathcal{R}_\theta(\mathbf{x_r})|^k$ and $L^p$ norm of the PDE loss are related to each other by a scaling term $\frac{1}{Z^{1/2}V^{-1/2}}$. However, since this scaling term is variable in nature as $Z$ depends on the neural network parameters $\theta$, optimizing $\mathcal{L}^p_r(\mathcal{U})$ is not directly equivalent to optimizing $\mathcal{L}^2_r(\mathcal{Q}^k)$ (or a power thereof).

% where $C$ is the volume of the domain $\Omega$, i.e., $C = \prod_{i=1}^{n} \Big( \text{supp}(x_i) - \text{inf}(x_i) \Big)$ with $\text{supp}(.)$ and $\text{inf}(.)$ being the supremum and infimum operators, and $x_i$ is the $i$-th dimension of $\mathbf{x_r}$ (e.g., the space dimension $x$ or the time dimension $t$). 

% Therefore, for $k = p-2$, we get the following equivalence between $L^p$ Physics-informed loss and sampling from $\mathcal{Q}(\Omega)$:

% \begin{align}
%     \mathcal{L}^2_r(\theta)(\mathcal{Q}(\Omega)) = \frac{1}{Z^{1/2}C^{1/2}} \Big( \mathcal{L}^{p}_r(\theta)(\mathcal{U}(\Omega)) \Big)^{p/2}
% \end{align}

\section{Analyzing PINN-Dynamic: A Strong Baseline}
\label{sec:analyse_dynamic}

A very simple baseline for mitigating propagation failures is to \emph{dynamically} sample a random set of $N_r$ collocation points from a uniform distribution at every iteration, independently of previous iterations.
To see how this simple sampling strategy can help in sampling points from high PDE residual regions, let us consider partitioning the entire input domain $\Omega$ into two subsets: regions with high PDE residuals, $\Omega_{high}$, and regions with low PDE residuals $\Omega_{low}$. If the high PDE residual regions are imbalanced because of propagation failure, we would expect that the area of $\Omega_{high}$ ($A_{high}$) is significantly smaller than the area of $\Omega_{low}$ ($A_{low}$), i.e., $A_{high} \ll A_{low}$.
Let us compute the probability of sampling at least one collocation point from $\Omega_{high}$ across all iterations of PINN training. If we use a \textit{fixed} set of collocation points sampled from a uniform distribution at every iteration, this probability will be equal to $A_{high}/(A_{high} + A_{low})$, which can be very small. Hence, we are likely to under-represent points from $\Omega_{high}$ in the training process and thus get stuck in trivial solutions with high residual regions around its boundaries, especially when $N_r$ is low.

On the other hand, if we perform \textit{dynamic} sampling, the probability of picking at least one point from $\Omega_{high}$ across all iterations will be equal to $1 - (1 - A_{high}/(A_{high} + A_{low}))^N$, where $N$ is the number of iterations for which the high residual region $A_{high}$ persists during the training of PINN. Note that $N << N_{train}$ where $N_{train}$ is the total number of training iterations for the PINN. We can see that when $N$ is large, this probability approaches $1$, indicating that across all iterations of PINN training, we would have likely sampled a point from $\Omega_{high}$ in at least one iteration, and used it to minimize its PDE residual. As we empirically demonstrate later in Section \ref{sec:results}, dynamic sampling is indeed able to control the skewness of PDE residuals compared to fixed sampling, and thus act as a strong baseline for mitigating propagation failures.

However, note that even if we use dynamic sampling, the contribution of points from $\Omega_{high}$ in the overall PDE residual loss computed at any iteration is still low. In particular, since the probability of sampling points from $\Omega_{high}$ at any iteration is equal to $A_{high}/(A_{high} + A_{low})$, the expected PDE residual loss computed over all collocation points will be equal to
% \vspace{-1ex}
\begin{equation}
    \mathbb{E}_{\Omega}[\mathcal{L}_r(\theta)] = \mathbb{E}_{\Omega_{high}}[\mathcal{L}_r(\theta)]\times \frac{A_{high}}{A_{high} + A_{low}} + \mathbb{E}_{\Omega_{low}}[\mathcal{L}_r(\theta)]\times \frac{A_{low}}{A_{high} + A_{low}} \label{eq:expectation_dynamic}
\end{equation}

Since $A_{high} \ll A_{low}$, the gradient update of $\theta$ at every epoch will be dominated by the low PDE residuals observed over points from $\Omega_{low}$, leading to slow propagation of information from initial/boundary points to interior points.
% }

%% file: Appendix/evo_dynamics.tex
\section{Analysis of R3 Sampling}
In this section, we  analyze the dynamic behavior (or evolution) of the collocation points for our proposed \emph{Retain-Resample-Release Sampling} (R3) approach over the iterations. 
% A schematic representation of R3 is also provided in Figure \ref{fig:evosample_schematic}.

% \begin{figure}
%     % \vspace{-1ex}
%     \centering
%     \includegraphics[width=0.75\textwidth]{fig/methods/evo_sampling.pdf}
%     \caption{Schematic to describe our proposed R3 sampling algorithm, where collocation points are incrementally accumulated in regions with high PDE residuals (shown as contour lines).}
%     \vspace{-3ex}
%     \label{fig:evosample_schematic}
% \end{figure}

\subsection{Retain Property of R3.} 
\label{sec:grad_accum_proof}

In this section, we provide the proof of Theorem  \ref{theorem:expectation_retained_pop} presented in the main paper.
\begin{definition}[Objective Function]
\label{def:obj_fn}
    Let $\mathcal{F}_\theta(\mathbf{x}):\mathbb{R}^n \rightarrow \mathbb{R}^+$ be an arbitrary positive real-valued $k$-Lipschitz continuous function, where $\theta$ denotes the neural network parameters. When $\theta$ is fixed, the function $\mathcal{F}(\mathbf{x})$ does not vary with iterations, representing a fixed objective function.
\end{definition}

Let $\mathbf{X}^* = \{\mathbf{x}_i^*: \mathcal{F}(\mathbf{x}_i^*) = \max_\mathbf{x} \mathcal{F}(\mathbf{x}) \:\:  \forall \: i \in [n] \}$ be the set of points where the objective function $\mathcal{F}$ is maximal. 

Now, let us define an $\epsilon$-neighborhood around each of point $\mathbf{x}_i^* \in \mathbf{x}^*$ as $\mathcal{N}_\epsilon(\mathbf{x}_i)$ such that $||\mathbf{x}_i - \mathbf{x}_i^*|| \leq \epsilon$ for any arbitrarily small $\epsilon > 0$ and for all $\mathbf{x}_i \in \mathcal{N}_\epsilon(\mathbf{x}_i)$ with $i \in [n]$.

Let us also assume that the objection function $\mathcal{F}$ is $k$-Lipschitz continuous. Then the following is true:
\begin{align}
    |\mathcal{F}(\mathbf{x}_i^*)-\mathcal{F}(\mathbf{x}_i)| &\leq k\epsilon \qquad \qquad \forall \: \mathbf{x}_i \in \mathcal{N}_\epsilon(\mathbf{x}_i) \:\: \& \:\: i \in [n] \\
    \implies \mathcal{F}(\mathbf{x}_i^*)-\mathcal{F}(\mathbf{x}_i) &\leq k\epsilon \qquad \qquad   \because \mathcal{F}(\mathbf{x}_i^*) = \max_\mathbf{x} \mathcal{F}(\mathbf{x})
\end{align}

\begin{definition}[$\epsilon$-maximal Neighborhood]
\label{def:max_neighborhood}
    Let $\mathcal{F}^* = \max_\mathbf{x} \mathcal{F}(\mathbf{x})$ be the maximal value of the objective function $\mathcal{F}(\mathbf{x})$ (Definition \ref{def:obj_fn}). Then an $\epsilon$-maximal Neighborhood $\mathcal{N}^{\infty}_\epsilon$ can be defined as: $\mathcal{N}^{\infty}_\epsilon = \mathcal{N}_\epsilon(\mathbf{x}_0) \cup \mathcal{N}_\epsilon(\mathbf{x}_1) \cup ... \cup \mathcal{N}_\epsilon(\mathbf{x}_n)$ such that any point $\mathbf{x}$ sampled from $\mathcal{N}^{\infty}_\epsilon$ would have $\mathcal{F}^*-\mathcal{F}(\mathbf{x}) \leq k\epsilon \:\:\: \forall \:\: \mathbf{x} \in \mathcal{N}^{\infty}_\epsilon$ and for any arbitrarily small $\epsilon > 0$.
\end{definition}

{Note} that since the volume of $\mathcal{N}^{\infty}_\epsilon$ is greater than 0, the probability of sampling any $\mathbf{x} \in \mathcal{N}^{\infty}_\epsilon$ from a uniform distribution $\mathcal{U}(\mathbf{x})$ is greater than 0.

\begin{lemma}[Population Properties]
\label{lemma:pop_properties}
For any population $\mathcal{P}$ generated at some iteration of R3 sampling, optimizing a given objective function $\mathcal{F}(\mathbf{x})$ (Definition \ref{def:obj_fn}), the following properties are always true:
\begin{enumerate}
    \item The re-sampled population is always non-empty, i.e., $|\mathcal{P}^s|>0$
    \item The size of the retained population is always less than the total population size, i.e., $|\mathcal{P}^r|<|\mathcal{P}|$
    \item The size of the retained population is zero, i.e., $|\mathcal{P}^r|=0$, if and only if  $\mathcal{F}(\mathbf{x}) = c, \forall \mathbf{x} \in \mathcal{P}$.
\end{enumerate} 
\end{lemma}

\begin{proof}
The threshold $\tau$ for the R3 Sampling can be computed as $\tau = \frac{1}{|\mathcal{P}|}\sum_{\mathbf{x} \in \mathcal{P}} \mathcal{F}(\mathbf{x})$.

The retained population is defined as: $\mathcal{P}^r \leftarrow \{ \mathbf{x}: \mathcal{F}(\mathbf{x}) > \tau \:\:\: \forall \mathbf{x} \in \mathcal{P} \}$, 

Similarly, the non-retained population can be defined as: $\overline{\mathcal{P}^r} \leftarrow \{ \mathbf{x}: \mathcal{F}(\mathbf{x}) \leq \tau \:\:\: \forall \mathbf{x} \in \mathcal{P} \}$. 

\textbf{Proof of Property 1:} For any arbitrary set of real numbers, there always exists some element in the set that is less than or equal to the mean. Hence, the size of the non-retained population is always non-zero as there always exists some point $\mathbf{x} \in \mathcal{P}$ such that $\mathcal{F}(\mathbf{x}) \leq \tau$. Thus, $|\overline{\mathcal{P}^r}| > 0$. 

Now by definition, since the re-sampled population $\mathcal{P}^s$ replaces the non-retained population at every iteration, $|\mathcal{P}^s|=|\overline{\mathcal{P}^r}|$. Hence, $|\mathcal{P}^s|>0$, i.e., the size of the resampled population is always non-zero.

\textbf{Proof of Property 2:} By definition, $|\overline{\mathcal{P}^r}| + |\mathcal{P}^r| = |\mathcal{P}|$ (where $|\mathcal{P}|$ is the total size of the population and is always constant). Since, $|\overline{\mathcal{P}^r}| > 0$, we can say that $|\mathcal{P}^r| < |\mathcal{P}|$, i.e., the size of the retained population can never be equal to the entire population size $|\mathcal{P}|$.

\textbf{Proof of Property 3:} Let us consider the case where  $\mathcal{F}(\mathbf{x}) = c, \forall \mathbf{x} \in \mathcal{P}$ (where $c$ is some constant), i.e., the value of the function is constant at all of the points $\mathbf{x} \in \mathcal{P}$. In this case, the mean of the population $\mathcal{P}$, which is equal to the threshold, $\tau$ will be equal to $c$. This condition would lead to the entire population to be re-sampled as all element $\mathbf{x} \in \mathcal{P}$ would satisfy the condition to belong in the non-retained population. Note that the constant function $\mathcal{F}(\mathbf{x})=c$ is the only case where all of the elements are less than or equal to the mean. Otherwise, there would always be at least one element greater than the mean, resulting in a non-zero size of the retained population.
\end{proof}

% len of non-retained population > 0 thus 

% The re-sampled population $\mathcal{P}^s$ is defined 
% Let $f(x)$ be any arbitrary

% lemma: retained population can never be equal to the entire population. not every point can be greater than the mean.

% Retained = 0, if everything is constant.

% Even if the size of the retained population somehow converges to the entire population size by adding every collocation points from the sampled population  

\begin{lemma}[Entry Condition]
\label{lemma:entry}
% Let us assume that the re-sampled population is non-empty, i.e., $|\mathcal{P}^s|>0$. 
If a point $\mathbf{x}_m$ is sampled from $\mathcal{N}^{\infty}_\epsilon$ at any arbitrary iteration $m$, then it will always enter the retained population $\mathcal{P}_m^r$ unless $\mathbb{E}_{\mathbf{x} \in \mathcal{P}^r_m}[\mathcal{F}(\mathbf{x})] > \mathcal{F}^* - k\epsilon$.
\end{lemma}

\begin{proof}
    
The condition for any arbitrary point $\mathbf{x}_m$ to enter the retained population $\mathcal{P}_m^r$ at any arbitrary iteration $m$ is given by the following:
\begin{align}
    \mathcal{F}(\mathbf{x}_m) > \tau_m = \mathbb{E}_{\mathbf{x} \in \mathcal{P}_m}[\mathcal{F}(\mathbf{x})]. \qquad \qquad \text{(By definition of the threshold $\tau_m$)}
\end{align}

Now, if the point $\mathbf{x}_m$ is sampled from $\mathcal{N}^{\infty}_\epsilon$, then $\mathcal{F}(\mathbf{x}_m) \geq \mathcal{F}^* - k\epsilon$ (from Definition \ref{def:max_neighborhood}). 
Hence, for $\mathbf{x}_m$ to enter the retained population $\mathcal{P}_m^r$, we need to ensure that $\mathcal{F}^* - k\epsilon > \tau_m$. 
%We use proof by contradiction as follows.

Let us consider the case where $\mathbf{x}_m$ is not able to enter the retained population. In such a case, we will have the following inequality: 
\begin{equation}
\label{eq:inequality_1}
    \mathcal{F}^* - k\epsilon < \tau_m.
\end{equation}
 
 It is also easy to show from the definition of retained population that the threshold $\tau$ is always less than the expectation of the retained population: 
\begin{equation}
\label{eq:inequality_2}
    \tau_m \leq \mathbb{E}_{\mathbf{x} \in \mathcal{P}^r_m}[\mathcal{F}(\mathbf{x})]
\end{equation}

From Equations \ref{eq:inequality_1} and \ref{eq:inequality_2}, we get, 
\begin{align}
\label{eq:inequality_3}
    \mathcal{F}^* - k\epsilon < \tau_m &\leq \mathbb{E}_{\mathbf{x} \in \mathcal{P}^r_m}[\mathcal{F}(\mathbf{x})] \nonumber\\
    \implies \mathbb{E}_{\mathbf{x} \in \mathcal{P}^r_m}[\mathcal{F}(\mathbf{x})] &> \mathcal{F}^* - k\epsilon
\end{align}

We have thus proved that $\mathbf{x}_m$ will not be able to enter the retained population $\mathcal{P}^r_m$ only if $\mathbb{E}_{\mathbf{x} \in \mathcal{P}^r_m}[\mathcal{F}(\mathbf{x})] > \mathcal{F}^* - k\epsilon$, which suggests that the expectation of the retained population is already close to $\mathcal{F}^*$, for any arbitrarily small $\epsilon > 0$. On the other hand, if $\mathbb{E}_{\mathbf{x} \in \mathcal{P}^r_m}[\mathcal{F}(\mathbf{x})] \leq \mathcal{F}^* - k\epsilon$, we would necessarily add $\mathbf{x}_m$ 
%  to the retained population.
\end{proof}

\begin{lemma}[Exit Condition]
\label{lemma:exit}
A point $\mathbf{x}_m$ sampled from $\mathcal{N}^{\infty}_\epsilon$ that entered the retained population at any arbitrary iteration $m$, can exit the retained population $\mathcal{P}^r_n$ at an arbitrary iteration $n$ (such that $n > m$) only if $\mathbb{E}_{\mathbf{x} \in \mathcal{P}^r_n}[\mathcal{F}(\mathbf{x})] \geq \mathcal{F}^* - k\epsilon$.
\end{lemma}

\begin{proof}
The generic condition for any arbitrary point $\mathbf{x}$ to exit the retained population $\mathcal{P}_n^r$ at iteration $n$ is given by: $\mathcal{F}(\mathbf{x}) \leq \tau_n = \mathbb{E}_{\mathbf{x} \in \mathcal{P}_n}[\mathcal{F}(\mathbf{x})]$ (By definition of the threshold $\tau$).

Since a point $\mathbf{x}_m$ that was originally sampled from $\mathcal{N}^{\infty}_\epsilon$ will have $\mathcal{F}(\mathbf{x}_m) \geq \mathcal{F}^* - k\epsilon$ (from Definition \ref{def:max_neighborhood}), we can use this inequality in the generic exit condition shown above to get,
\begin{equation}
    \mathcal{F}^* - k\epsilon \leq \tau_n \leq \mathbb{E}_{\mathbf{x} \in \mathcal{P}^r_n}[\mathcal{F}(\mathbf{x})]
\end{equation}

Hence, the point $\mathbf{x}_m$ can exit the retained population at iteration $n$ only if $\mathbb{E}_{\mathbf{x} \in \mathcal{P}^r_n}[\mathcal{F}(\mathbf{x})] \geq \mathcal{F}^* - k\epsilon$. 
\end{proof}
% \begin{theorem}
%     Let $f_\theta(x):\mathbb{R}^n \rightarrow \mathbb{R}^+$ be an arbitrary positive real-valued, $k$-Lipschitz continuous objective function optimized using the Evolutionary Sampling algorithm with fixed $\theta$. Then, the expectation of the retain population  $\mathbb{E}_{x \in \mathcal{P}^r}[f(x)] \geq \max_x f(x) - k\epsilon$ as iteration $i \rightarrow \infty$, for some $\epsilon > 0$.
% \end{theorem}

\begin{theorem}[Accumulation  Dynamics Theorem]
    Let $\mathcal{F}_\theta(\mathbf{\mathbf{x}}):\mathbb{R}^n \rightarrow \mathbb{R}^+$ be a fixed real-valued $k$-Lipschitz continuous objective function optimized using the R3 Sampling algorithm. Then, the expectation of the retained population  $\mathbb{E}_{\mathbf{x} \in \mathcal{P}^r}[\mathcal{F}(\mathbf{x})] \geq \max_\mathbf{x} \mathcal{F}(\mathbf{x}) - k\epsilon$ as iteration $i \rightarrow \infty$, for any arbitrarily small $\epsilon > 0$.
\end{theorem}

% \begin{theorem}[Accumulation  Dynamics Theorem]
%     Let $f_\theta(x):\mathbb{R}^n \rightarrow \mathbb{R}^+$ be an arbitrary positive real-valued $k$-Lipschitz continuous objective function optimized using the Evolutionary Sampling algorithm with fixed $\theta$. Then, the expectation of the retain population  $\mathbb{E}_{x \in \mathcal{P}^r}[f(x)] \geq \max_x f(x) - k\epsilon$ as iteration $i \rightarrow \infty$, for any arbitrarily small $\epsilon > 0$.
% \end{theorem}

% \textbf{Proof} 
\begin{proof}
    We prove this theorem by contradiction. For the sake of contradiction, let us assume that as iterations $i \rightarrow \infty$, the expectation of the retained population $\mathbb{E}_{\mathbf{x} \in \mathcal{P}^r}[\mathcal{F}(\mathbf{x})] < \max_\mathbf{x} \mathcal{F}(\mathbf{x}) - k\epsilon$, for any arbitrarily small $\epsilon > 0$. We can then make the following two remarks.

    \textbf{Entry of collocation points}: Note that the probability of sampling $\mathbf{x}$ from $\mathcal{N}^\infty_\epsilon$ is  non-zero because the size of the re-sampled population is non-zero, i.e., $|\mathcal{P}^s| > 0$ (proved in Lemma \ref{lemma:pop_properties}). Also, since we have assumed $\mathbb{E}_{\mathbf{x} \in \mathcal{P}^r}[\mathcal{F}(\mathbf{x})] < \mathcal{F}^* - k\epsilon$, we can use the Entry condition proved in Lemma \ref{lemma:entry} to arrive at the conclusion that a point from $\mathcal{N}^\infty_\epsilon$ will always be able to enter the retained population. 
    
    \textbf{Exit of collocation points}: Similarly, a point $\mathbf{x}$ that belongs in the $\epsilon$-maximal neighborhood and is part of the retained population $\mathcal{P}^r$ will not be able to escape the retained population as we have asssumed $\mathbb{E}_{\mathbf{x} \in \mathcal{P}^r}[\mathcal{F}(\mathbf{x})] < \mathcal{F}^* - k\epsilon$ (using the Exit condition proved in Lemmas \ref{lemma:exit}). 
    
    From the above two remarks, we can see that  points would keep accumulating indefinitely in the retained population if our initial assumption (for the sake of contradiction) is true. However, since the total size of the population $|\mathcal{P}|$ is bounded, the size of the retained population $|\mathcal{P}^r|$ cannot grow indefinitely. We have thus arrived at a contradiction suggesting our assumption is incorrect. Hence, as iterations $i \rightarrow \infty$, the expectation of the retained population $\mathbb{E}_{\mathbf{x} \in \mathcal{P}^r}[\mathcal{F}(\mathbf{x})] \geq \max_\mathbf{x} \mathcal{F}(\mathbf{x}) - k\epsilon$, for any arbitrarily small $\epsilon > 0$.

\end{proof}

From Theorem \ref{theorem:expectation_retained_pop}, we can prove a continuous accumulation of collocation points from the $\epsilon$-maximal neighborhood until the expectation of the retained population is close to the maximum point (i.e., reaches $L^\infty$), thus exhibiting the \textbf{\emph{Retain Property}} described in Section \ref{sec:evo}. Although this theorem assumes that the objective function $\mathcal{F}(\mathbf{x})$ (or in the context of PINNs, the absolute residual values, $\mathcal{R}_\theta(\mathbf{x_r})$) is constant, this theorem is still valid when $\mathcal{R}_\theta(\mathbf{x_r})$ is gradually changing with the highest error regions (defined using our $\epsilon$-maximal neighborhood) persisting over iterations. 
% This condition aims to mimic the scenario when the highest PDE residual region persists for a long time during the training of PINNs. 
Under such conditions, the theorem states that the retained population $\mathcal{P}^r$ would always accumulate points from the $\epsilon$-maximal neighborhood, thereby adaptively increasing their contribution to the overall PDE residual loss and eventually resulting in their minimization.

\subsection{Release of Collocation Points from High PDE Residual Regions.}
\label{sec:release_proof}

Our definition of the \textbf{\emph{Release Property}} states that the distribution of collocation points should revert back to its original form by releasing the accumulated points in the high PDE residual regions once they are ``\emph{sufficiently minimized}''. Let us define that for an arbitrary collocation point $\mathbf{x_r}^i \in \mathcal{P}$, ``\emph{sufficient minimization}'' of the PDE is achieved if  $\mathcal{R}_\theta(\mathbf{x_r}^i) \leq \mathbb{E}_{\mathbf{x_r} \in \mathcal{P}}[\mathcal{R}_\theta(\mathbf{x_r})]=\tau$ (where $\tau$ is the threshold used by R3 sampling). Then, by definition, such points will belong to the ``non-retained population'' and will be immediately replaced by the re-sampled population. Since we generate the re-sampled population $\mathcal{P}^s$ from a uniform distribution, these ``\emph{sufficiently minimized}'' collocation points are replaced with a uniform density. Thus, R3 sampling satisfies the ``\emph{Release Property}''.

\section{Additional Details for Causal R3 Sampling}
\label{sec:causal_evo_appendix}

% Figure \ref{fig:causal_evo_schematic} represents a schematic describing the causally biased R3 sampling described in Section \ref{sec:causalevo} and the causal gate $g$ that is updated every iteration.

% \begin{figure}[ht]
% \centering
% \subfigure[Causal Gate.]{\label{fig:causal_gate} \includegraphics[scale=0.45]{fig/methods/CausalGate.pdf}}
% \subfigure[Causally R3 Sampling]{\label{fig:causal_evo_schematic} \includegraphics[scale=0.58]{fig/methods/causal_evo_sampling.pdf}} 
% \vspace{-3ex}
% \caption{Causal R3 uses a time-dependent causal gate for computing PDE loss and for sampling. }
% \label{fig:causal_evo}
% \end{figure}

% \begin{figure}[ht]
% \begin{subfigure}{.31\textwidth}
%   \centering
%   \vspace{1.0ex}
%   \includegraphics[width=\linewidth]{fig/methods/CausalGate.pdf}
%   \vspace{-2.2ex}
%   \caption{Causal Gate.}
%   \label{fig:causal_gate}
% \end{subfigure}
% \begin{subfigure}{.58\textwidth}
%   \centering
%   \includegraphics[width=\linewidth]{fig/methods/causal_evo_sampling.pdf}
%   \vspace{-2.5ex}
%   \caption{\centering Causally Biased Evo Sampling}
%   \label{fig:causal_evo_schematic}
% \end{subfigure}
% \caption{
% Causal Evo uses a time-dependent causal gate for computing PDE loss and for sampling. 
% }
% \label{fig:causal_evo}
% % \vspace{-3ex}
% \end{figure}

\subsection{Preventing Abrupt Causal Gate Movement.}
The shift parameter $\gamma$ of the causal gate is updated every iteration using the following scheme: $\gamma_{i+1} = \gamma_{i} + \eta_{g} e ^ {-\epsilon \mathcal{L}^{g}_r(\theta)}$, where $\eta_{g}$ is the learning rate that controls how fast the gate should propagate and $\epsilon$ denotes tolerance that controls how low the PDE loss needs to be before the gate shifts to the right, and $i$ denotes the $i^{th}$ iteration. Typically, in our experiments we set the learning rate to $1\text{e-}3$. Thus, for example, if the expectation of $e ^ {-\epsilon \mathcal{L}^{g}_r(\theta)}$ over 1000 iterations is 0.1, then  $\gamma$ would change by a value of 0.1 after 1000 iterations (since $\gamma_{i+N} \approx \gamma_{i} + \eta_{g}*N*\mathbb{E}[e ^ {-\epsilon \mathcal{L}^{g}_r(\theta)}]$). Additionally, note that, for a typical ``tanh'' causal gate, the operating range of $\gamma$ values vary from $-0.5$ to $1.5$. However, if the loss is very small ($\mathcal{L}^{g}_r(\theta) \rightarrow 0$), the magnitude of the update $e ^ {-\epsilon \mathcal{L}^{g}_r(\theta)} \rightarrow 1$, i.e., leads to an abrupt change in the causal gate. Thus, to prevent an abrupt gate movement due to large magnitude update, we employ a magnitude clipping scheme (similar to gradient clipping in conventional ML) as follows: $\gamma_{i+1} = \gamma_{i} + \eta_{g} \min(e ^ {-\epsilon \mathcal{L}^{g}_r(\theta)}, \Delta_{max})$, where $\Delta_{max}$ is the maximum allowed magnitude of update. Typically, for our experiments we keep $\Delta_{max} = 0.1$. Note, that $\Delta_{max}$ needs to be carefully chosen depending on the gate learning rate $\eta_{g}$. 

% Therefore, we do not want our $\gamma$ to increase abruptly otherwise the principle of causality could be violated.
  
% $\Delta_{max}$
\subsection{Choice of Other Gate Functions.}
The gate function $g$ to enforce the principle of causality is not limited to the ``tanh'' gate presented in Section \ref{sec:causalevo} of the main paper. Any arbitrary function can be used for a causal gate as long as it obeys the following criteria: 
\begin{enumerate}
    \item \textbf{Continuous Time Property}: The function $g$ should be continuous in time, such that it can be evaluated at any arbitrary time $t$.
    \item \textbf{Monotonic Property}: The value of gate $g$ at time $t + \Delta t$ should be less than the value of the gate at time $t$, i.e., $g(t + \Delta t) \leq g(t)$. In other words, $g$ should be a monotonically decreasing function,
    \item \textbf{Shift Property}: The gate function should be parameterized using a shift parameter $\gamma$, such that $g_{\gamma}(t) < g_{\gamma+\delta}(t)$, where $\delta>0$, i.e., by increasing the value of the shift parameter the gate value of any arbitrary time should increase.
\end{enumerate}

An alternate choice of a causal gate is using a composition of ReLU and $\tanh$ functions: $g = ReLU(-\tanh(\alpha(t - \gamma))$ (as shown in Figure \ref{fig:relu_causal_gate}. We can see that by using ReLU, this alternate gate function provides a stricter thresholding of gate values to 0 after a cutoff value of time. The effect of this strict thresholding on the incorporation of causality in training PINNs can be studied in future analyses. In our current analysis, we simply used the $\tanh$ gate function for all our experiments.

\begin{figure}[ht]
    \centering
    \includegraphics[width=0.37\textwidth]{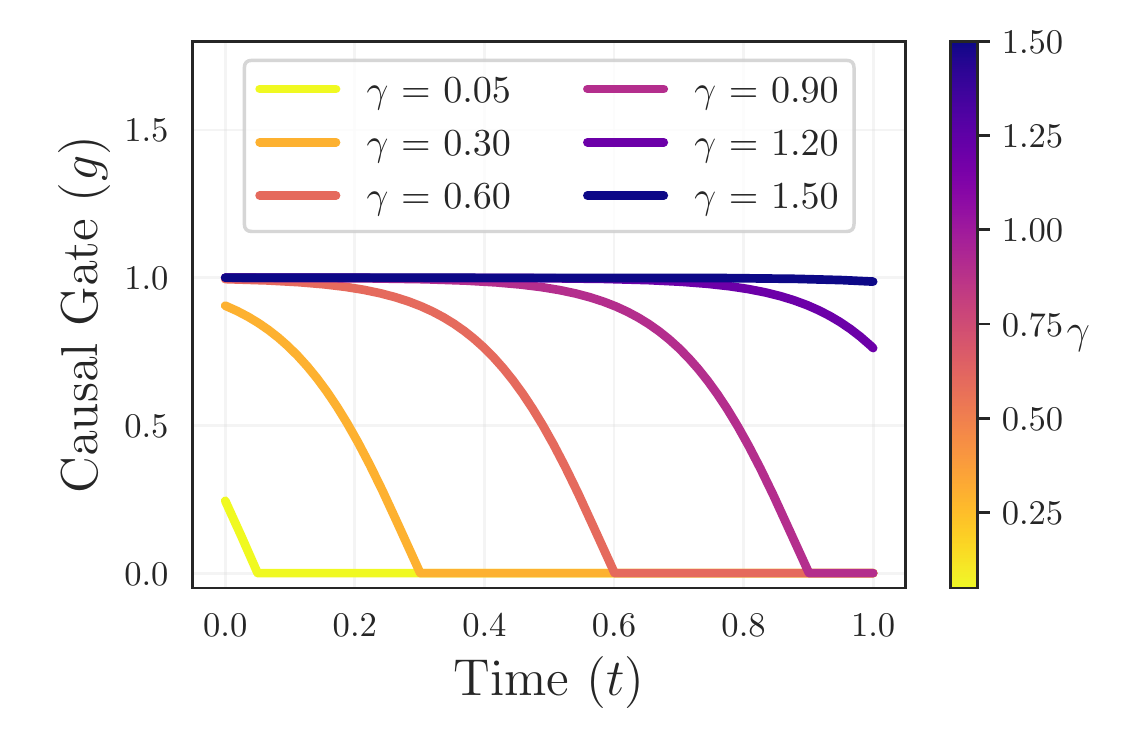}
    \vspace{-3ex}
    \caption{ReLU-tanh Causal Gate}
    \label{fig:relu_causal_gate}
\end{figure}

%% file: Appendix/complexity_analysis.tex
\section{Comparison of Baselines}
In this Section, we compare the different baseline methods w.r.t. the motivating properties. See Table \ref{tab:comp_baselines}.

\begin{table}[ht]
\centering
\setlength\tabcolsep{3pt} % let LaTeX figure out amount of intercolumn whitespace
\fontsize{9pt}{9}\selectfont
\caption{Table comparing baseline methods in terms of their ability to comply with the motivating properties of this work.}
\label{tab:comp_baselines}
\begin{tabular}{l|ccccl}
\hline 
                                        & \textbf{\begin{tabular}[c]{@{}c@{}}Retain \\ Property\end{tabular}} & \textbf{\begin{tabular}[c]{@{}c@{}}Release \\ Property\end{tabular}} & \textbf{\begin{tabular}[c]{@{}c@{}}Resample \\ Property\end{tabular}} & \textbf{\begin{tabular}[c]{@{}c@{}}Computational \\ Efficiency\end{tabular}} & \multicolumn{1}{c}{\textbf{Causality}} \\ \hline \vspace{-1ex} \\ 
RAR-G \cite{lu2021deepxde}  & \textbf{\checkmark}                                                                & \multicolumn{1}{l}{}                                                 & \textbf{\checkmark}                                                                         & \multicolumn{1}{l}{}                                                         &                                        \\
RAD \cite{nabian2021efficient}                            & \textbf{\checkmark}                                                                & \textbf{\checkmark}                                                           & \textbf{\checkmark}                                                                         & \multicolumn{1}{l}{}                                                         &                                        \\
RAR-D \cite{wu2022comprehensive}                         & \textbf{\checkmark}                                                                & \multicolumn{1}{l}{}                                                 & \textbf{\checkmark}                                                                         & \multicolumn{1}{l}{}                                                         &                                        \\
$L-\infty$ \cite{wang2022is} & \textbf{\checkmark}                                                                & \textbf{\checkmark}                                                           & \multicolumn{1}{l}{}                                                               & \textbf{\checkmark}                                                                   &                                        \\
CausalPINN \cite{wang2022respecting}                    & \multicolumn{1}{l}{}                                                      & \multicolumn{1}{l}{}                                                 & \textbf{\checkmark}                                                                         & \textbf{\checkmark}                                                                   & \multicolumn{1}{c}{\textbf{\checkmark}}         \\
R3 (ours)                            & \textbf{\checkmark}                                                                & \textbf{\checkmark}                                                           & \textbf{\checkmark}                                                                         & \textbf{\checkmark}                                                                   &                                        \\
Causal R3 (ours)                     & \textbf{\checkmark}                                                                & \textbf{\checkmark}                                                           & \textbf{\checkmark}                                                                         & \textbf{\checkmark}                                                                   & \multicolumn{1}{c}{\textbf{\checkmark}}  \\ \hline     
\end{tabular}
\end{table}

\section{Computational Complexity Analysis of R3 Sampling vs its baselines}
\label{sec:complexity}
In this section, we aim to provide a comprehensive comparison of the computational complexity of R3 Sampling and its baselines. It is well-known that the cost of computing the PDE residuals using automatic-differentiation during training amounts is reasonably large, especially when the PDE contains higher order gradients that require repeated backward passes through the computational graphs used by standard deep learning packages like PyTorch/Tensorflow. Thus, in this section we would mainly focus on comparing the number of PDE residual computations that each algorithm makes during training. We can quantify the effect of this difference on computational costs as follows. Let us first define the Notations that we are going to use for the analysis:

\vspace{5ex}

\centerline{\bf Notations for Computational Analysis}
\bgroup
\def\arraystretch{1.5}
\begin{tabular}{p{1in}p{3.5in}}
$\displaystyle C$ & Computational Cost to evaluate the PDE residual on a single collocation point.\\
$\displaystyle N$ & Total Number of Training Iterations\\
$\displaystyle |\mathcal{P}|$ & Total Number of Collocation Points/Population Size (Also referred to as the initial set of collocation points for RAR based methods)\\
$\displaystyle |\mathcal{P}_{dense}|$ & Auxiliary Set of dense points [For RAR-G,RAD, and RAR-D], where $|\mathcal{P}_{dense}| >> |\mathcal{P}|$\\
$\displaystyle K$ & Resampling Period [For RAR-G,RAD, and RAR-D]\\
$\displaystyle M$ & Number of additional collocation points added to the initial set\\
\end{tabular}

Next, we present the computational cost to run each method separately.

\textbf{PINN/R3}: The cost of computing the PDE residual at an arbitrary epoch $i$: $C_{R3}(i) = C|\mathcal{P}|$. Thus, the overall computational cost for $N$ iterations is $C_{R3} = NC|\mathcal{P}|$.

\textbf{RAR-G/RAD/RAR-D}: For RAR-based methods, the initial set of collocation points $\mathcal{P}$ keeps growing as $M$ new test points are added every $K$ iterations. Thus, at an iteration $i$, the size of the collocation point set $|\mathcal{P}_i| = |\mathcal{P}_0| + m\lfloor i/K \rfloor = |\mathcal{P}| + M\lfloor i/K \rfloor$ (to simplify notations, let $|\mathcal{P}_0|=|\mathcal{P}|$, i.e., all of the methods start with the same number of initial collocation points).

Thus, the cost to compute the PDE residuals on this training set $\mathcal{P}$ is:
\begin{align}
    C_{RAR}^{train} &= KC|\mathcal{P}| + KC(|\mathcal{P}|+M) + KC(|\mathcal{P}|+2M)+...+KC(|\mathcal{P}|+M\Big\lfloor \frac{N}{K} \Big\rfloor) \nonumber \\
    &=KC|\mathcal{P}| \Bigg( \Big\lfloor \frac{N}{K} \Big\rfloor + 1 + \frac{M}{2|\mathcal{P}|}\Big\lfloor \frac{N}{K} \Big\rfloor \Big( \Big\lfloor \frac{N}{K} \Big\rfloor + 1 \Big) \Bigg) 
\end{align}

There is also an additional cost of evaluating the PDE residuals on the dense set $C_{RAR}^{dense}$ to select these $M$ points from high PDE residual region.
\begin{align}
    C_{RAR}^{dense} = C|\mathcal{P}_{dense}|\Big\lfloor \frac{N}{K} \Big\rfloor
\end{align}

Therefore, the overall cost for RAR-based methods is: $C_{RAR} = C_{RAR}^{dense} + C_{RAR}^{train}$

\textbf{Comparing the cost between R3 and RAR}:
Assuming that the total number of epochs $N$ is divisible by the resampling period $K$, which is true for most practical scenarios. 
\begin{align}
    C_{RAR} &= NC|\mathcal{P}| + KC|\mathcal{P}| + \frac{MNC}{2}\Big( \frac{N}{K} + 1 \Big) + C|\mathcal{P}_{dense}|\frac{N}{K} \nonumber\\
    C_{RAR} &= C_{R3} + KC|\mathcal{P}| + \frac{MNC}{2}\Big( \frac{N}{K} + 1 \Big) + C|\mathcal{P}_{dense}|\frac{N}{K} \nonumber\\
    C_{RAR} - C_{R3} &= KC|\mathcal{P}| + \frac{MNC}{2}\Big( \frac{N}{K} + 1 \Big) + C|\mathcal{P}_{dense}|\frac{N}{K}
\end{align}

Thus, we can see that the difference in the computational cost can quickly grow depending on the choice of $RAR$ setting. Also note that, since $|\mathcal{P}_{dense}| >> |\mathcal{P}|$, the additional cost of $C|\mathcal{P}_{dense}|\frac{N}{K}$ is significant, especially if  we want to re-sample/re-evaluate the adaptive sampling frequently (i.e., for small values of $K$).

%% file: Appendix/details_pde.tex
\section{Details of Partial Differential Equations Used in this Work}
\label{sec:details_pde}

\subsection{Convection Equation}
We considered a 1D-convection equation that is commonly used to model transport phenomenon, described as follows:

\begin{align}
    &\frac{\partial u}{\partial t} + \beta \frac{\partial u}{\partial x} = 0, \: \: x \in [0, 2\pi], t \in [0, 1] \\
    &u(x, 0) = h(x) \\
    &u(0, t) = u(2\pi, t) 
\end{align}

where $\beta$ is the convection coefficient and h(x) is the initial condition. For our case studies, we used a constant setting of $h(x) = \sin(x)$ with periodic boundary conditions in all our experiments, while varying the value of $\beta$ in different case studies. 

% \textbf{Implementation Details}: For all of the baselines, we use a multi-layered perceptron (MLP) with 4 hidden layers, 50 neurons each and hyperbolic tan (tanh) as the activation function. For Evo, Causal Evo, PINN (fixed and dynamic), and Causal PINN (fixed and dynamic), we use the Adam optimizer with an initial learning rate of 1e-3 and a step learning rate scheduler with a decay rate of 0.9 every 5000 steps. However, for curriculum regularization baseline, we follow their recommended Adam optimizer with 1e-4 learning rate and no learning rate scheduler. 

\subsection{Allen-Cahn Equation}
We considered a 1D - Allen Cahn equation that is used to describe the process of phase-separation in multi-component alloy systems as follows: 

\begin{align}
    &\frac{\partial u}{\partial t} - 0.0001 \frac{\partial^2 u}{\partial x^2} + 5u^3 - 5u = 0, \: \: x \in [-1, 1], t \in [0, 1] \\
    &u(x, 0) = x^2\cos(\pi x) \\
    &u(t, -1) = u(t, 1) \\
    &\frac{\partial u}{\partial t} \Big|_{x=-1} = \frac{\partial u}{\partial t} \Big|_{x=1}
\end{align}

\subsection{Eikonal Equation}
We formulate the Eiknonal equation for signed distance function (SDF) calculation as:
\begin{align}
    & |\nabla{u}| = 1, \ \ \  &x, t \in [-1, 1] \\
    & u(x_s) = 0, & x_s \in \mathcal{S} \\
    & u(x, -1), u(x, 1), u(-1, y), u(1, y) > 0 &  
\end{align}
where $\mathcal{S}$ is zero contour set of the SDF. In training PINN, we use the zero contour constraint as initial condition loss and positive boundary constraint as boundary loss (see Table \ref{tab:imp_details} for details of loss balancing).

\subsection{Kuramoto–Sivashinsky Equation}
We use 1-D Kuramoto–Sivashinsky equation from CausalPINN \cite{wang2022respecting}:
\begin{equation}
    \frac{\partial{u}}{\partial{t}} + \alpha u \frac{\partial{u}}{\partial{x}} + \beta \frac{\partial^2{u}}{\partial{x}^2} + \gamma \frac{\partial^4{u}}{\partial{x}^4} = 0,
\end{equation}
subject to periodic boundary conditions and an initial condition
\begin{equation}
u(0, x) = u_0(x)
\end{equation}
The parameter $\alpha$, $\beta$, $\gamma$ controls the dynamical behavior of the equation. We use the same configurations as the CausalPINN: $\alpha = 5, \beta = 0.5, \gamma = 0.005$ for regular settings, and $\alpha = 100/16, \beta = 100/162
, \gamma = 100/164$ for chaotic behaviors.

\section{Details on Skewness and Kurtosis Metrics}
\label{sec:metric_details}

Skewness and kurtosis are two basic metrics used in statistics to characterize the properties of a distribution of values $\{Y_i\}_{i=1}^N$. A high value of Skewness indicates lack of symmetry in the distribution, i.e., the distribution of values to the left and to the right of the center point of the distribution are not identical. On the other hand, a high value of Kurtosis indicates the presence of a heavy-tail, i.e., there are more values farther away from the center of the distribution relative to a Normal distribution.  In our implementation using \textit{scipy}, we used the adjusted Fisher-Pearson coefficient of skewness and Fisher’s definition of kurtosis, as defined below.

\textbf{Skewness}: For univariate data $Y_1, Y_2, ..., Y_N$, the formula of skewness is 

\begin{equation}
    \mbox{skewness} = \frac{\sqrt{N(N-1)}}{N-2} \times
       \frac{\sum_{i=1}^{N}(Y_{i} - \bar{Y})^{3}/N} {s^{3}},
\end{equation}
where $\bar{Y}$ is the sample mean of the distribution and $s$ is the standard deviation. For any symmetric distribution (e.g., Normal distribution), the skewness is equal to zero. A positive value of skewness means there are more points to the right of the center point of the distribution than there are to the left. Similarly, a negative value of skewness means there are more points to the left of the center point than there are to the right. In our use-case, a large positive value of skewness of the PDE residuals indicates that there are some asymmetrically high PDE residual values to the right.

\textbf{Kurtosis}: Kurtosis is the fourth central moment divided by the square of the variance after subtracting 3, defined as follows:

\begin{equation}
    \mbox{kurtosis} = \frac{\sum_{i=1}^{N}(Y_{i} - \bar{Y})^{4}/N}
      {s^{4}}  - 3
\end{equation}
For a Normal distribution, Kurtosis is equal to 0. A positive value of Kurtosis indicates that there are more values in the tails of the distribution than what is expected from a Normal distribution. On the other hand, a negative value of Kurtosis indicates that there are lesser values in the tails of the distribution relative to a Normal distribution. In our use-case, a large positive value of Kurtosis of the PDE residuals indicates that there are some high PDE residual values occurring in very narrow regions of the spatio-temporal domain, that are being picked up as the heavy-tails of the distribution.

%% file: Appendix/hyperparams.tex
\section{Hyper-parameter Settings and Implementation Details}
\label{sec:hyperparam_setting}

The hyper-parameter settings for the different baseline methods for every benchmark PDE are provided in Table \ref{tab:imp_details}. Note that we used the same network architecture and other hyper-parameter settings across all baseline method implementations for the same PDE. In this table, the column on `r/ic/bc' represents the setting of the $\lambda_{r}, \lambda_{ic}, \lambda_{bc}$ hyper-parameters that are used to weight the different loss terms in the overall learning objective of PINNs.  Table \ref{tab:imp_details} also lists the type of Optimizer, learning rate (lr), and learning rate scheduler (lr.scheduler) used across all baselines for every PDE. For the Eikonal equation, we used the same modified multi-layer perceptron (MLP) architecture as the one proposed in \citep{wang2020understanding}. Additionally, 
for the Causal R3 Sampling method, we used the following hyper-parameter settings across all PDEs: $\alpha = 5$, learning rate of the gate $\eta_{g} = 1e-3$, tolerance $\epsilon = 20$, initial value of $\beta = -0.5$, and $\Delta_{max} = 0.1$. The number of iterations (and the corresponding PDE coefficients for the Convection Equation) are provided in Section \ref{sec:results} of the main paper.

\begin{table}[ht]
\centering
\setlength\tabcolsep{2pt} % let LaTeX figure out amount of intercolumn whitespace
\caption{Hyper-parameter settings for different baseline methods for every benchmark PDE}
\label{tab:imp_details}
\begin{footnotesize}

\begin{tabular}{cccccccc}
\hline
PDE                         & Method                                                                    & Architecture                                                   & \begin{tabular}[c]{@{}c@{}}Periodic\\ Encoding\end{tabular} & r/ic/bc   & \begin{tabular}[c]{@{}c@{}}Optimizer/\\ lr\end{tabular} & \begin{tabular}[c]{@{}c@{}}lr. \\ scheduler\end{tabular}               & \begin{tabular}[c]{@{}c@{}}RAR Hyperparams\\ k/m/N/$|\mathcal{S}|$\end{tabular}                      \\ \hline
\multirow{2}{*}{Convection} & \begin{tabular}[c]{@{}c@{}}PINN,\\ cPINN,\\ R3,\\ Causal R3\end{tabular} & \begin{tabular}[c]{@{}c@{}}50 x 4\\ (MLP)\end{tabular}         & No                                                          & 1/100/100 & \begin{tabular}[c]{@{}c@{}}Adam/\\ 1e-3\end{tabular}    & \begin{tabular}[c]{@{}c@{}}StepLR\\ rate=0.9\\ steps=5000\end{tabular} & N/A                                                                                      \\
                            & Curr. Reg.                                                                & \begin{tabular}[c]{@{}c@{}}50 x 4\\ (MLP)\end{tabular}         & No                                                          & 1/1/1     & \begin{tabular}[c]{@{}c@{}}Adam/\\ 1e-4\end{tabular}    & No                                                                     & N/A                                                                                      \\
\multicolumn{1}{l}{}        & \begin{tabular}[c]{@{}c@{}}RAR-G\\ RAD\\ RAR-D\end{tabular}               & \begin{tabular}[c]{@{}c@{}}50 x 4\\ (MLP)\end{tabular}         & No                                                          & 1/100/100 & \begin{tabular}[c]{@{}c@{}}Adam/\\ 1e-3\end{tabular}    & \begin{tabular}[c]{@{}c@{}}StepLR\\ rate=0.9\\ steps=5000\end{tabular} & \begin{tabular}[c]{@{}c@{}}-/1/100/100000\\ 1/1/100/100000\\ 1/1/100/100000\end{tabular} \\ \hline
Allen Cahn                  & \begin{tabular}[c]{@{}c@{}}PINN,\\ cPINN,\\ R3,\\ Causal R3\end{tabular} & \begin{tabular}[c]{@{}c@{}}128x4\\ (MLP)\end{tabular}          & Yes                                                         & 1/100/100 & \begin{tabular}[c]{@{}c@{}}Adam/\\ 1e-3\end{tabular}    & \begin{tabular}[c]{@{}c@{}}StepLR\\ rate=0.9\\ steps=5000\end{tabular} & N/A                                                                                      \\
\multicolumn{1}{l}{}        & \begin{tabular}[c]{@{}c@{}}RAR-G\\ RAD\\ RAR-D\end{tabular}               & \begin{tabular}[c]{@{}c@{}}128x4\\ (MLP)\end{tabular}          & Yes                                                         & 1/100/100 & \begin{tabular}[c]{@{}c@{}}Adam/\\ 1e-3\end{tabular}    & \begin{tabular}[c]{@{}c@{}}StepLR\\ rate=0.9\\ steps=5000\end{tabular} & \begin{tabular}[c]{@{}c@{}}-/1/100/100000\\ 1/1/100/100000\\ 1/1/100/100000\end{tabular} \\ \hline
Eikonal                     & \begin{tabular}[c]{@{}c@{}}PINN,\\ R3\end{tabular}                       & \begin{tabular}[c]{@{}c@{}}128x4\\\citep{wang2020understanding}\\ (modified MLP)\end{tabular} & No                                                          & 1/500/10  & \begin{tabular}[c]{@{}c@{}}Adam/\\ 1e-3\end{tabular}    & \begin{tabular}[c]{@{}c@{}}StepLR\\ rate=0.9\\ steps=5000\end{tabular} & N/A                                                                                      \\ \hline
\end{tabular}
\end{footnotesize}
\end{table}

\textbf{Hardware Implementation Details}: We trained each of our models on one Nvidia Titan RTX 24GB GPU.

%% file: Appendix/additional_viz.tex
\section{Additional Discussion of Results}
\label{sec:additional_results}

\subsection{Visualizing Propagation Failure for Different Settings of $\beta$}
\label{sec:prop_failure_viz}

In Figure \ref{fig:propagation_hypothesis} of the main paper, we demonstrated the phenomenon of propagation failure for convection equation with $\beta=50$, which was characterized by large values of Skewness and Kurtosis in the PDE residual fields for a large number of iterations (or epochs), and a simultaneous stagnation in the relative error values even though the mean PDE residual kept on decreasing. Here, in Figure \ref{fig:propagation_failure2}, we show that the same phenomenon can be observed for other large values of $\beta > 10$, namely, $\beta = 30,50,70$. We can see that the relative errors for all these three cases remains high even though the PDE residual loss keeps on decreasing with iterations. We can also see that the absolute values of skewness and kurtosis increase as we increase $\beta$, indicating higher risks of propagation failure. In fact, for $\beta = 30$, we can even see that the epoch that marks an abrupt increase in skewness and kurtosis (around 50K iterations) also shows a sudden increase in the relative error at the same epoch, highlighting the connection between imbalanced PDE residuals and the phenomenon of propagation failure.
\begin{figure}[ht]
    \centering
    \includegraphics[width=\textwidth]{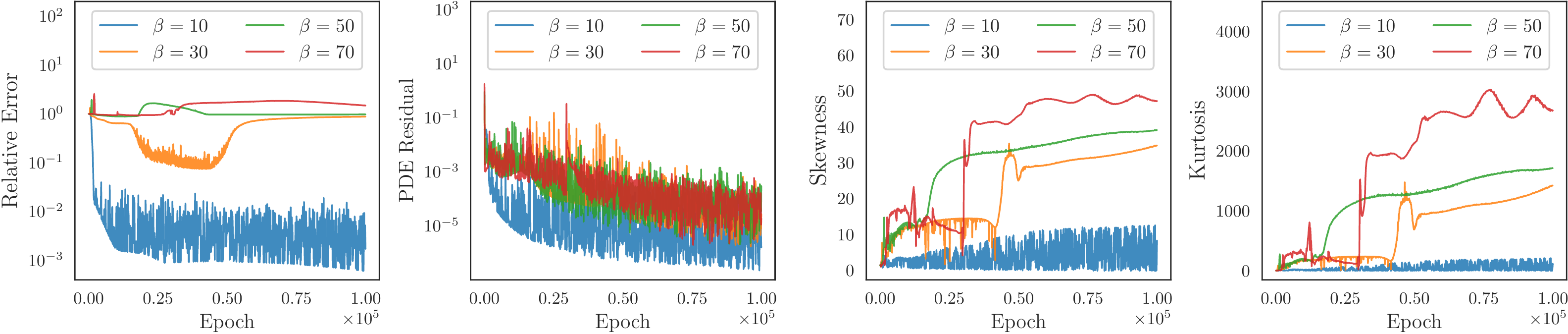}
    \vspace{-3ex}
    \caption{Demonstration of Propagation Failure for Different Settings of $\beta$ ($\beta = 10,30,50,70$)}
    \label{fig:propagation_failure2}
\end{figure}

\subsection{Visualizing Propagation Failure for Kuramoto-Shivashinksy Equation}
\label{sec:prop_failure_viz_ks_equation}

Figure \ref{fig:propagation_failure_ks_eq} is similar to Figure \ref{fig:propagation_hypothesis}, and demonstrates that during the training of conventional PINNs although the PDE residuals steadily decrease with iterations, the skewness and kurtosis quickly become very large. This behavior suggests the presence of small regions with very high PDE residuals, which we refer to as the “propagation failure” of PINNs. Figure \ref{fig:viz_propagation_failure_ks_eq} shows the Ground-truth solution, the predicted PDE field and the PDE Residuals after 100k iterations. It can be seen that there exists small regions or sharp boundaries of high error regions which we hypothesise is the effect for propagation failure.

\begin{figure}[ht]
    \centering
    \includegraphics[width=1.\textwidth]{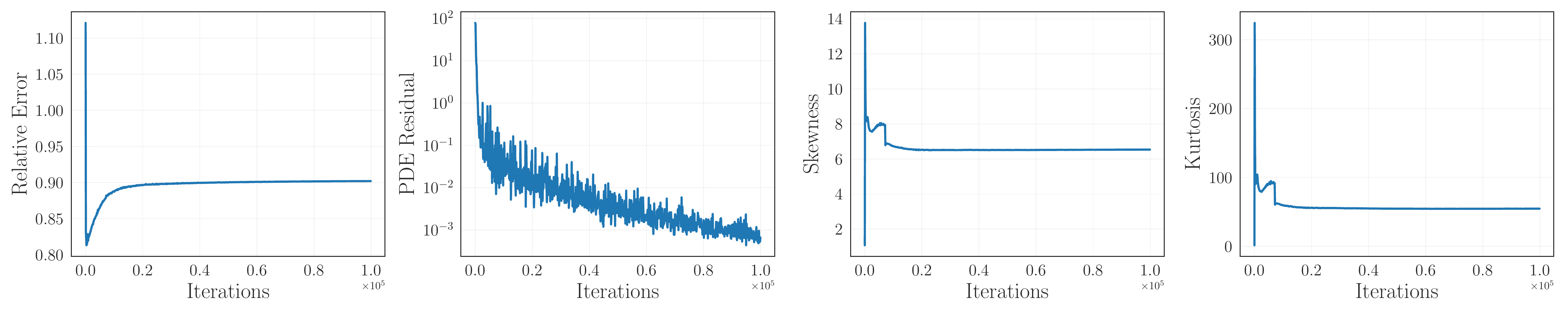}
    \vspace{-3ex}
    \caption{Demonstration of Propagation Failure for Kuramoto-Shivashinksky Equation.}
    \label{fig:propagation_failure_ks_eq}
\end{figure}

\begin{figure}[ht]
    \centering
    \includegraphics[width=0.8\textwidth]{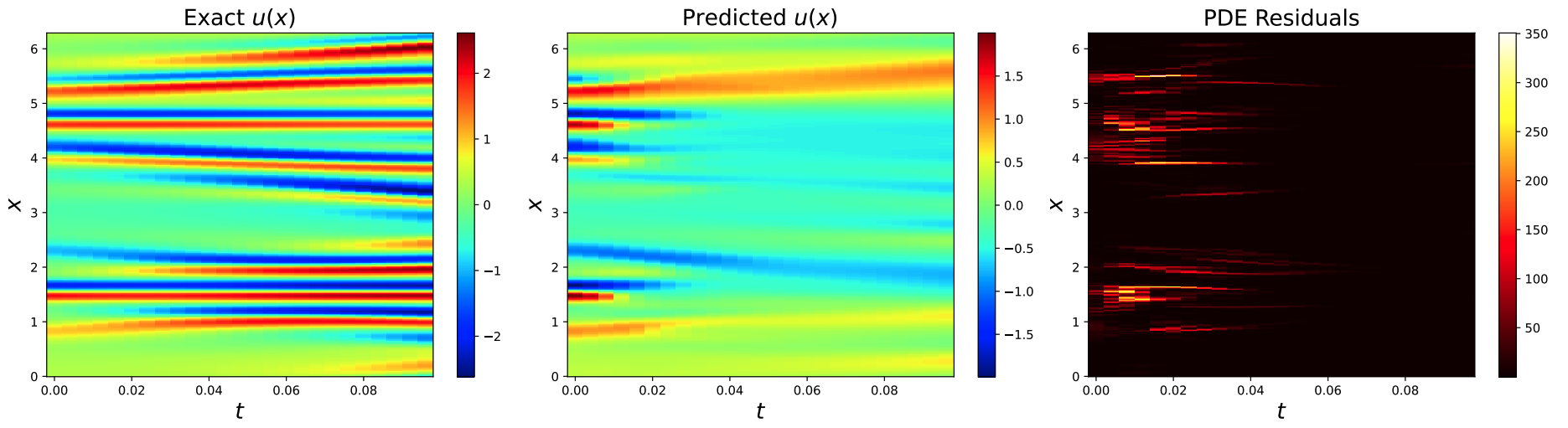}
    \vspace{-3ex}
    \caption{Visualization of High PDE Residuals due to Propagation Failure for Kuramoto-Shivashinksky Equation (Chaotic).}
    \label{fig:viz_propagation_failure_ks_eq}
\end{figure}

\subsection{Additional Results for Sample Efficiency of R3 Sampling}
\label{sec:additional_sample_efficiency}

We provide the results for the sample efficiency results for the Convection Equation and Allen Cahn Equations in Table \ref{tab:sample_eff_convection} and \ref{tab:sample_eff_allancahn} respectively, averaged across 3 random runs.

\begin{table}[htbp]
\centering
\setlength\tabcolsep{3pt} % let LaTeX figure out amount of intercolumn whitespace
\fontsize{8.1pt}{9}\selectfont
\caption{\fontsize{10pt}{11}\selectfont Relative $\mathcal{L}_2$ errors (in $\%$) of comparative methods for Convection Equation with $\beta=50$ for different number of collocation points $N_s$ averaged across 3 random runs.\\}
\label{tab:sample_eff_convection}
\begin{tabular}{lccccccc}
\hline
\textbf{Method} & \textbf{$N_s$ = 100} & \textbf{$N_s$ = 500} & \textbf{$N_s$ = 1000} & \textbf{$N_s$ = 2000} & \textbf{$N_s$ = 5000} & \textbf{$N_s$ = 10000} & \textbf{$N_s$ = 20000} \\ \hline
PINN (fixed) & $103.0 \pm 0.97$ & $121.0 \pm 13.1$ & $108.0 \pm 11.5$ & $114.0 \pm 22.8$ & $135.0 \pm 15.4$ & $7.04 \pm 7.81$ & $3.35 \pm 1.91$ \\
PINN (dynamic) & $66.2 \pm 4.87$ & $51.5 \pm 26.8$ & $52.2 \pm 17.8$ & $63.0 \pm 7.22$ & $4.93 \pm 1.59$ & $5.90 \pm 3.17$ & $3.14 \pm 1.76$ \\
CausalPINN & $89.0 \pm 9.49$ & $66.9 \pm 11.2$ & $72.5 \pm 3.82$ & $43.4 \pm 37.1$ & $44.9 \pm 36.3$ & $61.5 \pm 13.8$ & $61.8 \pm 6.77$ \\
RAD & $83.4 \pm 11.3$ & $28.9 \pm 34.9$ & $67.1 \pm 1.57$ & $51.0 \pm 29.9$ & $23.5 \pm 32.1$ & $23.5 \pm 32.1$ & $23.5 \pm 32.1$ \\
R3 & $4.27 \pm 2.11$ & $2.22 \pm 1.26$ & $1.47 \pm 0.45$ & $1.93 \pm 0.66$ & $3.27 \pm 2.70$ & $2.77 \pm 1.52$ & $1.96 \pm 0.33$ \\
CausalR3 & $5.54 \pm 3.52$ & $4.59 \pm 2.87$ & $1.14 \pm 0.11$ & $5.69 \pm 4.10$ & $6.51 \pm 5.50$ & $4.11 \pm 2.86$ & $1.52 \pm 9.69$ \\ \hline
\end{tabular}
\end{table}

\begin{table}[htbp]
\centering
\setlength\tabcolsep{3pt} % let LaTeX figure out amount of intercolumn whitespace
\fontsize{8.1pt}{9}\selectfont
\caption{\fontsize{10pt}{11}\selectfont Relative $\mathcal{L}_2$ errors (in $\%$) of comparative methods for Allen Cahn Equation for different number of collocation points $N_s$ averaged across 3 random runs.\\}
\label{tab:sample_eff_allancahn}
\begin{tabular}{lccccccc}
\hline
\textbf{Method} & \textbf{$N_s$ = 100} & \textbf{$N_s$ = 500} & \textbf{$N_s$ = 1000} & \textbf{$N_s$ = 2000} & \textbf{$N_s$ = 5000} & \textbf{$N_s$ = 10000} & \textbf{$N_s$ = 20000} \\ \hline
PINN (fixed) & $70.4 \pm 5.49$ & $70.1 \pm 0.28$ & $62.4 \pm 1.80$ & $56.0 \pm 3.70$ & $76.4 \pm 21.9$ & $51.2 \pm 0.02$ & $5.11 \pm 4.32$ \\
PINN (dynamic) & $2.88 \pm 3.67$ & $2.93 \pm 2.22$ & $3.69 \pm 5.00$ & $0.78 \pm 0.10$ & $0.83 \pm 0.16$ & $0.82 \pm 0.04$ & $0.78 \pm 0.02$ \\
CausalPINN & $75.4 \pm 9.43$ & $22.9 \pm 6.85$ & $23.6 \pm 21.2$ & $5.50 \pm 3.32$ & $2.53 \pm 1.23$ & $1.90 \pm 1.68$ & $0.79 \pm 0.10$ \\
RAD & $50.9 \pm 0.13$ & $21.3 \pm 23.2$ & $5.10 \pm 4.23$ & $1.08 \pm 0.40$ & $1.92 \pm 2.06$ & $0.82 \pm 0.11$ & $0.73 \pm 0.02$ \\
R3 & $0.86 \pm 0.22$ & $0.69 \pm 0.03$ & $0.81 \pm 0.16$ & $0.75 \pm 0.07$ & $0.69 \pm 0.01$ & $0.81 \pm 0.22$ & $0.81 \pm 0.10$ \\
CausalR3 & $0.70 \pm 0.01$ & $0.69 \pm 0.04$ & $0.71 \pm 0.01$ & $0.69 \pm 0.02$ & $0.71 \pm 0.01$ & $0.71 \pm 0.03$ & $0.72 \pm 0.02$ \\ \hline
\end{tabular}
\end{table}

\subsection{Convergence Speed of R3 Sampling}
\label{sec:convergence_speed}

% Since PINNs are trained to fit a particular solution under a certain PDE configuration, and they do not generalize to unseen examples at inference, thus the training time is part of the computational cost of PINNs. 
Figure \ref{fig:convergence} shows that Causal R3 is able to converge faster to low error solutions than all other baseline methods for both convection and Allen Cahn equations. This shows the importance of respecting causality along with focusing on high residual regions to ensure fast propagation of correct solution from initial/boundary points to interior points. While cPINN-fixed and PINN-dynamic do not converge for convection ($N_r=1K$), we can see that they both converge to lower errors compared to PINN-fixed for Allen Cahn and for convection when $N_r$ is large ($20K$).

\begin{figure}[ht]
\centering
\subfigure[Convection ($\beta =50$, $N_r = 1$k)]{\label{fig:convection_1k} \includegraphics[scale=0.51]{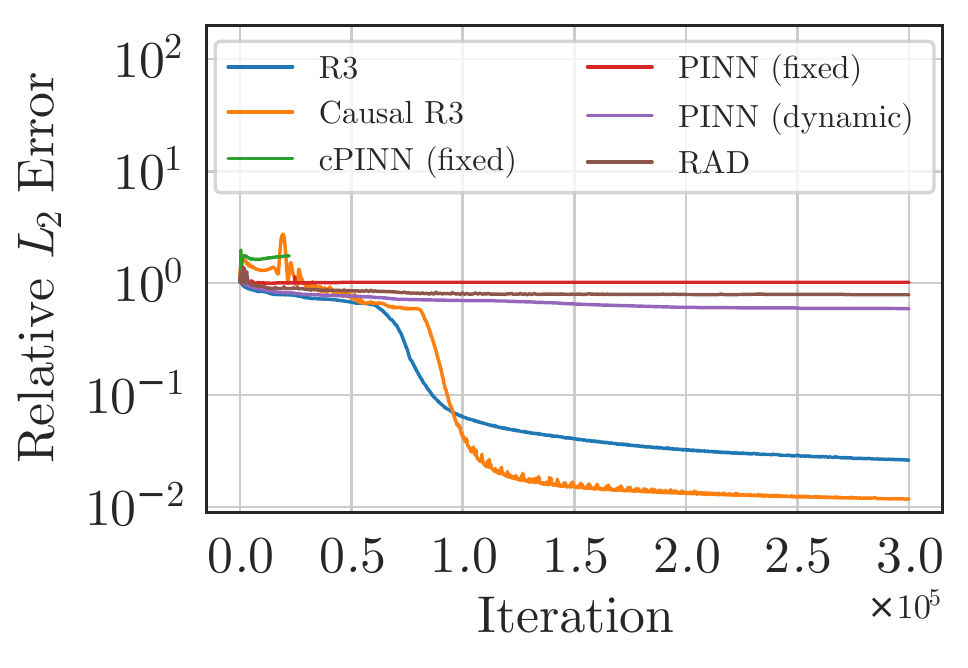}} 
\subfigure[Convection ($\beta =50$, $N_r = 20$k)]{\label{fig:convection_20k} \includegraphics[scale=0.51]{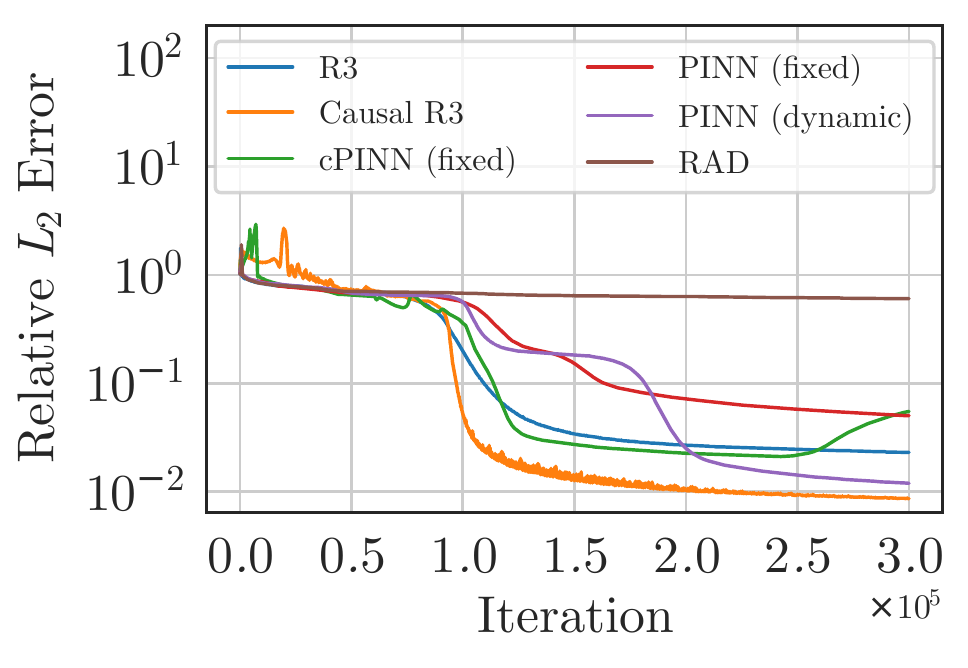}}
\subfigure[Allen Cahn ($N_r = 10$k)]{\label{fig:allen_10k} \includegraphics[scale=0.51]{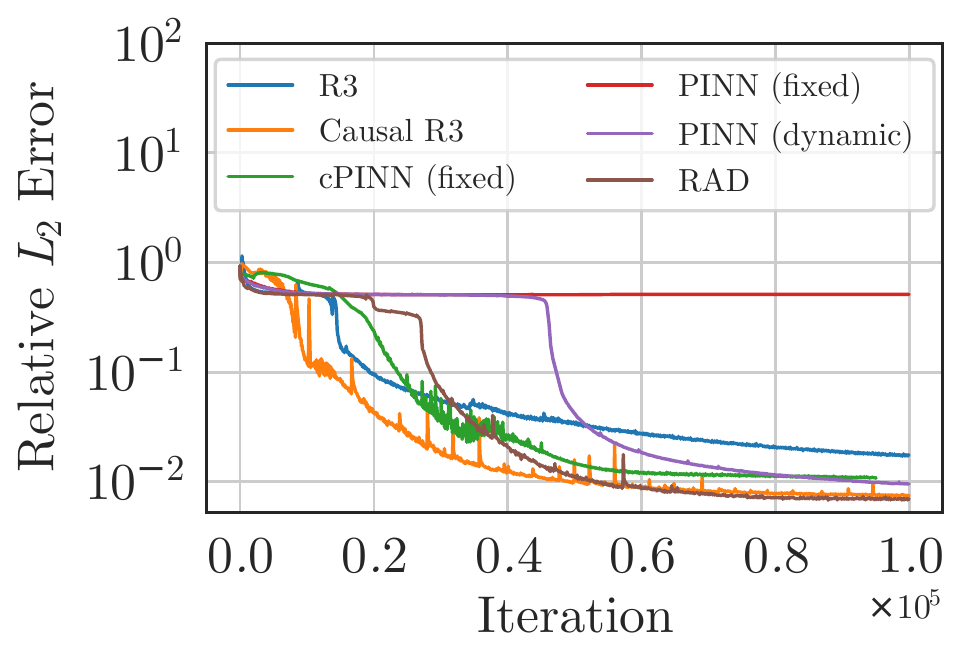}} 
\vspace{-3ex}
\caption{Comparing convergence speeds of baselines w.r.t. R3 sampling and Causal R3.}
\label{fig:convergence}
\end{figure}

% \begin{figure}[t]
% \vspace{-2ex}
% \begin{subfigure}{.3\textwidth}
%   \centering
%   \includegraphics[width=\linewidth]{fig/results/Convection_beta50_convergence_1k.pdf}
%   \vspace{-3ex}
%   \caption{Convection ($\beta =50$, $N_r = 1$k)}
%   \label{fig:convection_1k}
% \end{subfigure}
% \begin{subfigure}{.3\textwidth}
%   \centering
%   \includegraphics[width=\linewidth]{fig/results/Convection_beta50_convergence_20k.pdf}
%   \vspace{-3ex}
%   \caption{\centering Convection ($\beta =50$, $N_r = 20$k)}
%   \label{fig:convection_20k}
% \end{subfigure}
% \begin{subfigure}{.32\textwidth}
%   \centering
%   \includegraphics[width=\linewidth]{fig/results/AllenCahn_convergence_10k.pdf}
%   \vspace{-3ex}
%   \caption{\centering Allen Cahn ($N_r = 10$k)}
%   \label{fig:allen_10k}
% \end{subfigure}
% \vspace{-1ex}
% \caption{
% Comparing convergence speeds of baselines w.r.t. Evo and Causal Evo. 
% }
% \label{fig:convergence}
% \vspace{-2ex}
% \end{figure}

\subsection{Sensitivity of RAR-based Methods}
\label{sec:sensitivity_rar}

Figures \ref{fig:rar_sensitivity} shows the sensitivity of two RAR-based methods: RAR-G and RAR-D respectively on different values of the resampling period $K$ and the number of collocation points $m$ added from the dense set $\mathcal{P}_{dense}$. Note that although the size of the initial set of collocation points $|\mathcal{P}|$ was same for each of these experiments, the final size of the collocation points vary depending on the choice of $K$ (the final size of the collocation points increases as $k$ decreases) and $m$ (the final size of the collocation points increases with $m$). We essentially observe that adding more collocation points almost always improves the performance of these RAR-based methods.

\begin{figure}[ht]
\centering
\subfigure[RAR-G]{\label{fig:rar_g_sens} \includegraphics[scale=0.51]{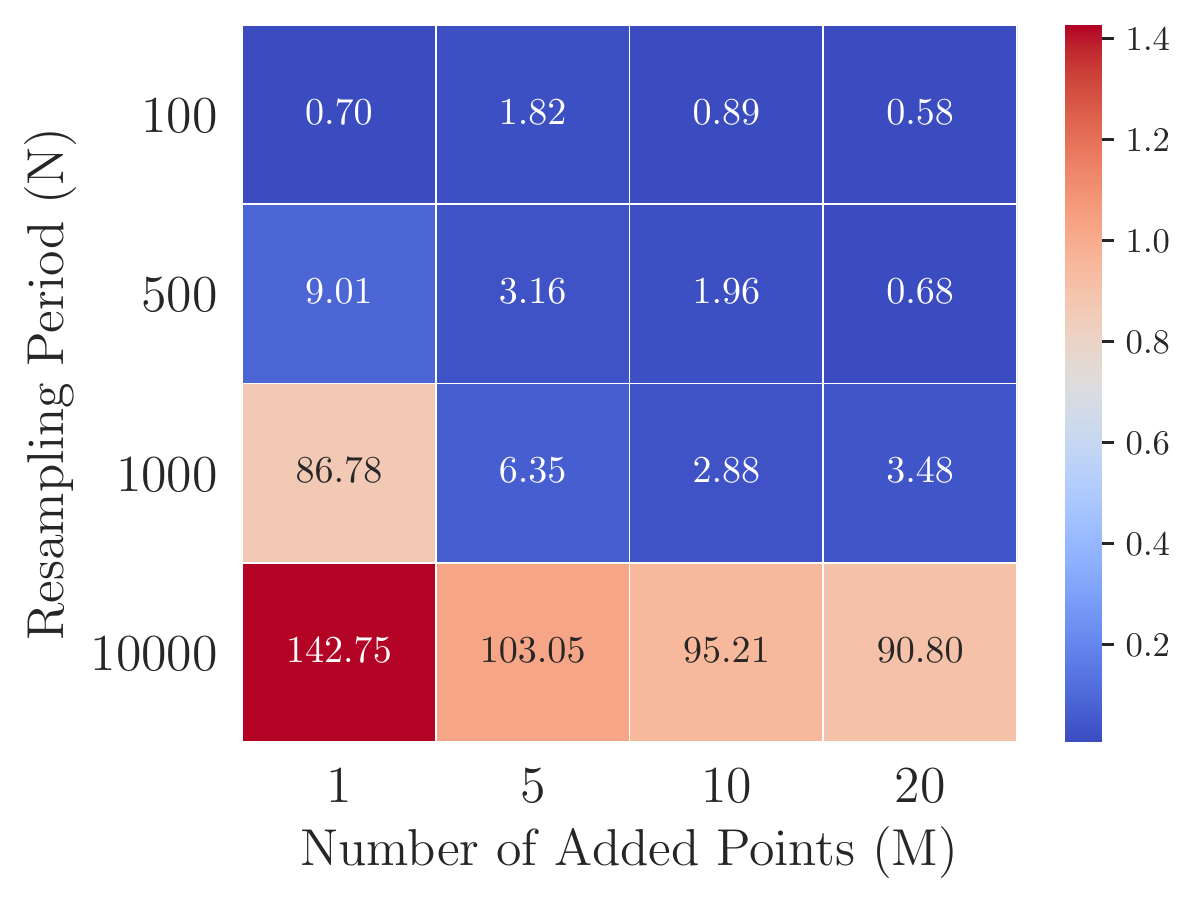}}
\subfigure[RAR-D]{\label{fig:rar_d_sens} \includegraphics[scale=0.51]{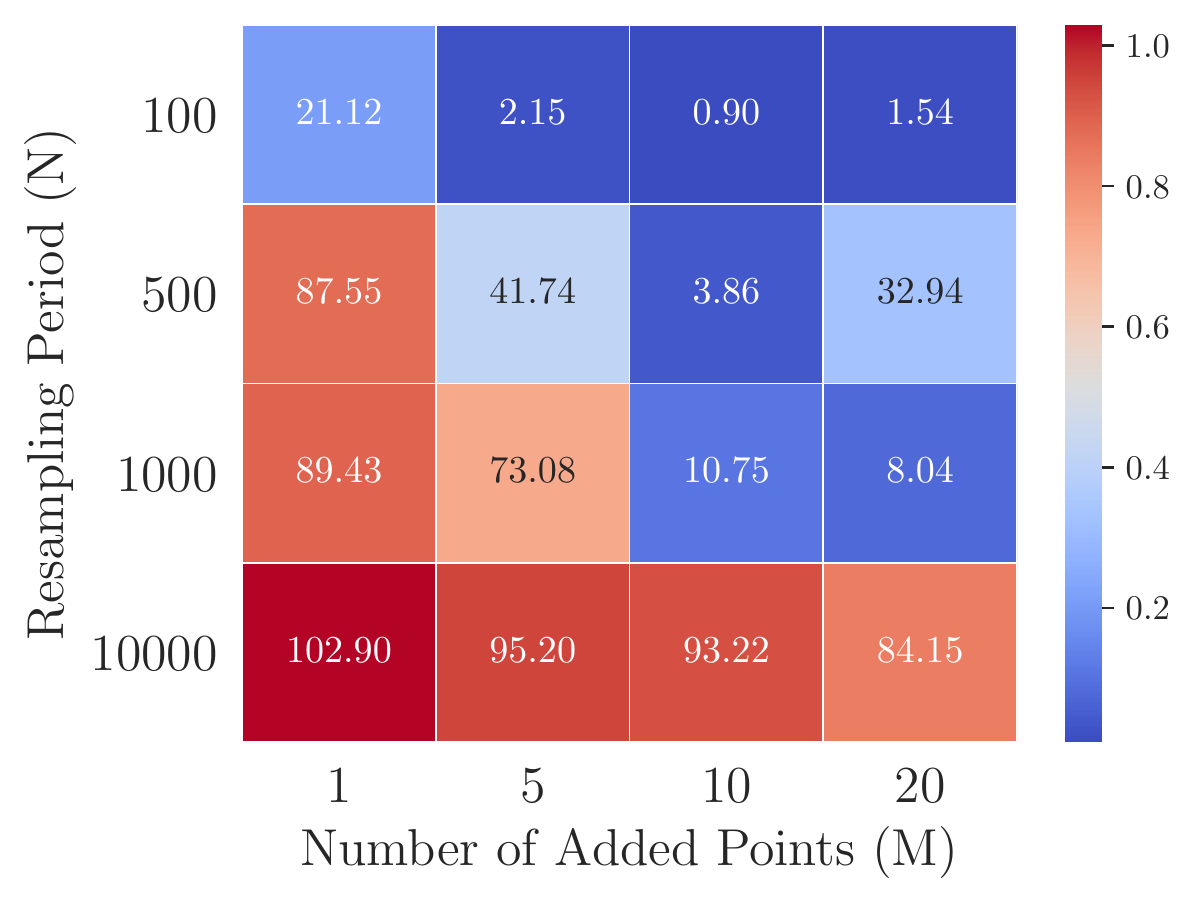}} 
\vspace{-3ex}
\caption{Sensitivity of RAR-based Methods on different values of $K$ and $M$ (for a fixed $|\mathcal{P}_{dense}|=100k$ on Convection Equation with $\beta=30$.) The numbers in the heatmap denote the $\%$ Relative $\mathcal{L}_2$ Error. }
\label{fig:rar_sensitivity}
\end{figure}

% \begin{figure}[ht]
% \centering
% \begin{subfigure}{.37\textwidth}
%   \centering
%   \vspace{1.0ex}
%   \includegraphics[width=\linewidth]{fig/Appendix_Figures/RAR_G_sensitivity.pdf}
%   \vspace{-2.2ex}
%   \caption{RAR-G}
%   \label{fig:rar_g_sens}
% \end{subfigure}
% \begin{subfigure}{.37\textwidth}
%   \centering
%   \includegraphics[width=\linewidth]{fig/Appendix_Figures/RAR_D_sensitivity.pdf}
%   \vspace{-2.5ex}
%   \caption{\centering RAR-D}
%   \label{fig:rar_d_sens}
% \end{subfigure}
% \caption{
% Sensitivity of RAR-based Methods on different values of $K$ and $M$ (for a fixed $|\mathcal{P}_{dense}|=100k$ on Convection Equation with $\beta=30$.) The numbers in the heatmap denote the $\%$ Relative $\mathcal{L}_2$ Error. 
% }
% \label{fig:rar_sensitivity}
% % \vspace{-3ex}
% \end{figure}

\subsection{Visualizing the Evolution of Collocation Points in R3 sampling}
Figure \ref{fig:evosample_propagation_err} shows the evolution of collocation points and PDE residual maps of R3 sampling as we progress in training iterations for the convection equation with $\beta=50$. We can see that the retained population of R3 sampling at every iteration (shown in red) selectively focuses on high PDE residual regions, while the re-sampled population (shown in blue) are generated from a uniform distribution. By increasing the contribution of high residual regions in the computation of the PDE loss, we can see that R3 sampling is able to reduce the PDE loss over iterations without admitting high imbalance, thus mitigating the propagation failure mode, in contrast to conventional PINNs.
% \begin{figure}[ht]
%     \centering
%     \includegraphics[width=\textwidth]{fig/results/EvoSample_propagation.pdf}
%     \vspace{-3ex}
%     \caption{Evolution of collocation points of Evo  for convection equation with $\beta = 50$.}
%     \label{fig:evosample_propagation}
% \end{figure}

\begin{figure}[ht]
    \centering
    \includegraphics[width=\textwidth]{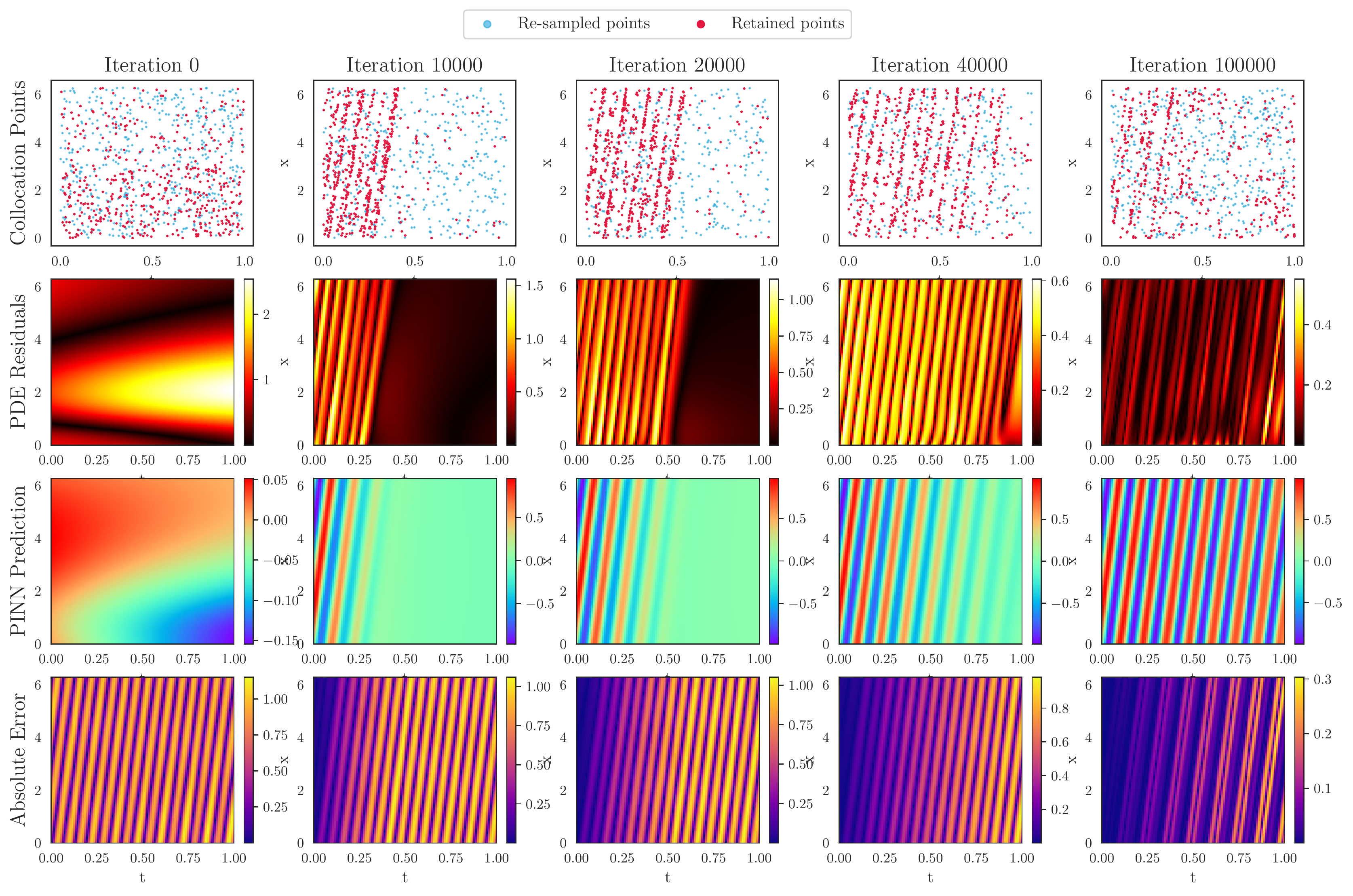}
    \vspace{-3ex}
    \caption{Demonstrating the propagation of information from the initial/boundary points to the interior points for R3 Sampling on Convection Equation($\beta=50$)}
    \label{fig:evosample_propagation_err}
\end{figure}

\begin{figure}[ht]
    \centering
    \includegraphics[width=\textwidth]{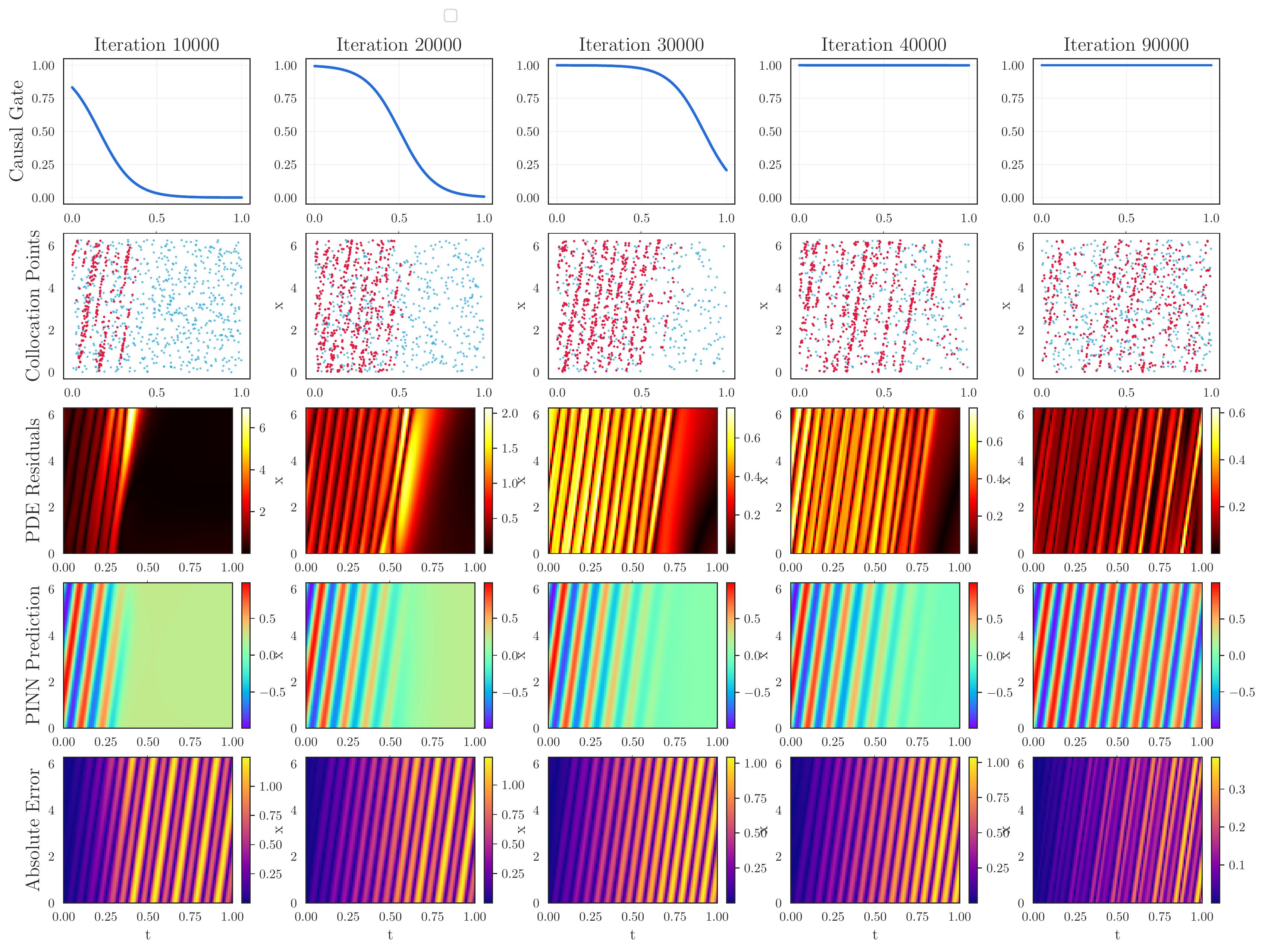}
    \vspace{-3ex}
    \caption{Demonstrating the propagation of information from the initial/boundary points to the interior points for Causal R3 on Convection Equation($\beta=50$)}
    \label{fig:causal_evosample_propagation_err}
\end{figure}

\subsection{Visualizing the Evolution of Collocation Points in Causal R3}
Figure \ref{fig:causal_evosample_propagation_err} shows the evolution of collocation points and PDE residuals of Causal R3, along with the dynamics of the Causal Gate function. We can see that the retained population at every iteration (shown in red) strictly adheres to the principle of causality such that the collocation points are sampled from later times only when the PDE residuals at earlier times have been minimized. This is also reflected in the movement of the causal gate function where the gate values are close to 1 for only small portions of time domain  at intermediate epochs. At 90K iterations, we can see that the causal gate values are close to 1 for all values of time, indicating that the entire time domain is now revealed for training PINN to converge to the correct solution.

\subsection{Visualizing the Dynamics of Collocation Points for RAR-Based Methods}

Figures \ref{fig:rar_g_samples}, \ref{fig:rad_samples}, annd \ref{fig:rar_d_samples}
shows the evolution of collocation points and PDE residuals of RAR-G, RAD, and RAR-D, respectively. We can see that all three RAR-based methods are failing to converge to the correct solution even after 100K iterations, demonstrating their inabillity to mitigate propagation failures.
\begin{figure}[ht]
    \centering
    \includegraphics[width=\textwidth]{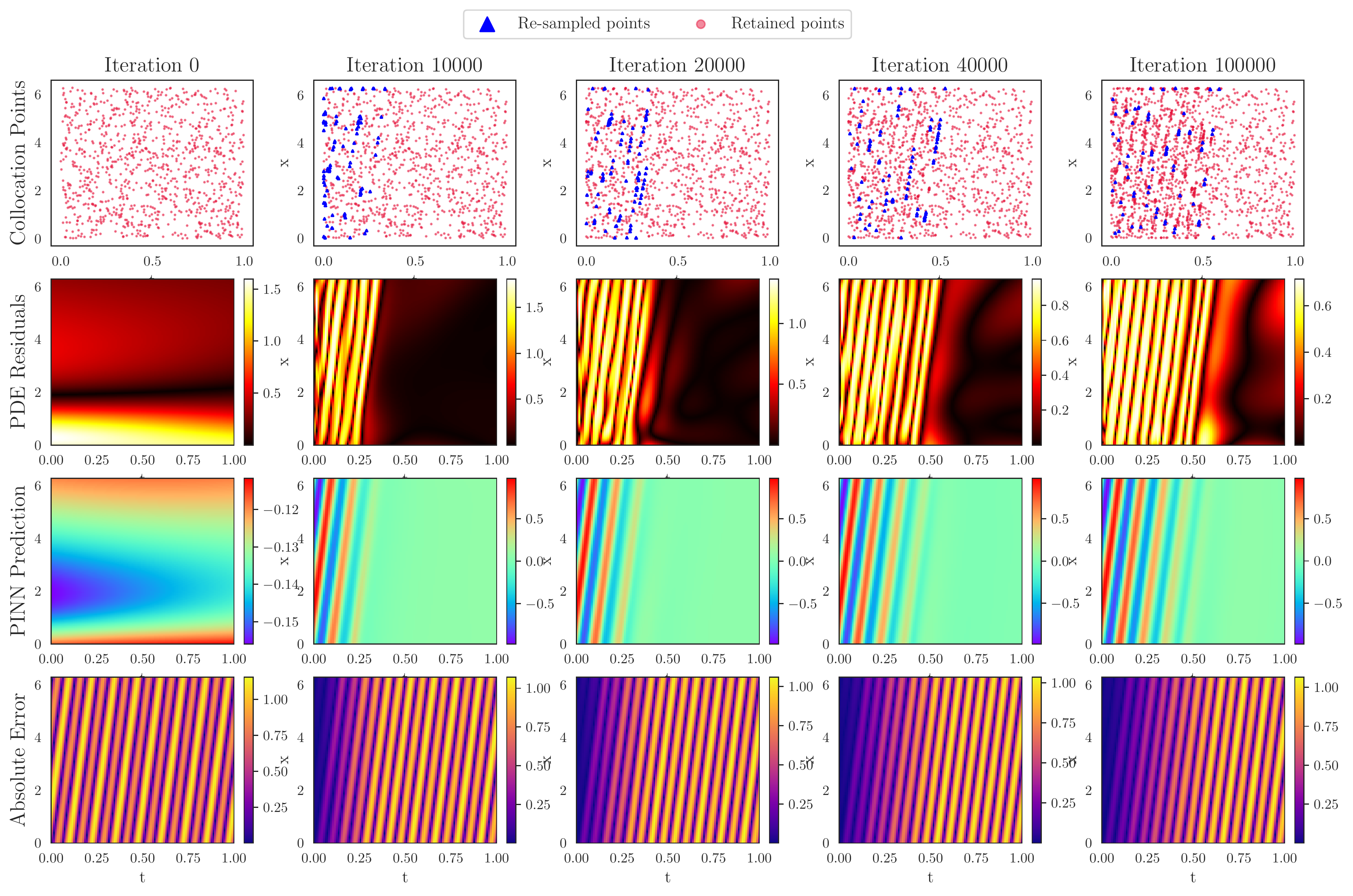}
    \vspace{-3ex}
    \caption{Demonstrating the dynamic changes in the collocation points for RAR-G on Convection Equation($\beta=50$)}
    \label{fig:rar_g_samples}
\end{figure}

\begin{figure}[ht]
    \centering
    \includegraphics[width=\textwidth]{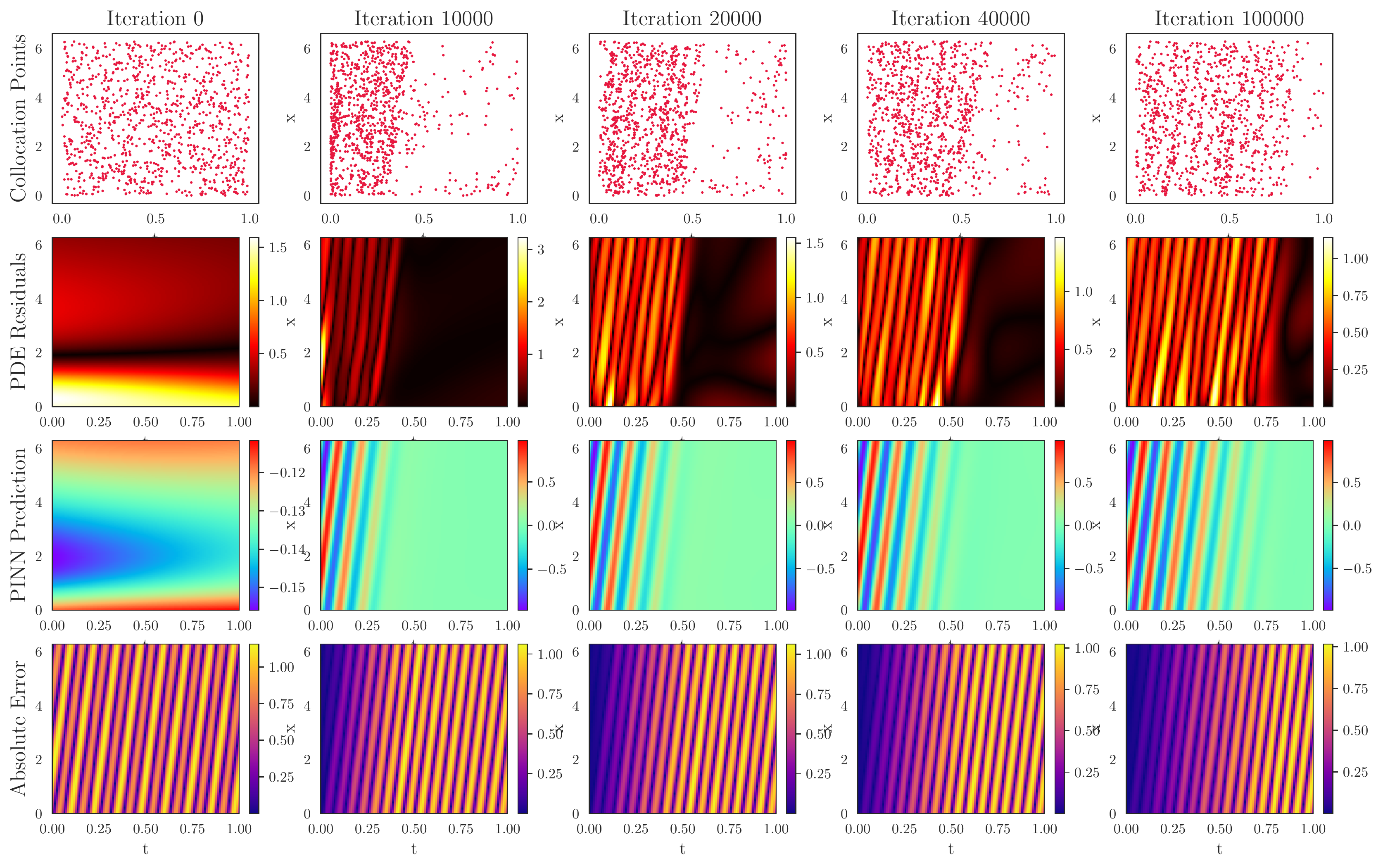}
    \vspace{-3ex}
    \caption{Demonstrating the dynamic changes in the collocation points for RAD on Convection Equation($\beta=50$)}
    \label{fig:rad_samples}
\end{figure}

\begin{figure}[ht]
    \centering
    \includegraphics[width=\textwidth]{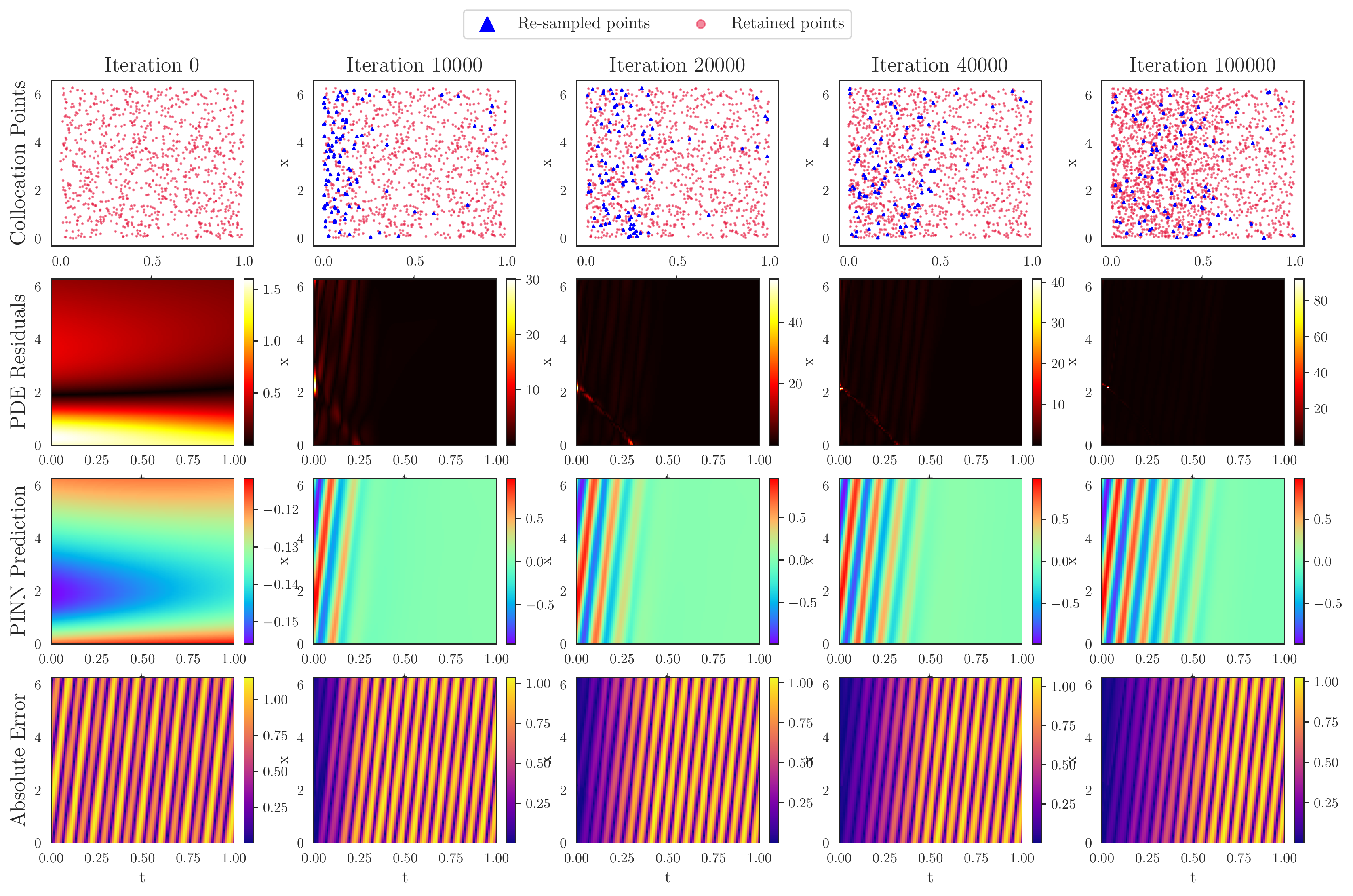}
    \vspace{-3ex}
    \caption{Demonstrating the dynamic changes in the collocation points for RAR-D on Convection Equation($\beta=50$)}
    \label{fig:rar_d_samples}
\end{figure}

\subsection{Kuramoto-Sivashinsky (KS) Equations}
\label{sec:ks_results}

We used three additional experiments on the Kuramoto-Sivashinsky (KS) Equations (one for a regular relatively simple case and the remaining two exhibiting chaotic behavior). Please note that these equations are particularly more complex, especially the chaotic cases where a small change in the state of the solution can result in very large errors downstream in time. Thus, for chaotic domains, the successful propagation of solution from the initial and boundary conditions is critical to guarantee convergence. We would also like to highlight that the computational cost of these experiments are significantly higher. We used the exact same hyper-parameter settings as those provided in CPINN except the sample size, which was varied from 128 to 2048 in the KS-regular case, and the number of training iterations, which was kept as 300k in our proposed approaches while CPINN was allowed to use about 1 M maximum number of iterations with early stopping. Our method on average takes 50-60\% less time than CPINN because of the significantly smaller number of iterations.

% \begin{figure}[ht]
%     \centering
%     \includegraphics[width=0.40\textwidth]{fig/Appendix_Figures/Simple_KS_Vary_Nr.pdf}
%     \vspace{-1ex}
%     \caption{Comparison of R3, Causal R3 and CausalPINNs on Simple KS Equation with varying number of collocation points.}
%     \label{fig:ks_equation_vary_Nr}
% \end{figure}

% Figure \ref{fig:ks_equation_vary_Nr} compares the performance of CPINN, R3 sampling, and Causal R3 on the KS equation (regular case) as a function of the number of collocation points used in PINN training. We can see that both R3 and Causal R3 show improvements over CPINN when the number of collocation points is small ($N_r = 128$). As the number of collocation points is increased, Causal R3 shows better performance than R3, as it incorporates an additional prior of causality along with satisfying the three properties of R3. Overall, Causal R3 mostly performs better than both CPINN and R3 across different training set sizes. Note that these curves have been obtained using a single run of every method due to the computational cost of training for the KS equations, and having multiple runs for every method will help to quantify the variance in these results.

\begin{table}[ht]
\centering
\caption{Relative $\mathcal{L}_2$ errors (in \%) of CausalPINN, R3 sampling and Causal R3 over the three different KS-Equation Benchmarks \cite{wang2022respecting}.}
\label{tab:ks_table_1}
\begin{tabular}{c|c|c|c}
\hline
                   & CausalPINN         & \begin{tabular}[c]{@{}c@{}}R3\\ (ours)\end{tabular} & \begin{tabular}[c]{@{}c@{}}Causal R3\\ (ours)\end{tabular} \\
\hline
Regular            & 2.120\%          & 3.740\%                                           & \textbf{0.761}\%                                          \\
Chaotic            & \textbf{3.272}\% & 6.924\%                                           & 7.630\%                                                   \\
Chaotic - Extended & 52.66\%         & \textbf{29.26}\%                                   & 33.50\%  \\
\hline
\end{tabular}
\end{table}

Table \ref{tab:ks_table_1} compares the performance of CPINN, R3, and Causal R3 on the three KS-Equations cases as was used in the original CPINN paper. Please note that for these experiments, we used the exact same hyper-parameter settings as the original CPINN, thus a large number of collocation points were used (2048 for the regular case and 8192 for the two chaotic regimes for each time-window). We can observe that on the regular case, R3 sampling performs similarly to CPINN, while Causal R3 is significantly better than both. However, in the first chaotic case, CPINN is slightly better than both R3 sampling and Causal R3. Finally, in the much more chaotic regime for the KS-Equation  (extended case), we find that all of the methods struggle to obtain a high fidelity solution of the field. However, R3 sampling and Causal R3 are somewhat better than CPINNs. Hence, we can comment that R3 sampling and Causal R3 have comparable performance to CPINN on their benchmark settings. Additional visualizations of the solutions of comparative methods on the three KS equations cases are provided in Figures \ref{fig:ks_equation_simple_viz}, \ref{fig:ks_equation_chaotic1_viz}, and \ref{fig:ks_equation_chaotic2_viz}.

\begin{figure}[ht]
    \centering
    \includegraphics[width=0.60\textwidth]{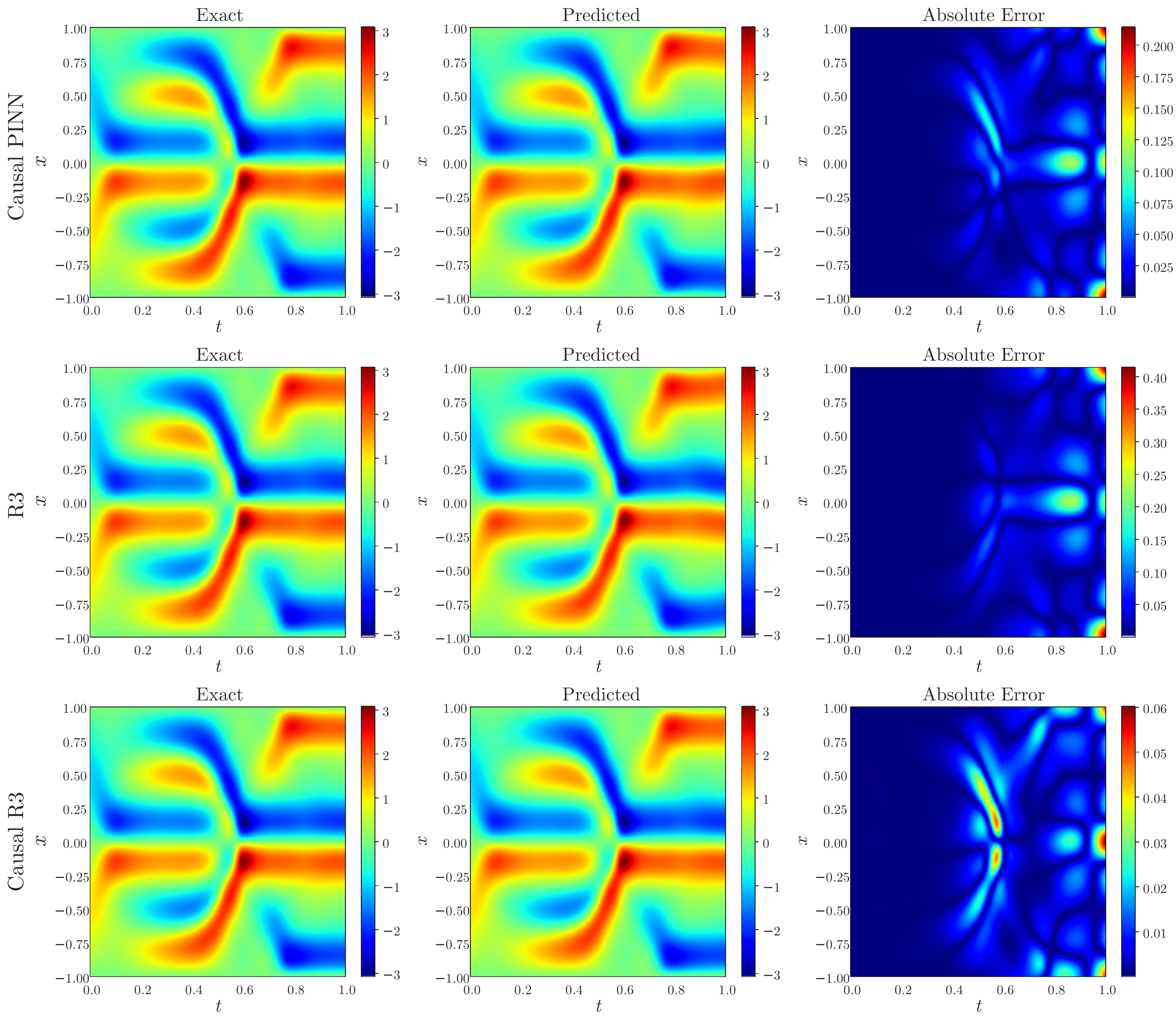}
    \vspace{-1ex}
    \caption{Visualization of the R3 sampling, Causal R3 and CausalPINN on KS-Equation (Regular Case).}
    \label{fig:ks_equation_simple_viz}
\end{figure}

\begin{figure}[ht]
    \centering
    \includegraphics[width=0.60\textwidth]{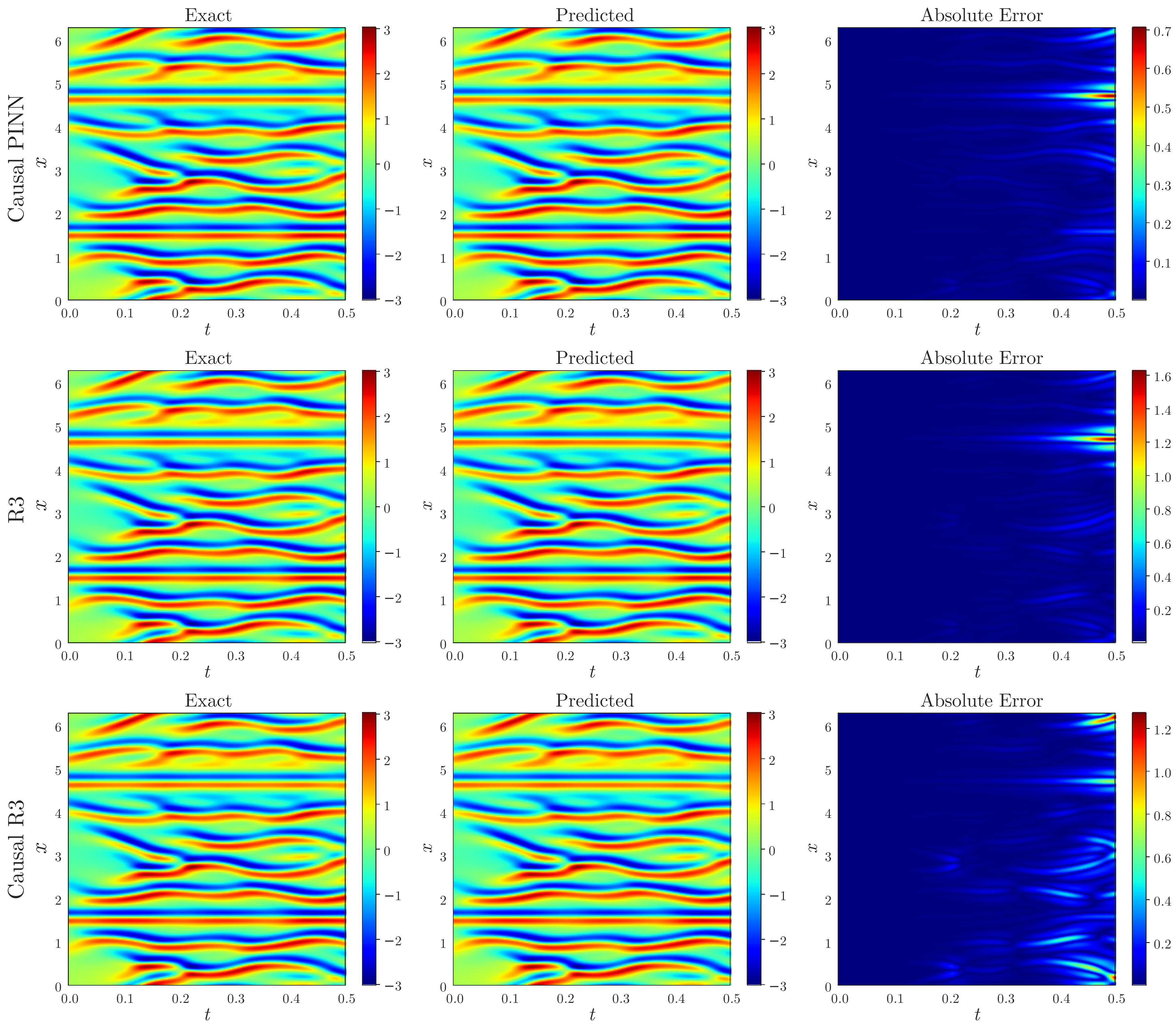}
    \vspace{-1ex}
    \caption{Visualization of the R3, Causal R3 and CausalPINN on KS-Equation (Chaotic Case).}
    \label{fig:ks_equation_chaotic1_viz}
\end{figure}

\begin{figure}[ht]
    \centering
    \includegraphics[width=0.60\textwidth]{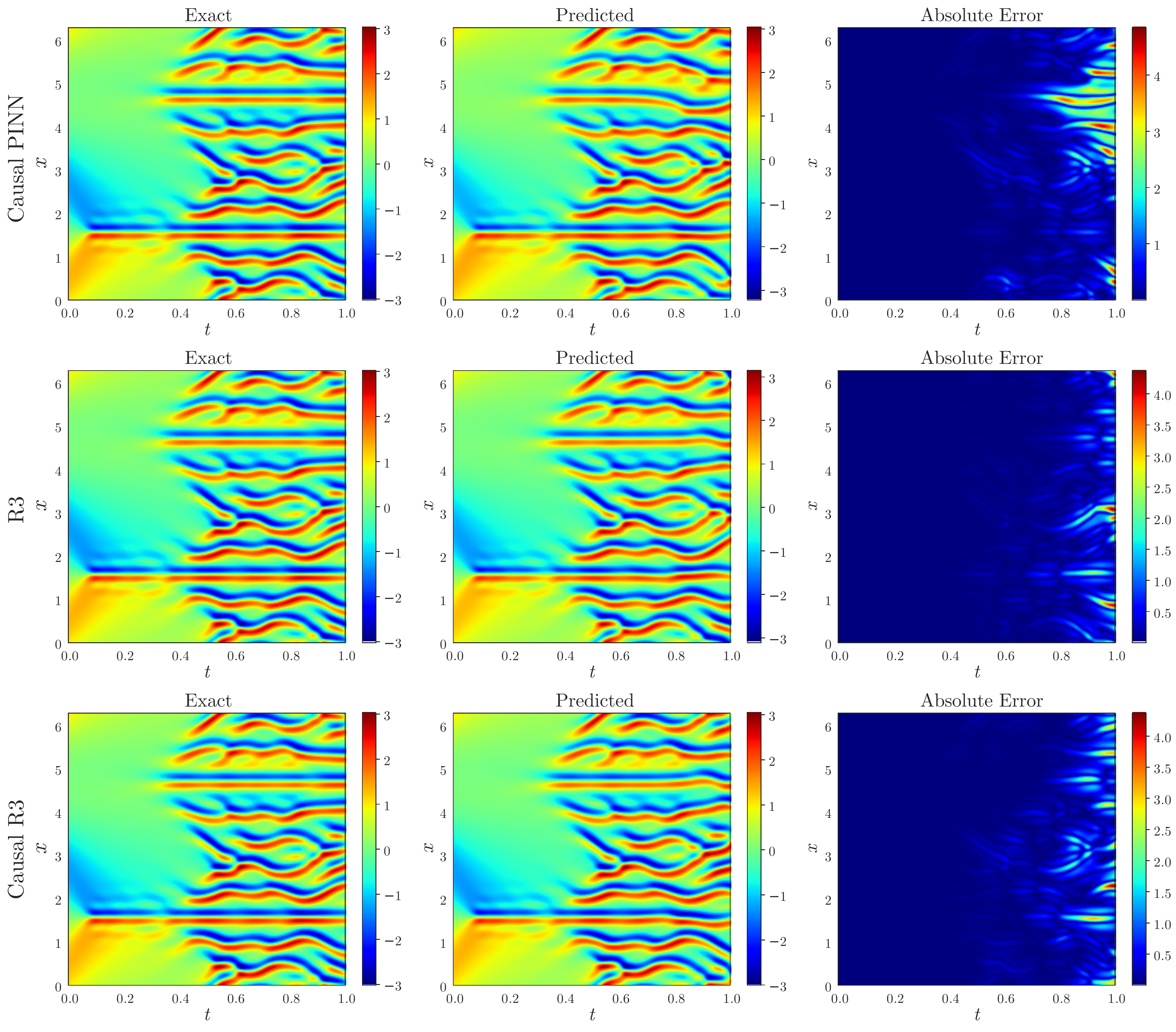}
    \vspace{-1ex}
    \caption{Visualization of the R3, Causal R3 and CausalPINN on KS-Equation (Extended Chaotic Case).}
    \label{fig:ks_equation_chaotic2_viz}
\end{figure}

% \begin{figure}[ht]
%     \centering
%     \includegraphics[width=\textwidth]{fig/Appendix_Figures/CausalEvoSample_propagation.pdf}
%     \vspace{-3ex}
%     \caption{Evolution of collocation points of Causal Evo for convection equation with $\beta = 50$.}
%     \label{fig:causal_evosample_propagation}
% \end{figure}

\subsection{Additional Discussion and Visualization for the Eikonal Equation}
\label{sec:evo_viz}

We chose to solve 2D Eikonal Equations for complex arbitrary surface geometries as they represent particularly hard PDE problems that are susceptible to PINN failure modes. In these problems, we are given the zero contours of the equation on the boundaries (representing the outline of the 2D object), which can take arbitrary shapes. The goal is to correctly propagate the boundary conditions to obtain the unique target solution where the interior is negative and the exterior is positive. Here, any small error in propagation from the boundaries can lead to cascading errors such that a large segment of the predicted field can have opposite signs compared to the ground-truth, even though their PDE residuals are close to 0. Since R3 sampling is explicitly designed to break propagation barriers and thus enable easy transmission of the solution from the boundary to the interior/exterior points, we can see that it shows significantly better performance. On the other hand, PINN (fixed) and PINN (dynamic) struggle to converge to the correct solution especially for complex geometries (e.g., the `gear’) because of the inherent challenge in sampling an adequate number of points from arbitrary shaped object boundaries exhibiting highly imbalanced residuals.

In Figure \ref{fig:eikonal_gear}, we show the evolution of the solutions of comparative methods for the `gear’ case over iterations. We can see that R3 sampling is able to resolve the high residual regions better than the baselines, and thus encounter less incorrect ``sign flips’’ compared to the ground-truth, even in the early iterations of PINN training.

\begin{figure}[ht]
    \centering
    \includegraphics[width=0.75\textwidth]{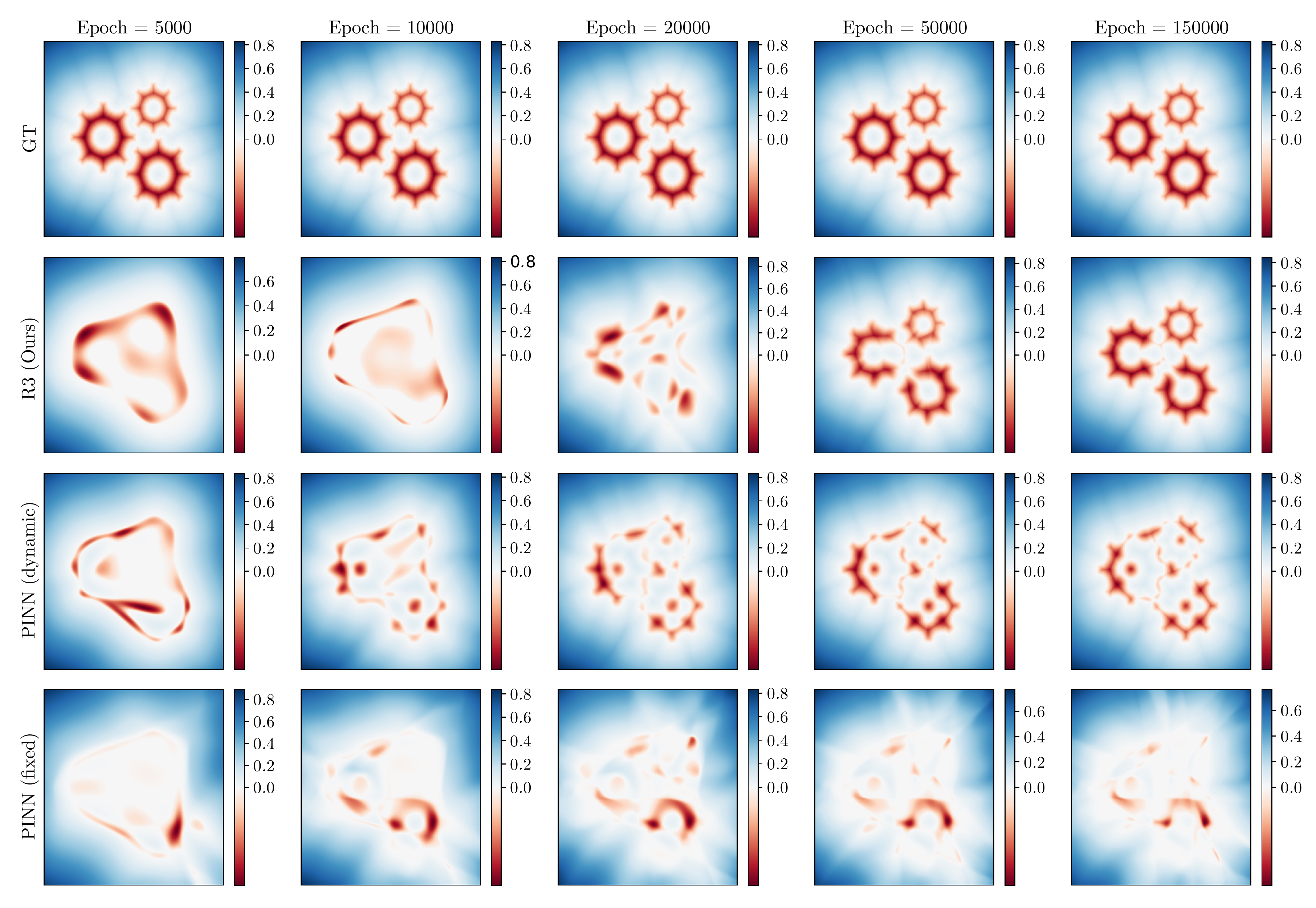}
    \caption{The predicted solutions of Eikonal equation (Figure \ref{fig:eikonal}) at different iterations during training.}
    \label{fig:eikonal_gear}
\end{figure}

%% file: Appendix/test_optim_func.tex
\section{Optimization Characteristics of R3 Sampling on Test Optimization Functions}
\label{sec:test_optim_func}

In this section, we demonstrate the ability of our proposed R3 Sampling Algorithm to find global minimas on various test optimization functions. We will also provide other characterizations of our proposed R3 Sampling algorithm.

\subsection{Auckley Function}
\label{sec:auckley}

The two-dimensional form of the Auckley function has multiple local maximas in the near-flat region of the function and one large peak at the center.

\begin{align}
    f(x) = a  \text{exp}\Bigg(-b \sqrt{\frac{1}{d}\sum_{i=1}^{d}x_i^2}\Bigg) + \text{exp}\Bigg(\frac{1}{d}\sum_{i=1}^{d}\cos(cx_i)\Bigg) + a + \text{exp}(1)
\end{align}

\begin{figure}
    \centering
    \includegraphics[width=0.45\textwidth]{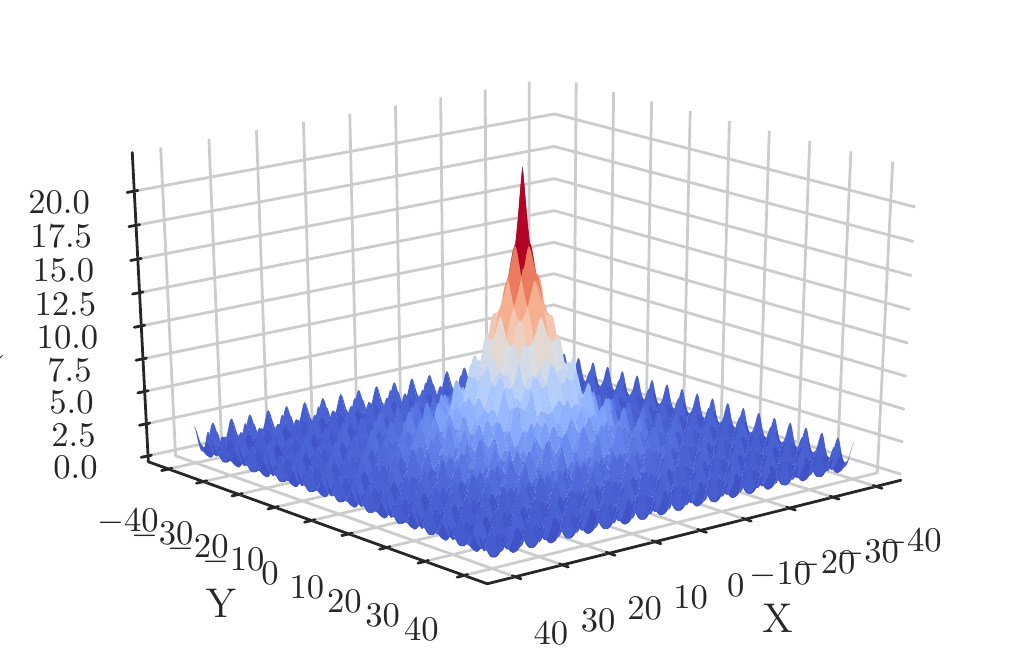}
    \caption{Surface Plot of the 2-D Auckley Function}
    \label{fig:auckley_surface}
\end{figure}

\begin{figure}
    \centering
    \includegraphics[width=1.0\textwidth]{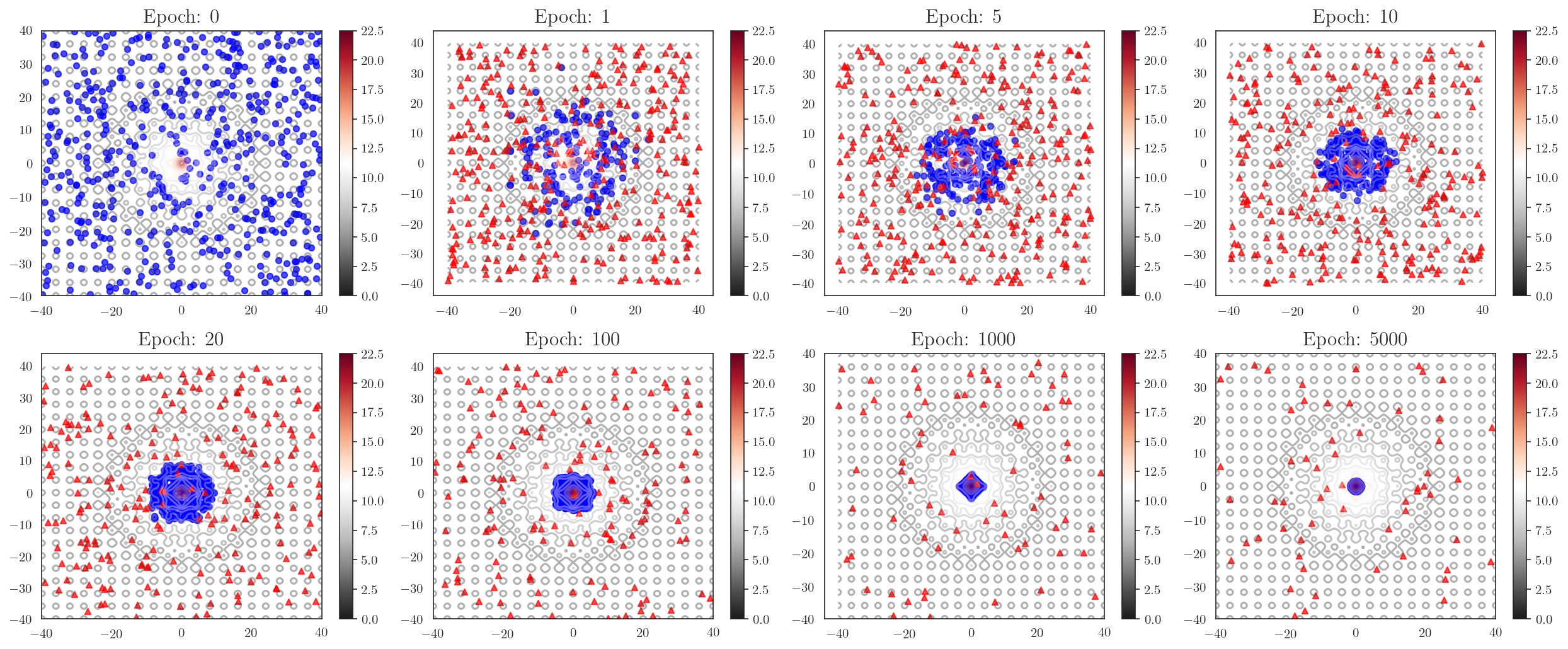}
    \caption{Demonstrating the evolution of the randomly initialized points while optimizing the Auckley function. The red triangles represent the re-sampled population at that epoch, and the blue dots represent the retained population at that epoch. The contour function of the objective function is shown in the background.}
    \label{fig:auckley_evosample}
\end{figure}

\begin{figure}
    \centering
    \includegraphics[width=0.35\textwidth]{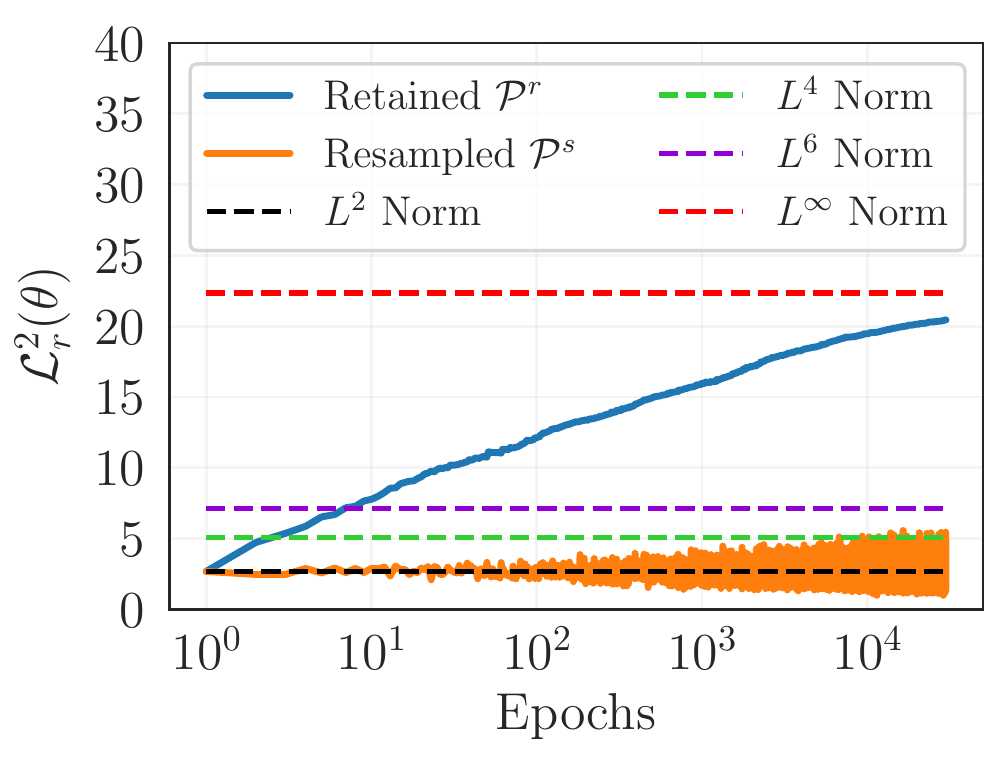}
    \caption{Illustrating the dynamic behavior of the R3 Sampling algorithm on Auckley Function using the $L^2$ Physics-informed Loss computed on the retained and re-sampled populations. The horizontal lines represent the $L^p$ Physics-informed Loss on a dense set of uniformly sampled collocation points (where $p=2,4,6,\infty$).}
    \label{fig:auckley_dynamic_lp}
\end{figure}

\begin{figure}
    \centering
    \includegraphics[width=0.35\textwidth]{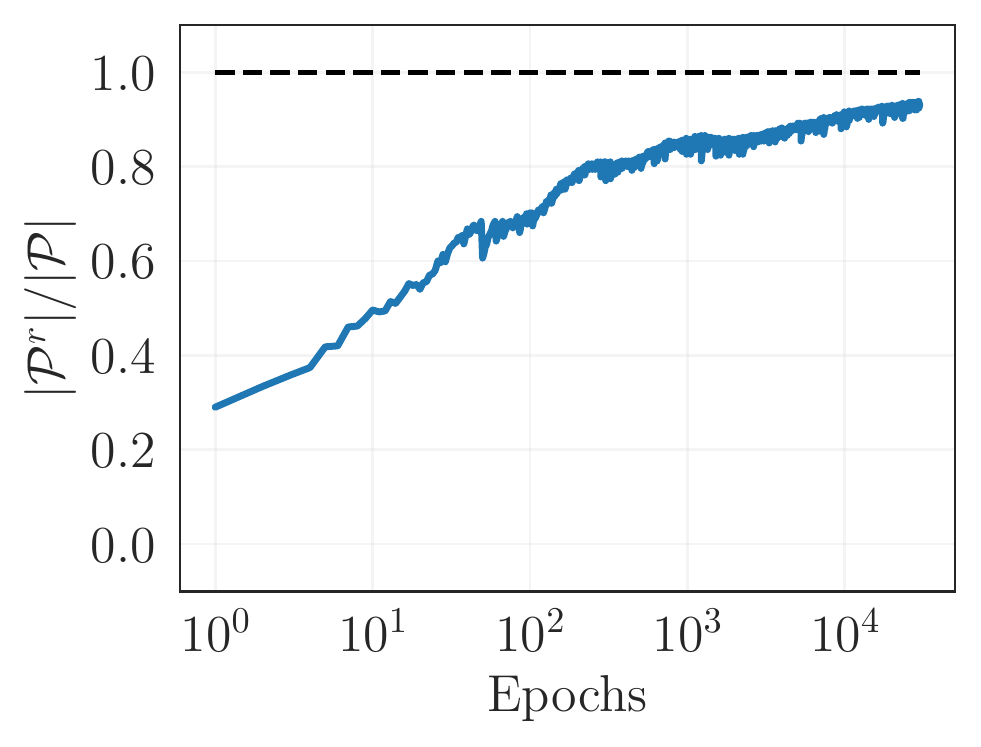}
    \caption{Demonstrating the dynamic evolution of the retained population size over epochs on the Auckley Function.}
    \label{fig:auckley_retained_size}
\end{figure}

\subsection{Bohachevsky Function}
The two-dimensional form of the Bohachevsky function is a bowl shaped function having one global maxima.

\begin{align}
    f(x, y) = - x^2 - 2y^2 + 0.3\cos(3 \pi x) + 0.4cos(4 \pi y) - 0.7
\end{align}

\begin{figure}
    \centering
    \includegraphics[width=0.45\textwidth]{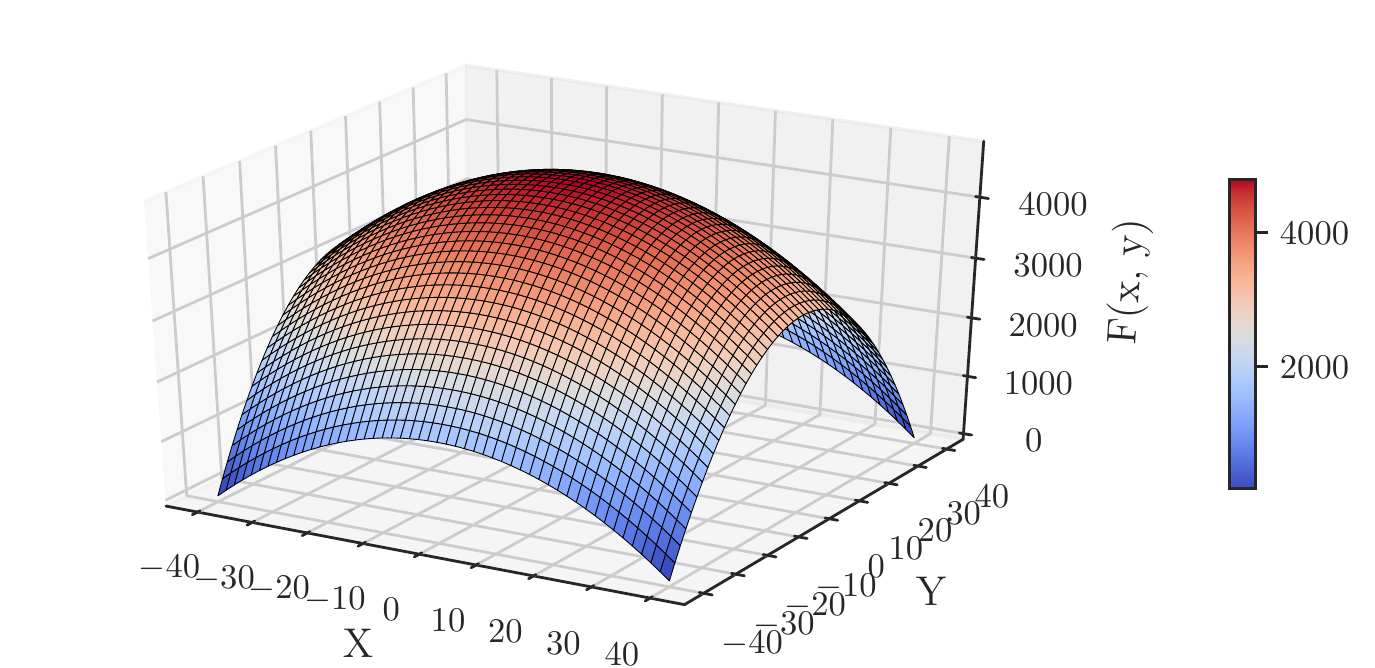}
    \caption{Surface Plot of the 2-D Bohachevsky Function}
    \label{fig:Bohachevsky_surface}
\end{figure}

\begin{figure}
    \centering
    \includegraphics[width=1.0\textwidth]{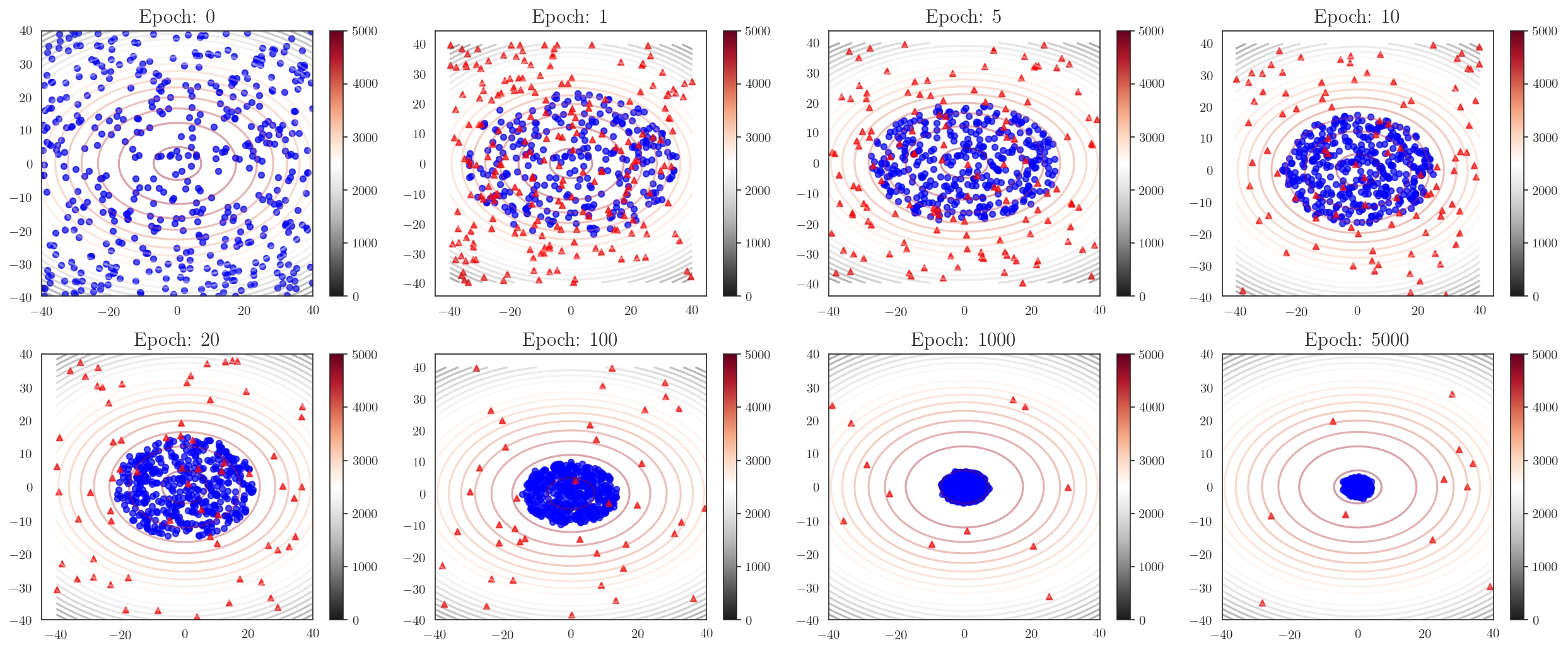}
    \caption{Demonstrating the evolution of the randomly initialized points while optimizing the Bohachevsky function. The red triangles represent the re-sampled population at that epoch, and the blue dots represent the retained population at that epoch. The contour function of the objective function is shown in the background.}
    \label{fig:Bohachevsky_evosample}
\end{figure}

\begin{figure}
    \centering
    \includegraphics[width=0.35\textwidth]{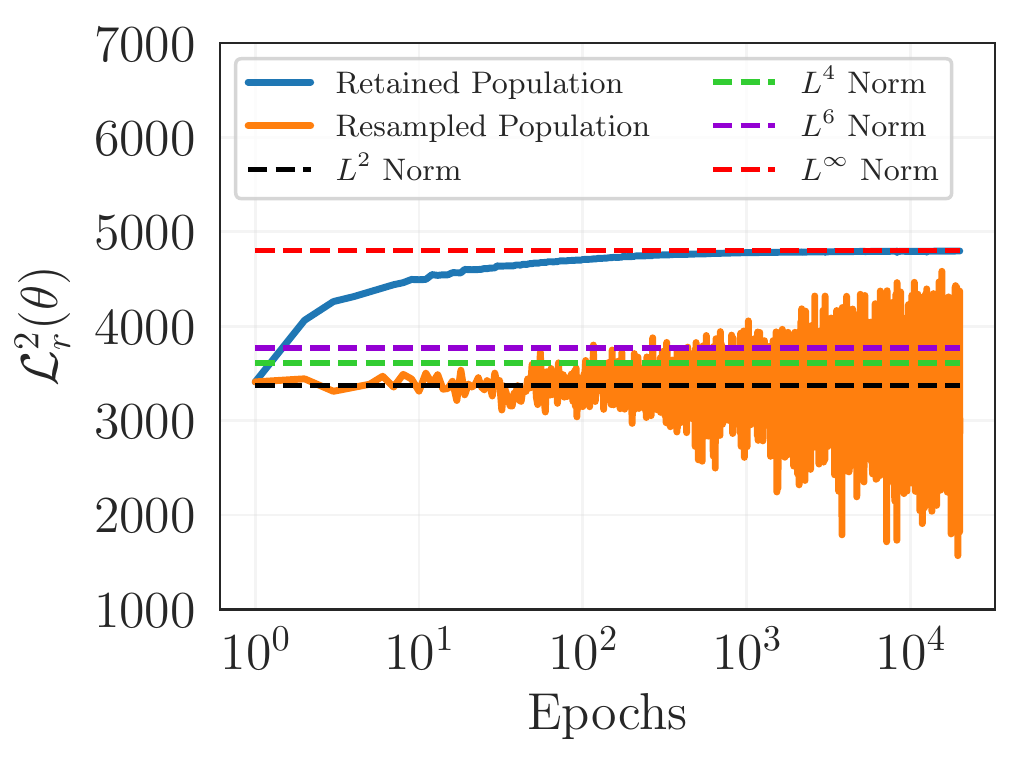}
    \caption{Illustrating the dynamic behavior of the R3 Sampling algorithm on Bohachevsky Function using the $L^2$ Physics-informed Loss computed on the retained and re-sampled populations. The horizontal lines represent the $L^p$ Physics-informed Loss on a dense set of uniformly sampled collocation points (where $p=2,4,6,\infty$).}
    \label{fig:Bohachevsky_dynamic_lp}
\end{figure}

\begin{figure}
    \centering
    \includegraphics[width=0.35\textwidth]{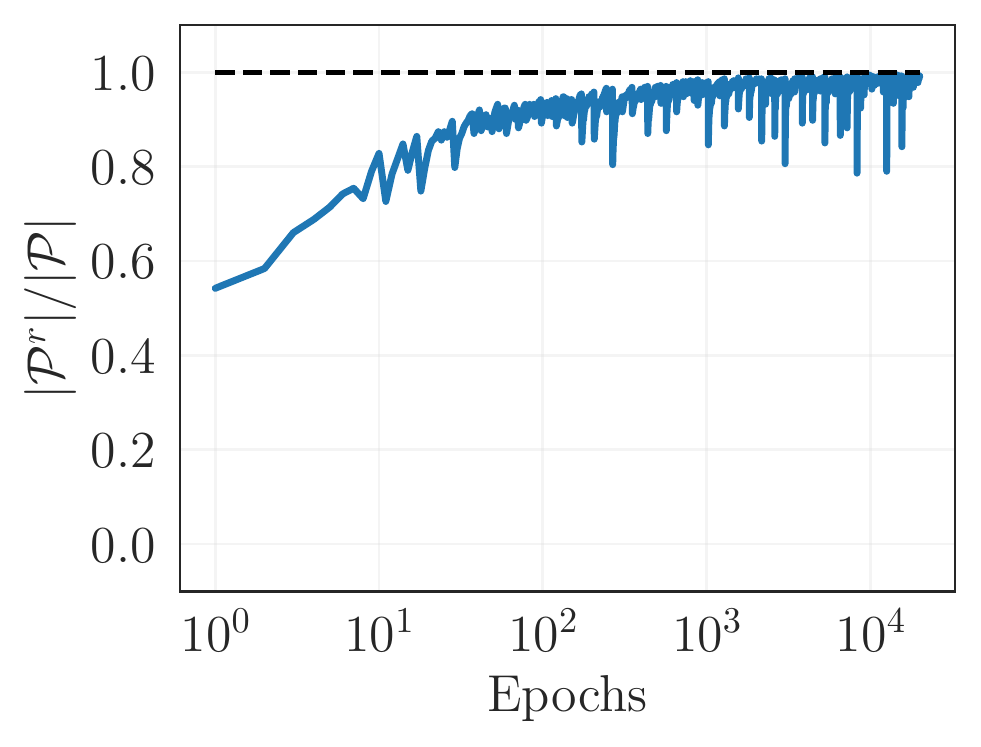}
    \caption{Demonstrating the dynamic evolution of the retained population size over epochs on the Bohachevsky Function.}
    \label{fig:Bohachevsky_retained_size}
\end{figure}

\subsection{Drop-Wave Function}
The two-dimensional form of the Drop-Wave function which is multimodal and highly complex.

\begin{align}
    f(x, y) = \frac{1+\cos(12\sqrt{x^2+y^2})}{0.5(x^2+y^2)+2}
\end{align}

\begin{figure}
    \centering
    \includegraphics[width=0.45\textwidth]{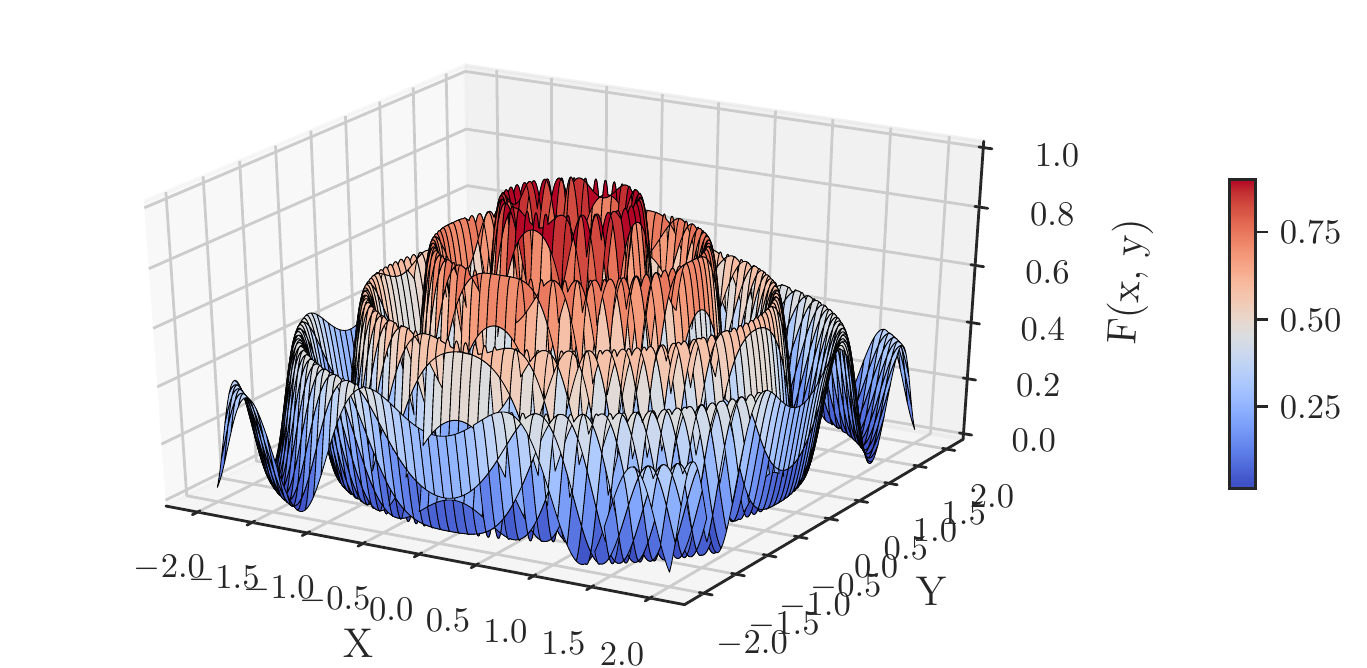}
    \caption{Surface Plot of the 2-D Drop-Wave Function}
    \label{fig:drop_wave_surface}
\end{figure}

\begin{figure}
    \centering
    \includegraphics[width=1.0\textwidth]{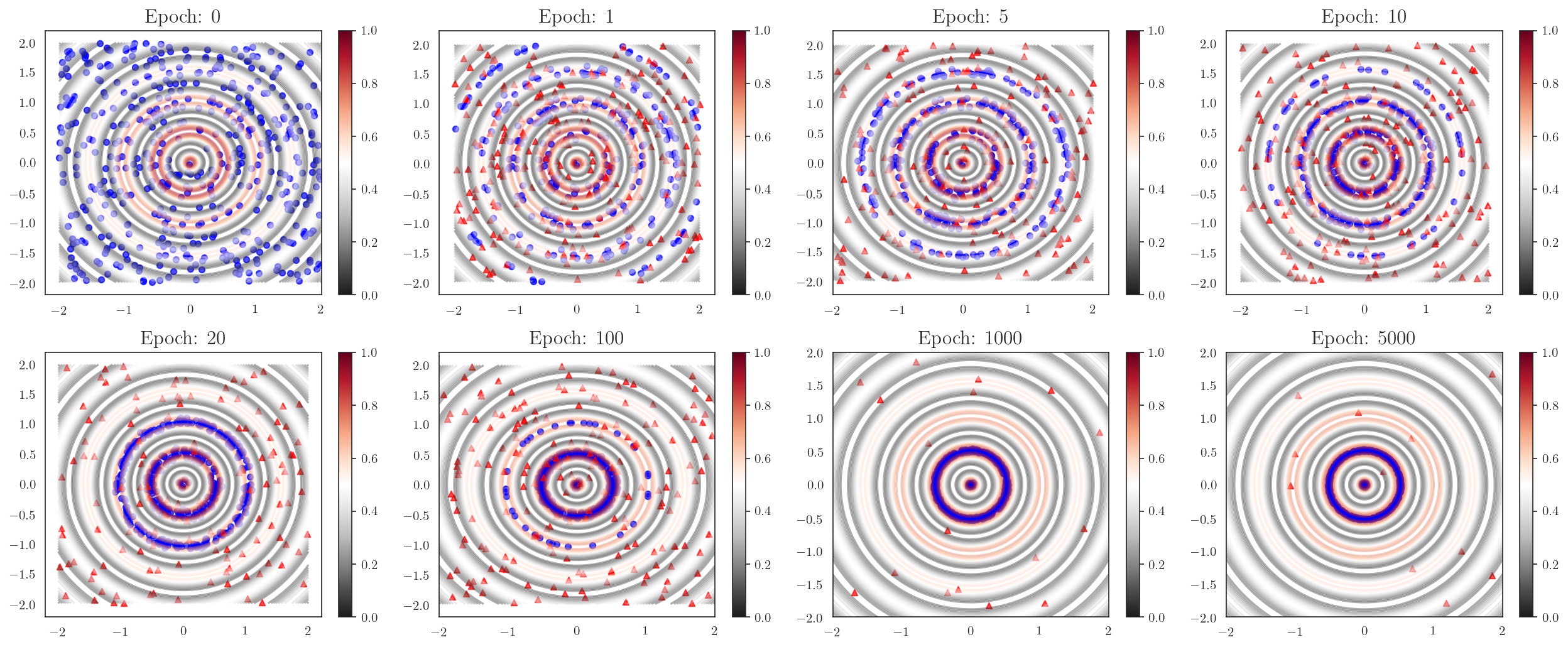}
    \caption{Demonstrating the evolution of the randomly initialized points while optimizing the Drop-Wave function. The red triangles represent the re-sampled population at that epoch, and the blue dots represent the retained population at that epoch. The contour function of the objective function is shown in the background.}
    \label{fig:drop_wave_evosample}
\end{figure}

\begin{figure}
    \centering
    \includegraphics[width=0.35\textwidth]{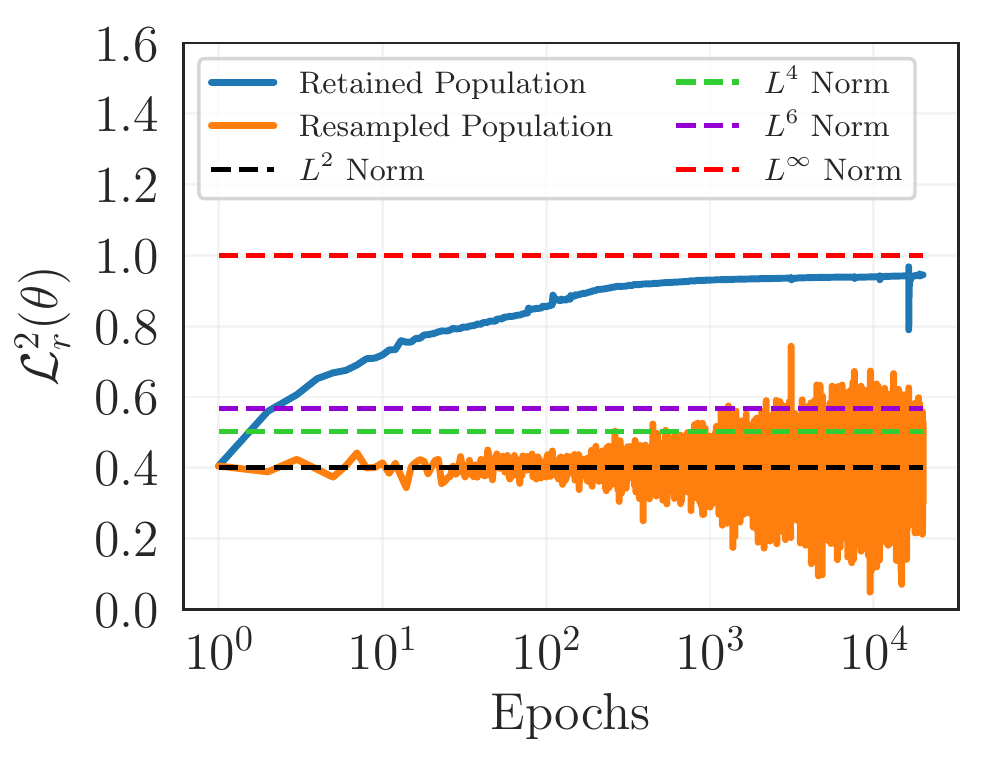}
    \caption{Illustrating the dynamic behavior of the R3 Sampling algorithm on Drop-Wave Function using the $L^2$ Physics-informed Loss computed on the retained and re-sampled populations. The horizontal lines represent the $L^p$ Physics-informed Loss on a dense set of uniformly sampled collocation points (where $p=2,4,6,\infty$).}
    \label{fig:drop_wave_dynamic_lp}
\end{figure}

\begin{figure}
    \centering
    \includegraphics[width=0.35\textwidth]{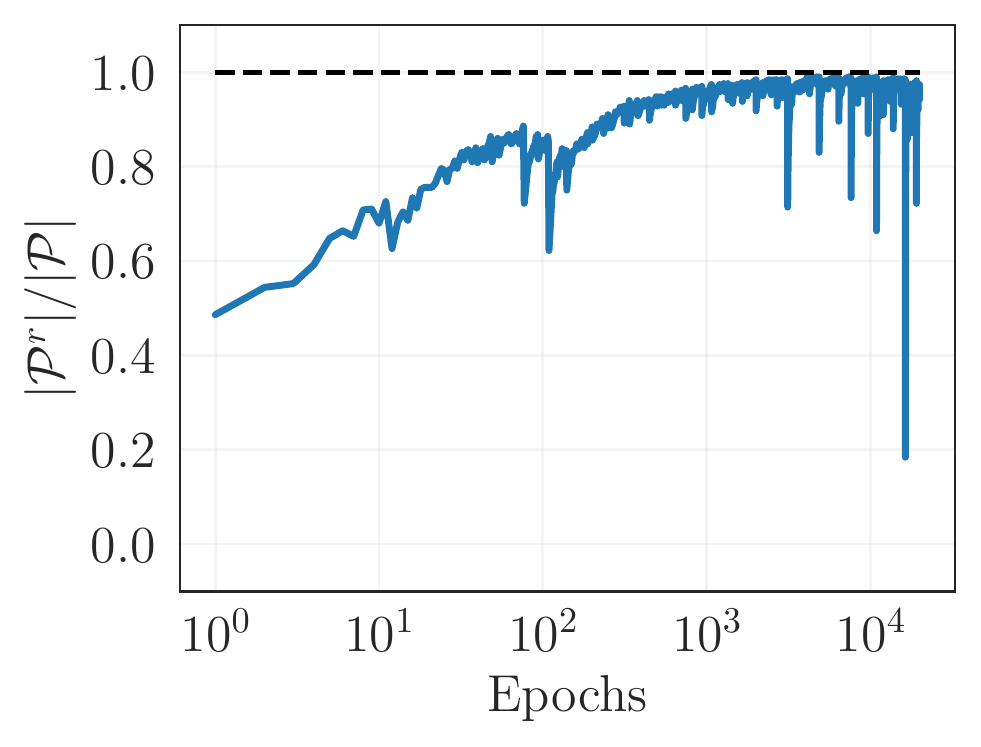}
    \caption{Demonstrating the dynamic evolution of the retained population size over epochs on the Drop-Wave Function.}
    \label{fig:drop_wave_retained_size}
\end{figure}

\subsection{Egg-Holder Function}
The two-dimensional form of the Egg-Holder function is highly complex function that is difficult to optimize because of the presence of multiple local maximas.

\begin{align}
    f(x, y) = (y+47)\sin(\sqrt{|y+\frac{x}{2}+47|}) + x\sin(\sqrt{|x-(y+47)|})
\end{align}

\begin{figure}
    \centering
    \includegraphics[width=0.45\textwidth]{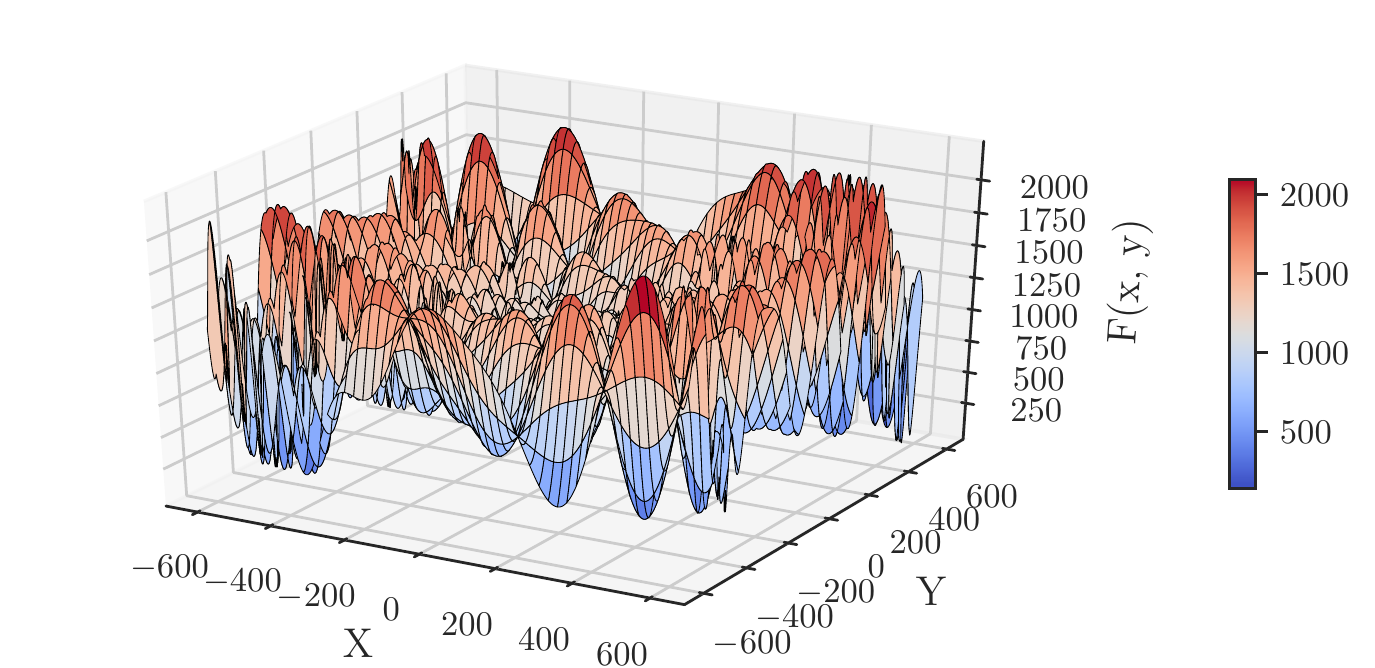}
    \caption{Surface Plot of the 2-D Egg-Holder Function}
    \label{fig:egg_holder_surface}
\end{figure}

\begin{figure}
    \centering
    \includegraphics[width=1.0\textwidth]{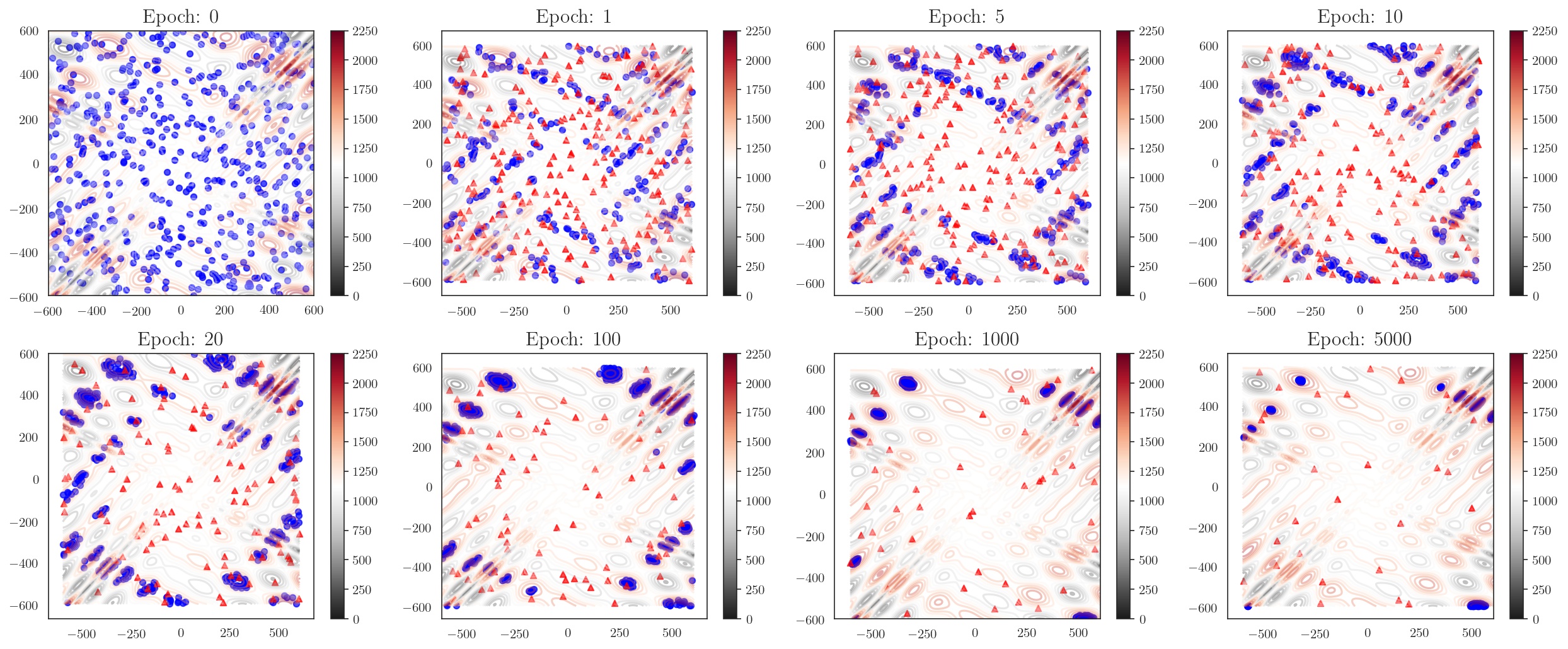}
    \caption{Demonstrating the evolution of the randomly initialized points while optimizing the Egg-Holder function. The red triangles represent the re-sampled population at that epoch, and the blue dots represent the retained population at that epoch. The contour function of the objective function is shown in the background.}
    \label{fig:egg_holder_evosample}
\end{figure}

\begin{figure}
    \centering
    \includegraphics[width=0.35\textwidth]{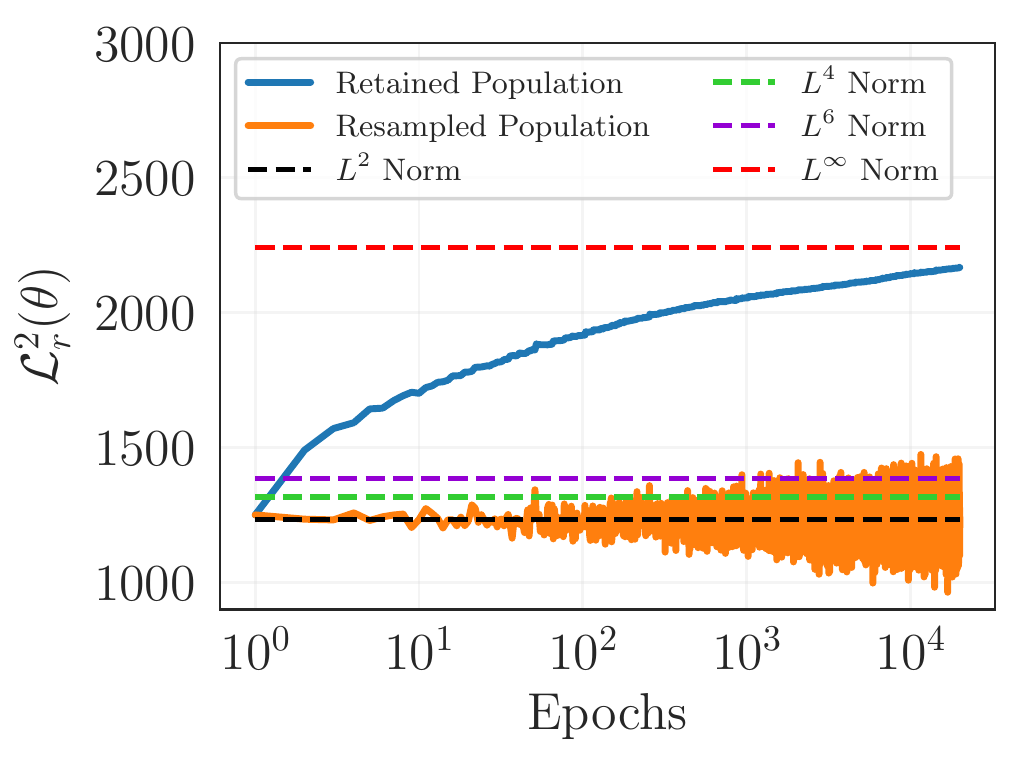}
    \caption{Illustrating the dynamic behavior of the R3 Sampling algorithm on Egg-Holder Function using the $L^2$ Physics-informed Loss computed on the retained and re-sampled populations. The horizontal lines represent the $L^p$ Physics-informed Loss on a dense set of uniformly sampled collocation points (where $p=2,4,6,\infty$).}
    \label{fig:egg_holder_dynamic_lp}
\end{figure}

\begin{figure}
    \centering
    \includegraphics[width=0.35\textwidth]{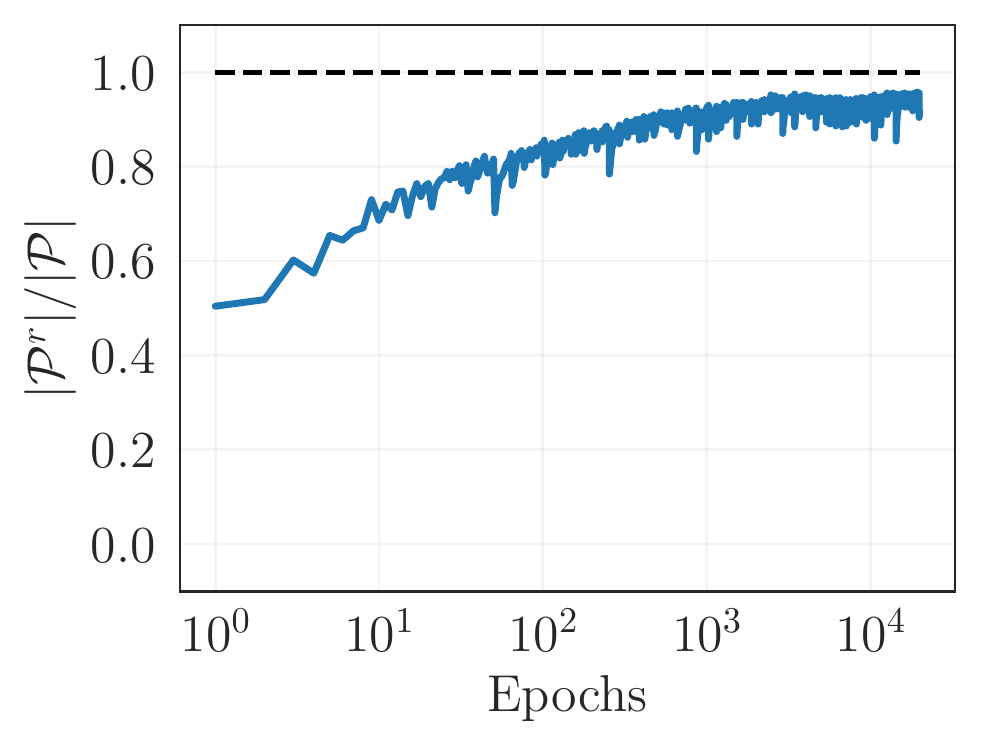}
    \caption{Demonstrating the dynamic evolution of the retained population size over epochs on the Egg-Holder Function.}
    \label{fig:egg_holder_retained_size}
\end{figure}

\subsection{Holder-Table Function}
The two-dimensional form of the Holder-Table function has many local maximas, but has 4 global maximas at the four corners.
\begin{align}
    f(x, y) = |\sin(x)\cos(y)\exp\Bigg(|1-\frac{\sqrt{x^2+y^2}}{\pi}|\Bigg)|
\end{align}

\begin{figure}
    \centering
    \includegraphics[width=0.45\textwidth]{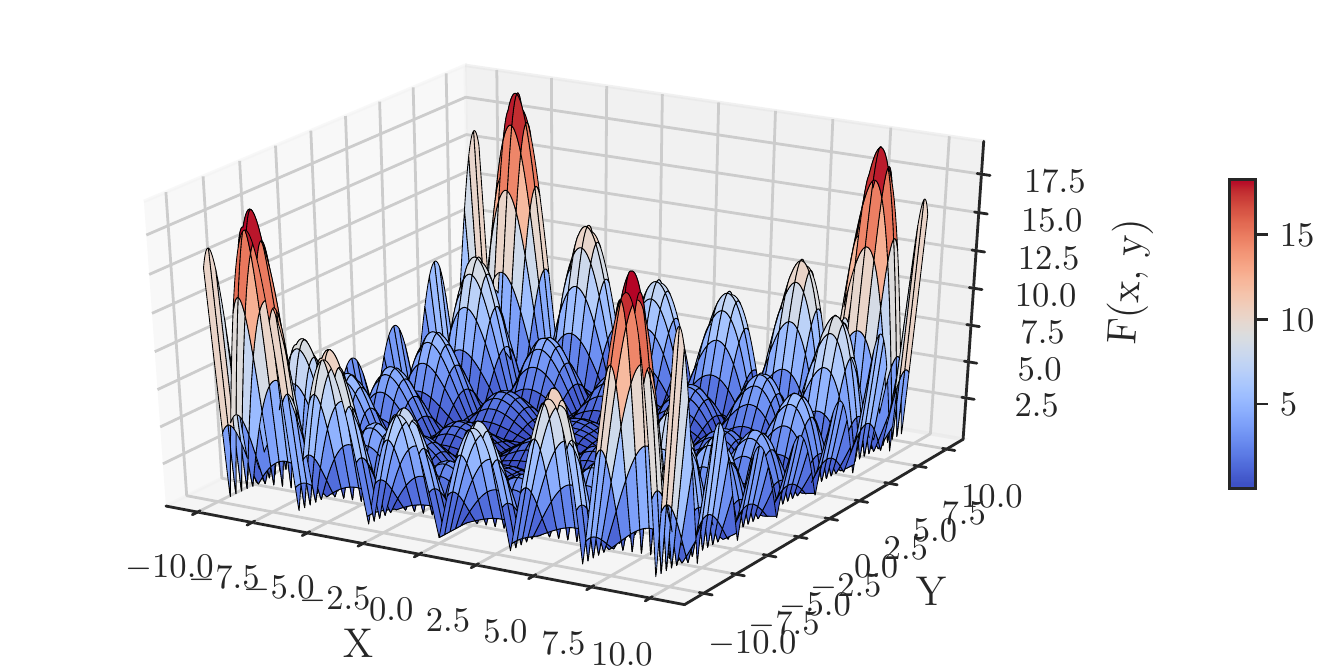}
    \caption{Surface Plot of the 2-D Holder-Table Function}
    \label{fig:holder_table_surface}
\end{figure}

\begin{figure}
    \centering
    \includegraphics[width=1.0\textwidth]{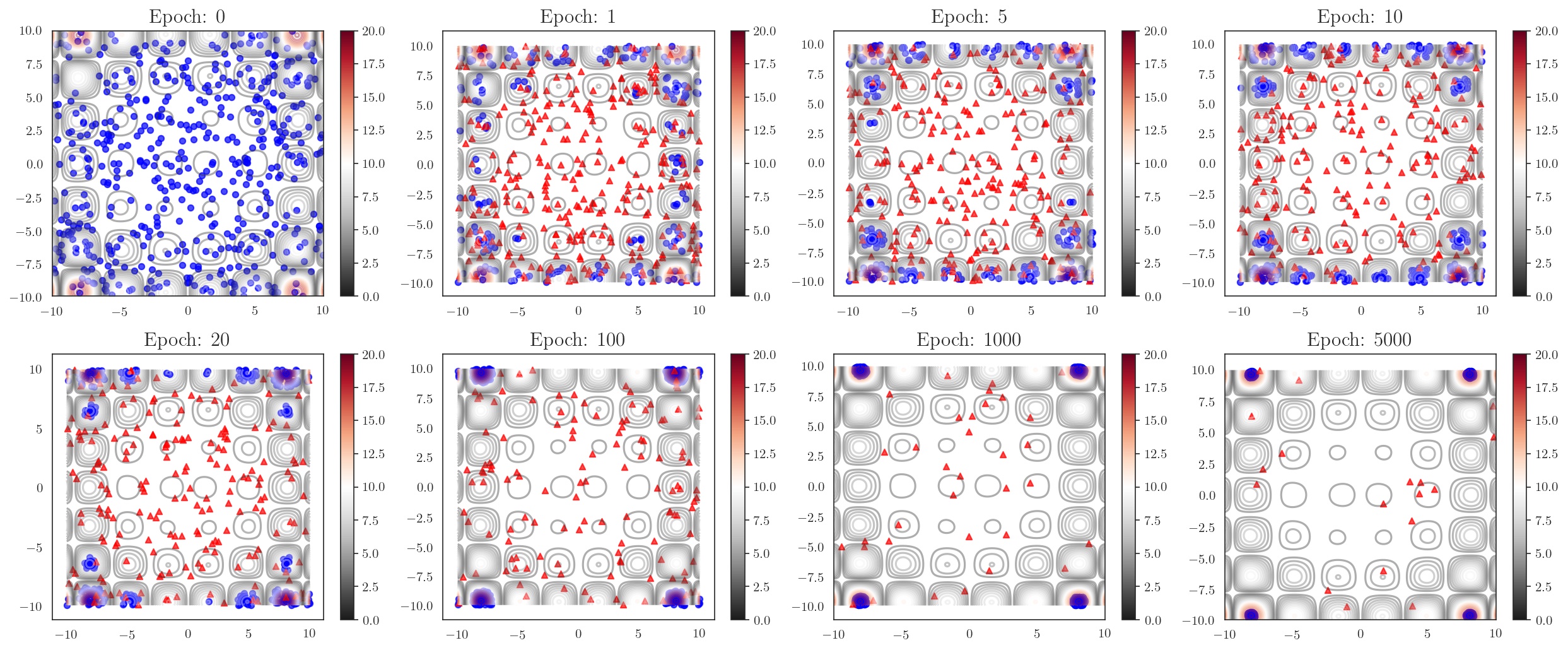}
    \caption{Demonstrating the evolution of the randomly initialized points while optimizing the Holder-Table function. The red triangles represent the re-sampled population at that epoch, and the blue dots represent the retained population at that epoch. The contour function of the objective function is shown in the background.}
    \label{fig:holder_table_evosample}
\end{figure}

\begin{figure}
    \centering
    \includegraphics[width=0.35\textwidth]{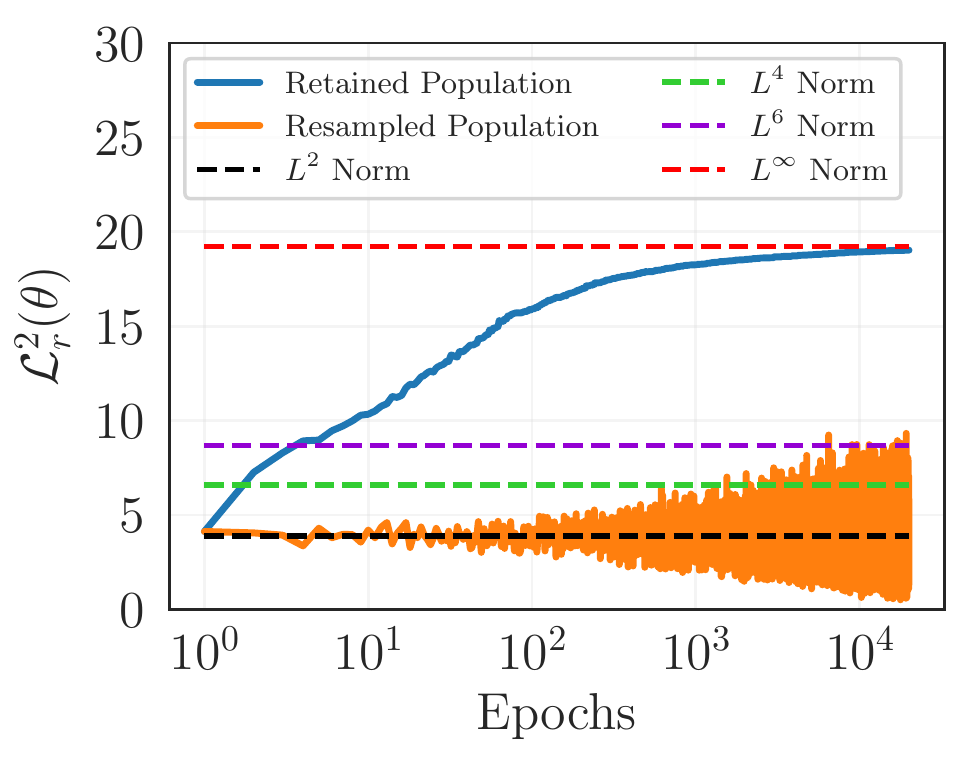}
    \caption{Illustrating the dynamic behavior of the R3 Sampling algorithm on Holder-Table Function using the $L^2$ Physics-informed Loss computed on the retained and re-sampled populations. The horizontal lines represent the $L^p$ Physics-informed Loss on a dense set of uniformly sampled collocation points (where $p=2,4,6,\infty$).}
    \label{fig:holder_table_dynamic_lp}
\end{figure}

\begin{figure}
    \centering
    \includegraphics[width=0.35\textwidth]{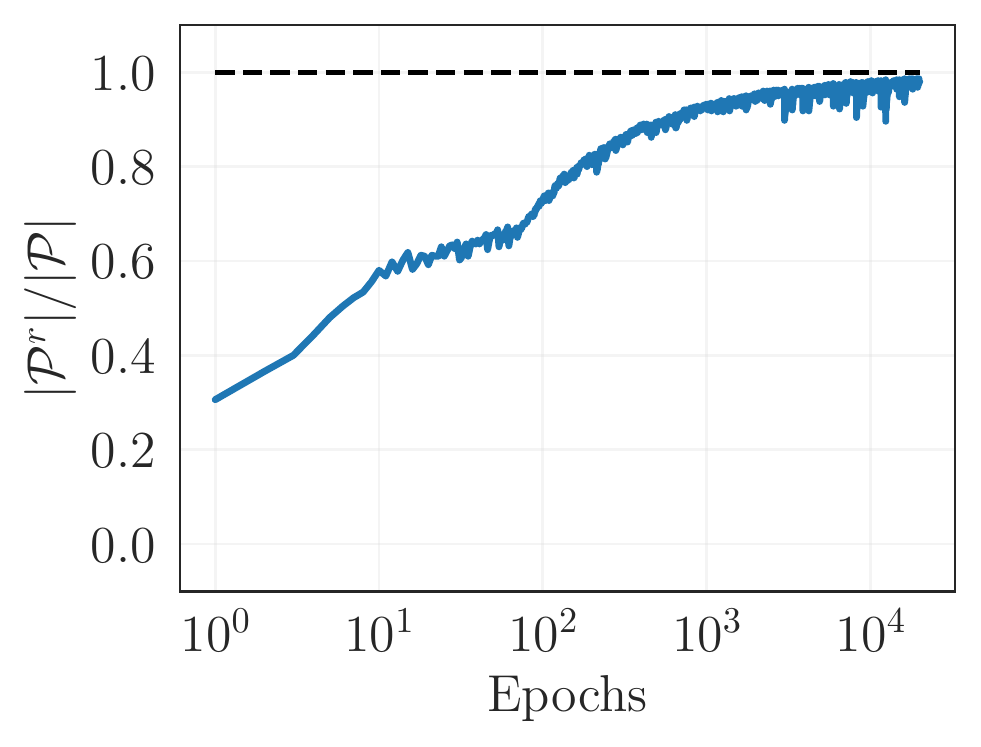}
    \caption{Demonstrating the dynamic evolution of the retained population size over epochs on the Holder-Table Function.}
    \label{fig:holder_table_retained_size}
\end{figure}

\subsection{Bukin Function}
The two-dimensional form of the Bukin function has many local maximas, all of which lie on a ridge.

\begin{align}
    f(x, y) = -100\sqrt{|y-0.01x^2|}-0.01|x+10|
\end{align}

\begin{figure}
    \centering
    \includegraphics[width=0.45\textwidth]{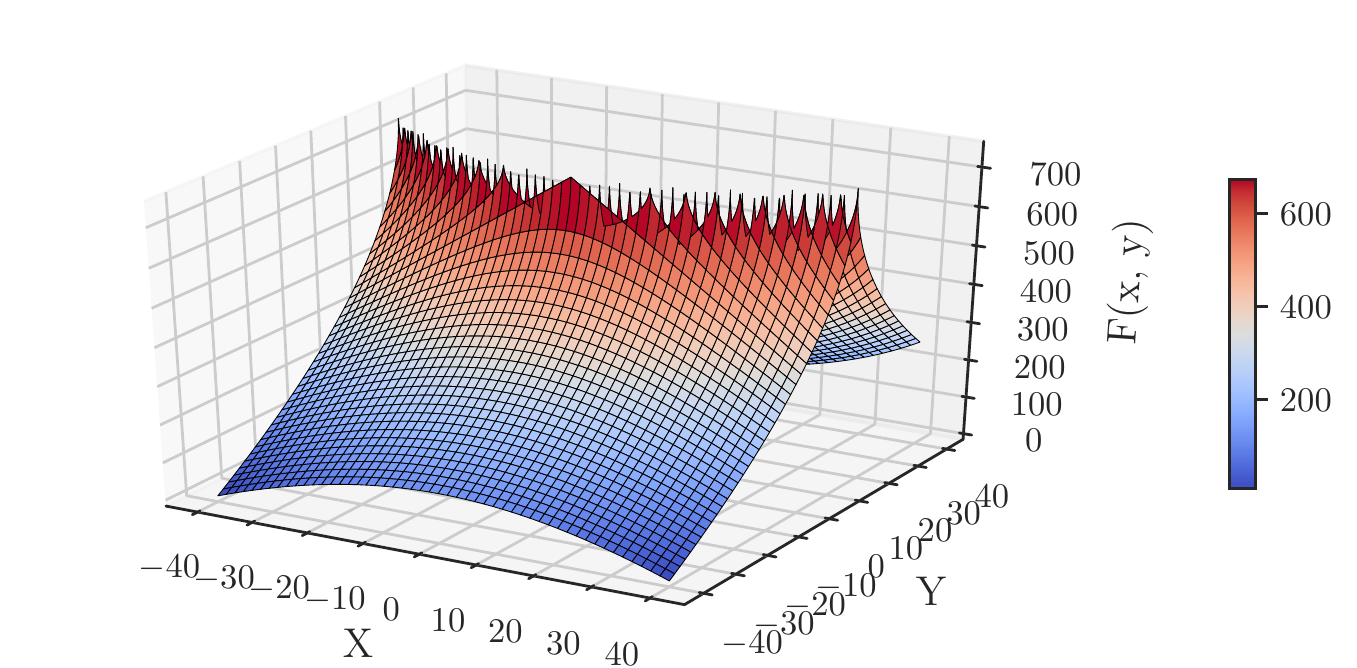}
    \caption{Surface Plot of the 2-D Bukin Function}
    \label{fig:bukin_surface}
\end{figure}

\begin{figure}
    \centering
    \includegraphics[width=1.0\textwidth]{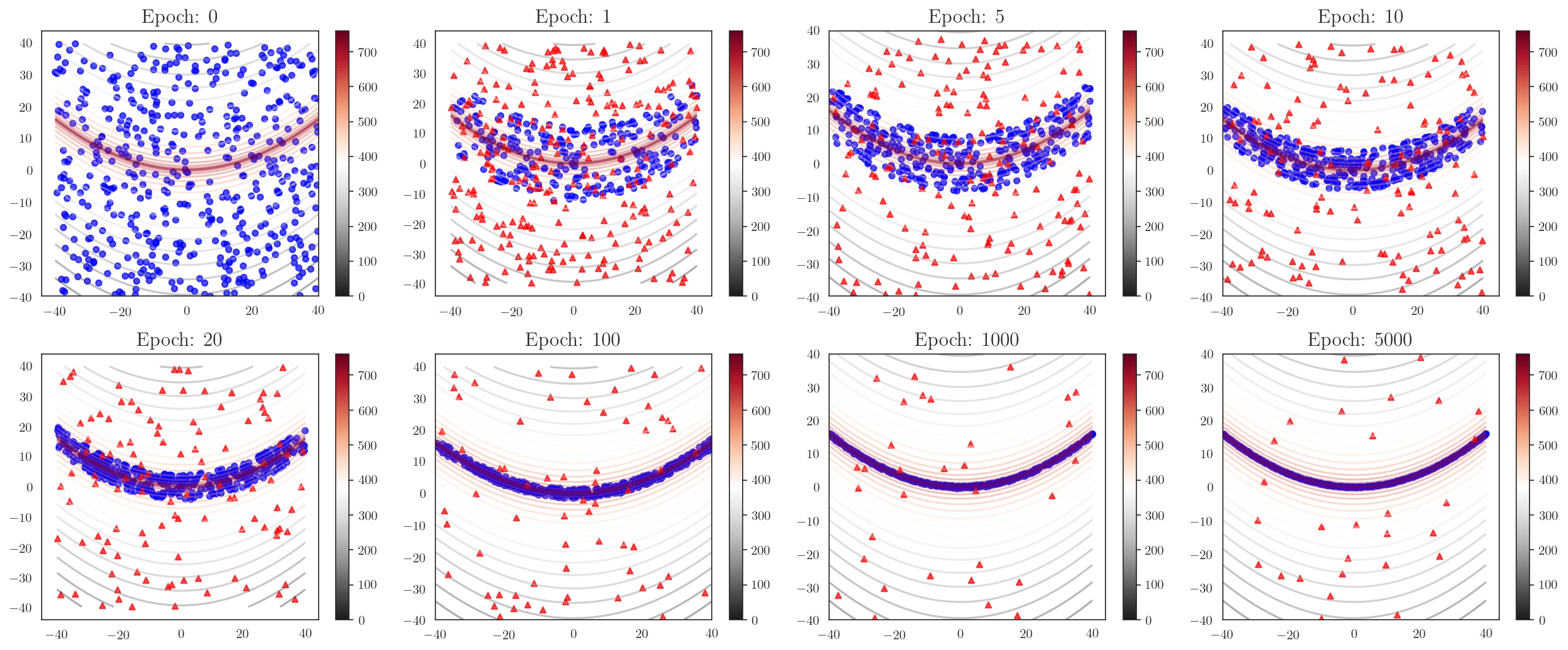}
    \caption{Demonstrating the evolution of the randomly initialized points while optimizing the Bukin function. The red triangles represent the re-sampled population at that epoch, and the blue dots represent the retained population at that epoch. The contour function of the objective function is shown in the background.}
    \label{fig:bukin_evosample}
\end{figure}

\begin{figure}
    \centering
    \includegraphics[width=0.35\textwidth]{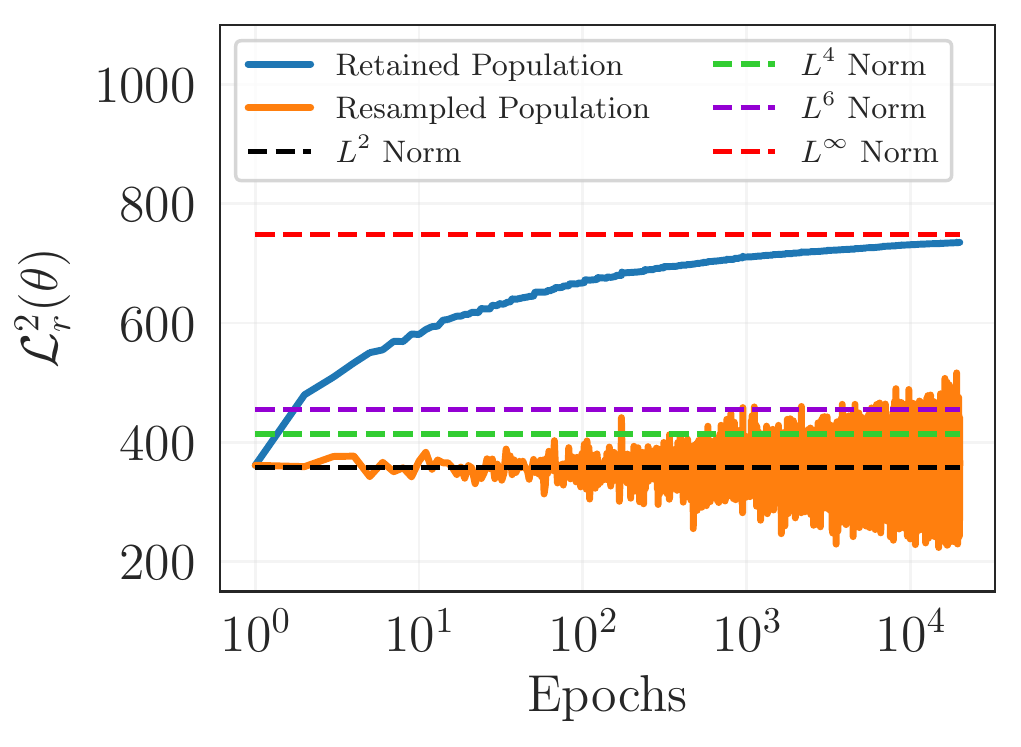}
    \caption{Illustrating the dynamic behavior of the R3 Sampling algorithm on Bukin Function using the $L^2$ Physics-informed Loss computed on the retained and re-sampled populations. The horizontal lines represent the $L^p$ Physics-informed Loss on a dense set of uniformly sampled collocation points (where $p=2,4,6,\infty$).}
    \label{fig:bukin_dynamic_lp}
\end{figure}

\begin{figure}
    \centering
    \includegraphics[width=0.35\textwidth]{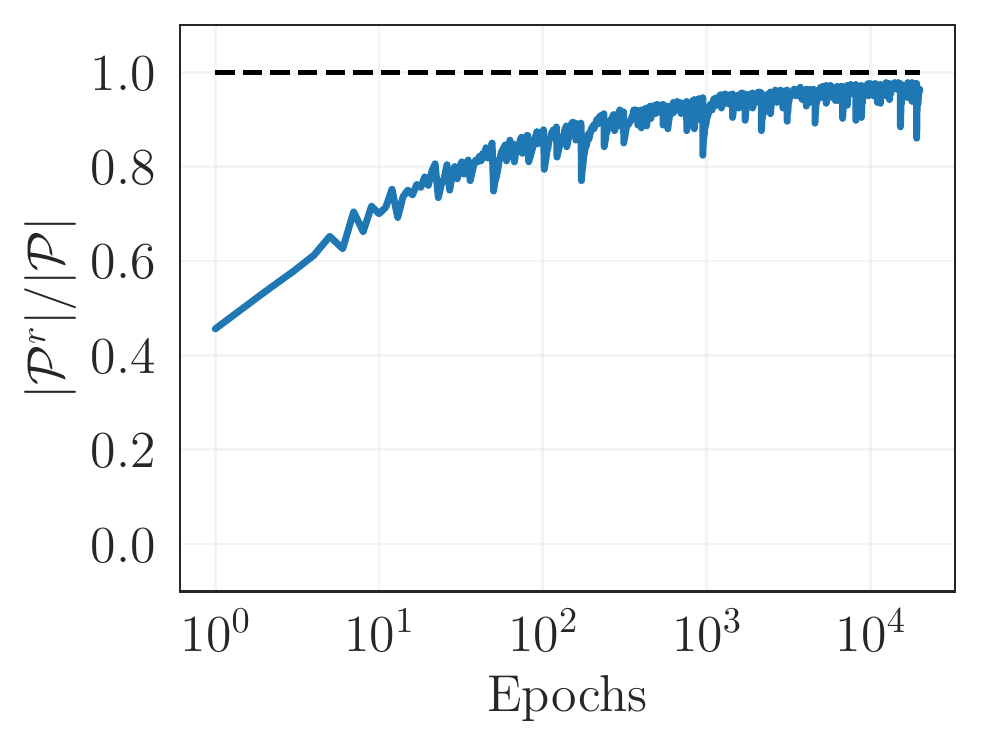}
    \caption{Demonstrating the dynamic evolution of the retained population size over epochs on the Bukin Function.}
    \label{fig:bukin_retained_size}
\end{figure}

\subsection{Michalewicz Function}
\label{sec:michael}
The two-dimensional form of the Michalewicz function has multiple ridges and valleys which are very steep.
\begin{align}
    f(x) = \sum_{i=1}^d \sin(x_i)\sin^{2m}\Bigg(\frac{ix_i^2}{\pi}\Bigg)
\end{align}

\begin{figure}
    \centering
    \includegraphics[width=0.45\textwidth]{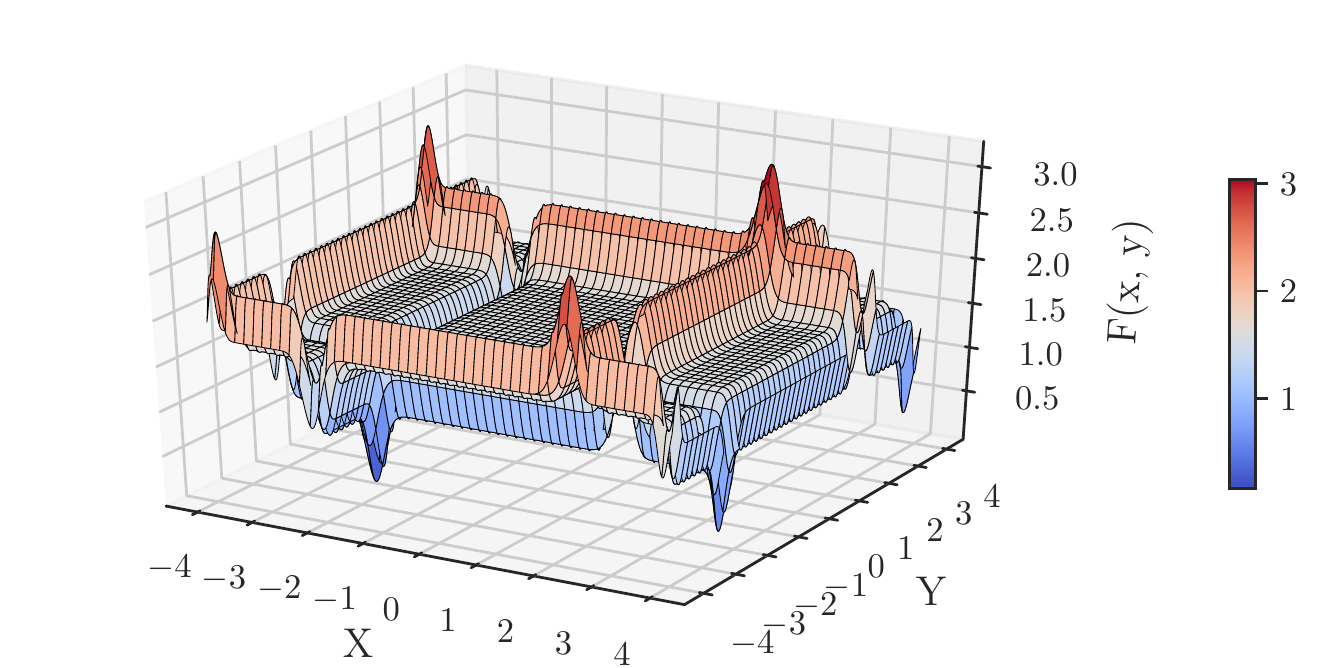}
    \caption{Surface Plot of the 2-D Michalewicz Function}
    \label{fig:michalewicz_surface}
\end{figure}

\begin{figure}
    \centering
    \includegraphics[width=1.0\textwidth]{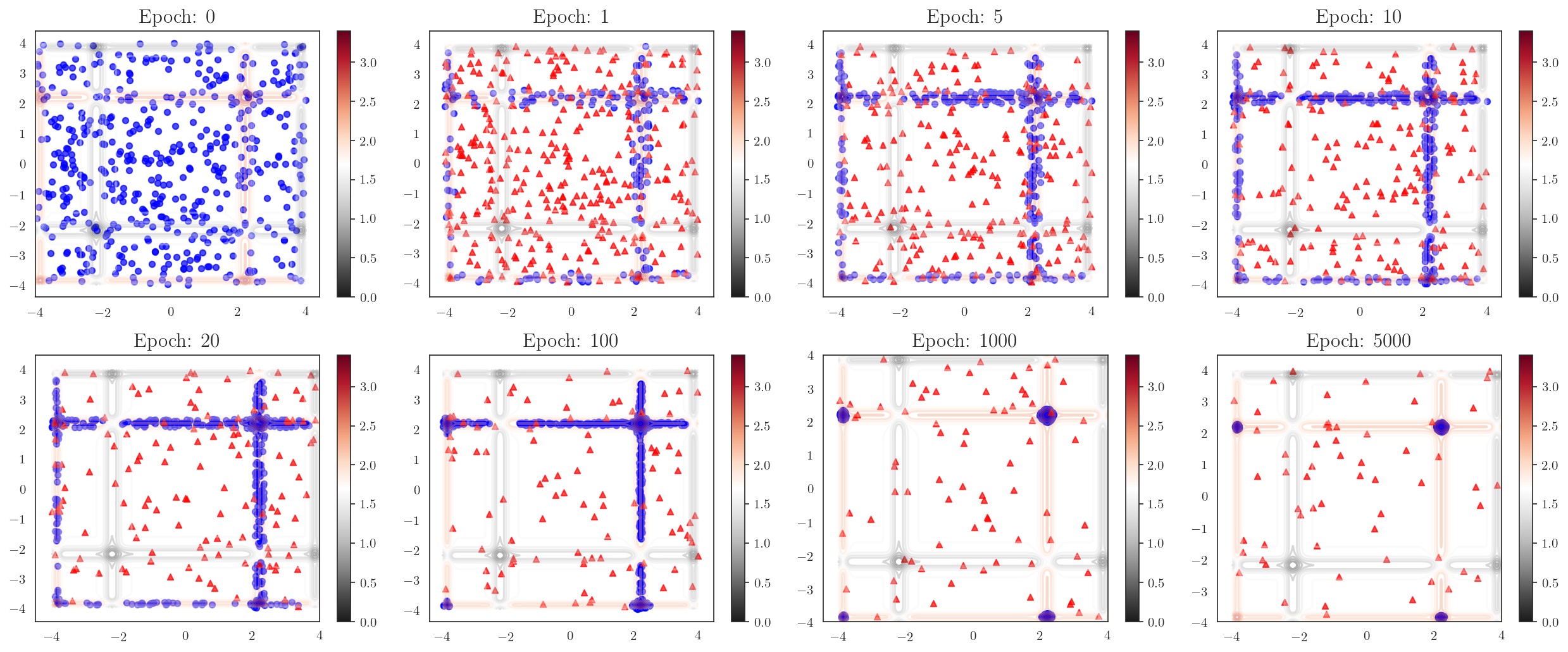}
    \caption{Demonstrating the evolution of the randomly initialized points while optimizing the Michalewicz function. The red triangles represent the re-sampled population at that epoch, and the blue dots represent the retained population at that epoch. The contour function of the objective function is shown in the background.}
    \label{fig:michalewicz_evosample}
\end{figure}

\begin{figure}
    \centering
    \includegraphics[width=0.35\textwidth]{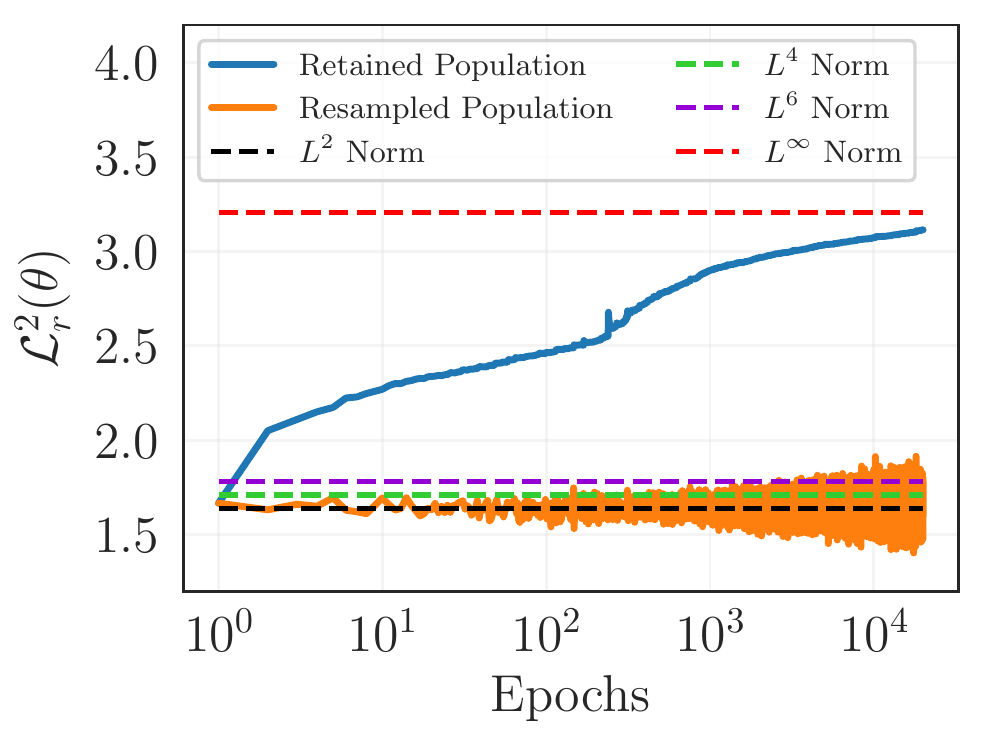}
    \caption{Illustrating the dynamic behavior of the R3 Sampling algorithm on Michalewicz Function using the $L^2$ Physics-informed Loss computed on the retained and re-sampled populations. The horizontal lines represent the $L^p$ Physics-informed Loss on a dense set of uniformly sampled collocation points (where $p=2,4,6,\infty$).}
    \label{fig:michalewicz_dynamic_lp}
\end{figure}

\begin{figure}
    \centering
    \includegraphics[width=0.35\textwidth]{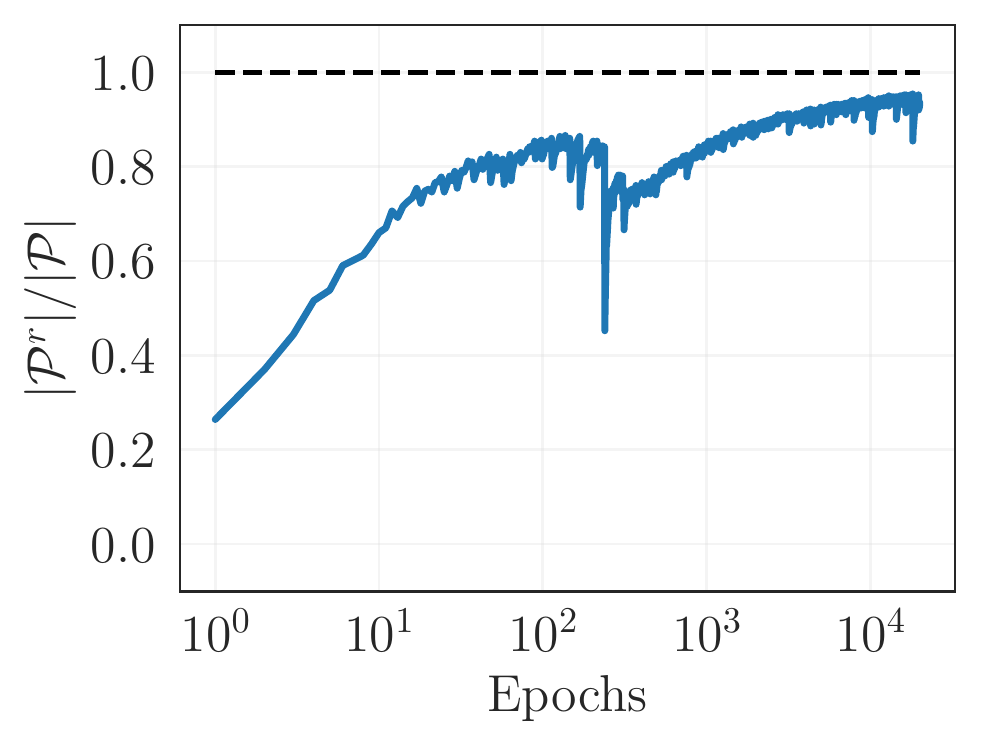}
    \caption{Demonstrating the dynamic evolution of the retained population size over epochs on the Michalewicz Function.}
    \label{fig:michalewicz_retained_size}
\end{figure}